\documentclass{article}
\usepackage{iclr2027_conference,times}

\usepackage{amsmath,amsfonts,bm}

\def\Figref#1{Figure~\ref{#1}}

\def\Secref#1{Section~\ref{#1}}

\def\eqref#1{equation~\ref{#1}}
\def\Eqref#1{Eq.~(\ref{#1})}

\def\1{\bm{1}}

\DeclareMathAlphabet{\mathsfit}{\encodingdefault}{\sfdefault}{m}{sl}
\SetMathAlphabet{\mathsfit}{bold}{\encodingdefault}{\sfdefault}{bx}{n}

\usepackage{hyperref}
\usepackage{url}

\usepackage{xcolor}
\definecolor{kaistblue}{RGB}{20, 135, 200}
\hypersetup{
    colorlinks=true,
    citecolor=kaistblue,
    linkcolor=kaistblue,
    urlcolor=kaistblue
}
\usepackage{amsmath,amssymb,bm}
\usepackage{algorithm}
\usepackage{algpseudocode}
\algrenewcommand\algorithmicrequire{\textbf{Input:}}
\algrenewcommand\algorithmicensure{\textbf{Output:}}

\usepackage{amsthm}
\newtheorem{proposition}{Proposition}

\usepackage{booktabs}
\usepackage{array}
\usepackage{makecell}
\usepackage{graphicx}
\usepackage{multirow}
\usepackage{multicol}
\usepackage[x11names]{xcolor}
\usepackage{colortbl}
\usepackage{kotex}
\usepackage{setspace}
\usepackage{caption}
\usepackage{enumitem}
\usepackage{booktabs}
\usepackage{wrapfig}

\title{
    DriftOPD: Sequence-Level Reverse-KL \newline Distillation for One-Step VLA Policies
}

\author{\fontsize{9pt}{10.8pt}\selectfont 
Youngjun Jun$^{1}$, 
Kyumin Choi$^2$, 
Youngmin Kim$^1$,
Seonghyun Jin$^1$,
Sunwoo Park$^1$,
\\
\fontsize{9pt}{10.8pt}\selectfont \bf  
Jangho Park$^1$, 
Jong Chul Ye$^{1}$
\\[0.8em]
\fontsize{9pt}{10.8pt}\selectfont
$^1$KAIST, Republic of Korea \\
\fontsize{9pt}{10.8pt}\selectfont
$^2$Sungkyunkwan University, Republic of Korea\\
\fontsize{9pt}{10.8pt}\selectfont
\texttt{\{youngjun.jun,jong.ye\}@kaist.ac.kr}
\\[0.8em]
\fontsize{9pt}{10.8pt}\selectfont
Project page: \url{https://yj-jun.github.io/DriftOPD/}
}

\iclrfinalcopy
\begin{document}

\maketitle

\begin{abstract}
Vision-Language-Action (VLA) models increasingly rely on action experts that generate short action chunks under receding-horizon control.
While chunk-level training is convenient across robot embodiments, it optimizes local action likelihood without explicitly accounting for long-horizon task success.
Sequence-level reinforcement learning can address this limitation, but typically requires policy rollouts and closed-loop interaction, which are costly for real-robot manipulation.
We introduce \textbf{DriftOPD}, a teacher-free, rollout-free framework for sequence-level on-policy distillation of continuous VLA action experts.
We show that the sequence-level reverse Kullback--Leibler (KL) divergence decomposes into a chunk-level reverse-KL term and a future-potential term that captures the long-horizon effect of the current action.
DriftOPD optimizes these two terms using a one-step drifting objective and a Q-function critic learned from offline demonstrations, respectively, enabling sequence-level optimization with only offline data and one-step action generation.
Across multiple VLA architectures in simulation and real-world manipulation, DriftOPD generally outperforms existing one-step distillation baselines while achieving task success performance comparable to multi-step teacher policies.
These results demonstrate that long-horizon behavior can be effectively distilled into one-step VLA action experts without online interaction or a separate teacher. 
\end{abstract}
\section{Introduction}
\label{sec:intro}

Vision-Language-Action (VLA) models have recently emerged as a promising paradigm for general-purpose robotic manipulation~\citep{RT2,OpenVLA,pi05}.
As a robot manipulation system, a VLA leverages the reasoning and planning knowledge of Vision-Language Models (VLMs), corresponding to System 2, together with the intuitive behavior generation of a flow-based action expert, corresponding to System 1~\citep{pi0,HiRobot}.
This VLA architecture is designed to reflect the state (e.g., environment and proprioception) that change according to actions, and VLAs perform tasks through closed-loop execution~\citep{RT1,Gr00tN1}.
Since replanning every low-level action at each control step is computationally inefficient, VLAs commonly adopt a receding-horizon control scheme, repeatedly predicting short action chunks from updated observations~\citep{ACT,OpenVLA-OFT}.

Action chunks provide a convenient common training unit for action experts across multiple embodiments~\citep{pi05,Gr00tN1}.
However, under this chunk-level training paradigm, VLAs suffer from two problems (\Figref{fig:motivation}a).
First, because VLAs are trained on demonstration states at train time, these VLAs suffer from a \textit{test-time distribution mismatch} and fail to generalize to model-generated states~\citep{DAger,Stable-BC}.
Second, within an action sequence composed of action chunks, the policy na\"ively generates only locally likely action chunks at each step without considering the \textit{final task success}~\citep{ProphRL}.
On the other hand, while recent reinforcement learning methods improve VLA training using various forms of sequence-level feedback~\citep{VLA-RL,SimpleVLA-RL,VLA-RFT}, these approaches heavily rely on \textit{policy-generated action sequences} to obtain such supervision.
They must execute the current policy, evaluate the resulting sequences, and often recollect new rollouts as the policy changes.
This dependence on closed-loop interaction is particularly costly for real-robot manipulation~\citep{WMPO}.

\begin{figure}[t]
  \centering
   \includegraphics[width=\linewidth]{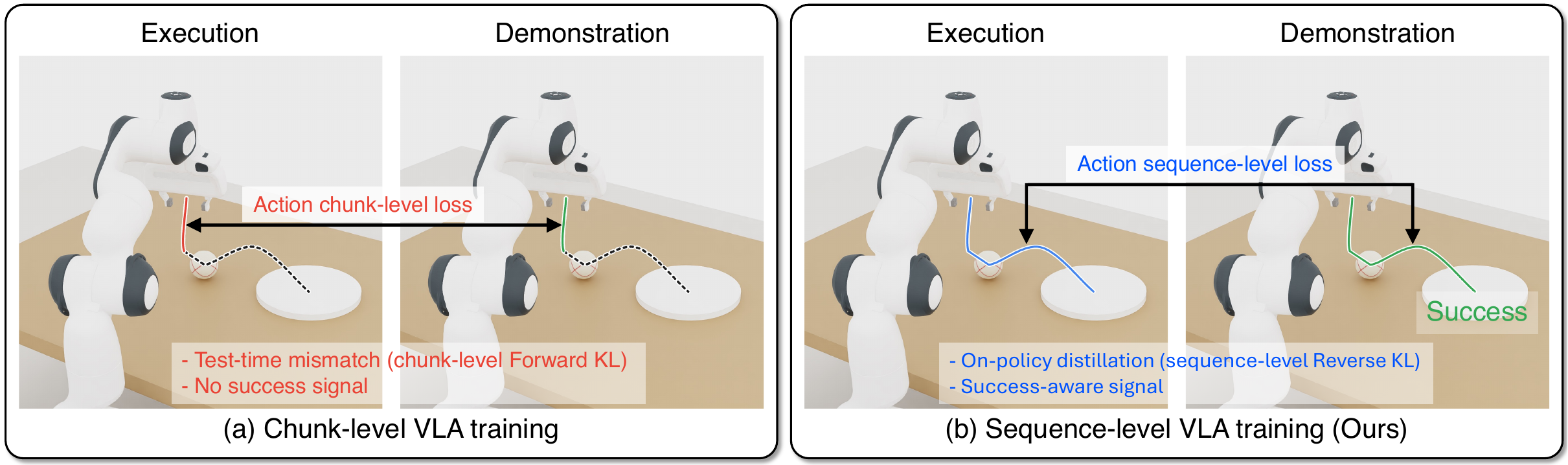}
    \vspace{-13pt}
   \caption{Chunk-level VLA training versus our sequence-level formulation}
   \vspace{-18pt}
   \label{fig:motivation}
\end{figure}

To this end, we explore a scalable and simple sequence-level objective function for action-expert training (\Figref{fig:motivation}b).
Currently, due to the multi-step sampling of diffusion/flow action experts, research on the trajectory-level reverse Kullback-Leibler divergence (KLD) for continuous generative policies has remained at the level of \textit{chunk}-level ODE trajectories~\citep{DiffusionOPD,FlowOPD}.
In this work, we formulate \textit{sequence}-level on-policy distillation (OPD) through the reverse KLD over the action sequence leading to success.
Our formulated sequence-level OPD consists of a \textit{chunk-level reverse KLD term} and a \textit{future-potential term} for the current action chunk, similar to MiniLLM~\citep{MiniLLM}.
However, unlike MiniLLM, the clean logits of continuous action chunks required for the future-potential term remain intractable.
Therefore, the sequence-level OPD formulation itself still relies on rollout data, is bounded by the performance of the teacher model, and incurs multi-step sampling latency.

To address these challenges, we propose \textbf{DriftOPD}, a teacher-free sequence-level OPD method for one-step distillation.
DriftOPD calculates the chunk-level reverse KLD term as a one-step distribution-matching drifting objective~\citep{Drifting}, while estimating the future-potential term with a critic-based neural network.
In this way, DriftOPD achieves one-step inference while learning directly from human demonstrations, without a separate teacher model.
In practice, we adapt VGAS~\citep{VGAS}, which provides a demonstration-trained Q-function critic, to achieve fully rollout-free sequence-level OPD.
Consequently, because DriftOPD uses only demonstration data, it can be readily applied across multiple embodiments, while its generated-sample-based sequence-level OPD avoids the test-time mismatch.
Moreover, one-step policy generation through the drifting objective reduces robot manipulation latency.

We apply DriftOPD to a diverse set of VLA models spanning different architectures, fine-tuning strategies, and action-expert parameterizations~\citep{pi05,Gr00tN1,SmolVLA,ABot-M0}. We distill their multi-step sampling capability into one-step policy generation at the action-sequence level.
Our experiments cover simulation and real-world execution across a wide range of robot manipulation tasks~\citep{RoboCasa365,RoboTwin2,Libero,LiberoPlus}.
Experiments show that DriftOPD outperforms existing one-step distillation methods for the VLA action expert, achieving task success rates comparable to those of multi-step teachers.
Further analysis shows that DriftOPD provides more stable training in data-scarce VLA settings, remains effective across diverse VLA configurations (e.g., architecture, fine-tuning strategy, and parameterization), and enables faster execution through one-step policy generation, with negligible loss of action diversity.

Our main contributions are summarized as follows:
\begin{itemize}[noitemsep, nolistsep, leftmargin=*]
    \item We formulate \textbf{sequence-level on-policy-distillation for VLAs} by minimizing the reverse KLD over action-sequence distributions, decomposing it into a local action-chunk reverse-KL term and a future-potential term.
    \item We propose \textbf{DriftOPD} to make sequence-level OPD tractable for continuous VLA action experts. It realizes the local reverse KLD term with a one-step drifting objective and approximates the future-potential term with a demonstration-trained Q-function critic, enabling teacher-free and rollout-free one-step distillation.
    \item Across diverse VLA configurations in simulation and real-world manipulation, DriftOPD generally outperforms existing one-step distillation baselines while achieving task success comparable to multi-step policies. It further provides stable training in data-scarce VLA settings and faster execution than multi-step policy generation.
\end{itemize}
\section{Related Work}
\label{sec:related}

\subsection{Vision-Language-Action (VLA) Models}
\label{ssuc:related-VLA}

Vision-Language-Action (VLA) models have emerged as general-purpose robot policies that couple vision-language representations with imitation learning from large-scale robot demonstrations~\citep{RT2,pi0,Gr00tN1}. Early models such as RT-1~\citep{RT1} and RT-2~\citep{RT2} established scalable language-conditioned robot policies. Recent models such as $\pi_0$ series~\citep{pi05,pi0} use continuous action experts for diverse robot tasks.
In parallel, ACT~\citep{ACT} and Diffusion Policy~\citep{DiffusionPolicy} demonstrated the effectiveness of predicting short action sequences (i.e., chunk) rather than individual actions. Action chunking has since become a common training and execution unit in continuous-control VLAs~\citep{pi0,OpenVLA-OFT,RLDX1}, and is adopted by recent models such as $\pi_0$ series~\citep{pi05,pi0}, RLDX~\citep{RLDX1}, ABot~\citep{ABot-M0}, and GR00T series~\citep{Gr00tN1}.
However, behavior-cloned policies remain vulnerable to test-time distribution shift, where the pretrained model fail to generalize to model-
generated state at the test time~\citep{DAger,Stable-BC}. Moreover, chunk-level imitation objectives provide no direct supervision for final task success~\citep{ProphRL,WMPO}. Recent reinforcement learning methods address these limitations through sequence-level feedback~\citep{VLA-RL,SimpleVLA-RL,VLA-RFT}. However, obtaining such feedback from physical rollouts requires costly closed-loop interaction on real robots~\citep{WMPO}.

\subsection{One-Step Distillation of Diffusion/Flow Models}
\label{ssec:related-distill}

Diffusion and flow-based models~\citep{NCSN,SongSDE,RectifiedFlow} have been widely adopted as action experts in VLA models~\citep{pi05,Gr00tN1}, but their iterative sampling incurs substantial inference latency~\citep{DDIM,DPMSolver}.
To reduce the number of sampling steps, prior work has explored progressive and consistency distillation~\citep{PD,CM,sCM}, distribution matching~\citep{DMD}, and more recently native one-step formulations such as consistency objective~\citep{CM,iCM,sCM}, MeanFlow~\citep{MF,iMF}, and distribution drifting~\citep{Drifting}.
Similar ideas have been introduced for robotic policies, including consistency- and distribution-matching-based policy distillation~\citep{ConsistencyPolicy,SDMPolicy}.
For VLA models, \citet{OFPolicy} employ self-distillation for one-step action generation, while \citet{MFPolicy} and \citet{DriftingVLA} directly introduce native one-step generation objectives into VLA action experts.
Despite their efficiency, existing approaches mainly optimize action-chunk-level ODE-trajectory or distribution matching, without directly distilling supervision from task success, which can leave a gap between chunk-level training objectives and closed-loop test-time behavior.

\subsection{On-Policy Distillation}
\label{ssec:related-OPD}

Prevalent action-chunking VLAs can be viewed as autoregressive (AR) sequence models at the chunk level, where each action chunk is conditioned on states induced by previously executed chunks.
This perspective naturally connects VLAs to knowledge distillation in large language models (LLMs), where recent work minimizes the divergence between teacher and student distributions~\citep{MiniLLM,GKD}.
In particular, a common formulation of on-policy distillation (OPD) minimizes the sequence-level reverse Kullback-Leibler divergence (KLD), thereby evaluating the student and teacher distributions at student-visited states~\citep{MiniLLM,DecouplingKLTraj}. 
This reverse KLD formulation alleviates the train-test distribution mismatch of supervised fine-tuning by training on student-generated sequences~\citep{MiniLLM,GKD}. 
Such on-policy supervision has been particularly effective for long-horizon generation and reasoning tasks, including mathematical reasoning~\citep{DecouplingKLTraj,SelfDistilledReasoner}.

Recent works extend this principle to continuous generative models, where DiffusionOPD~\citep{DiffusionOPD} and FlowOPD~\citep{FlowOPD} distill teacher knowledge along student-generated ODE \textit{trajectories} from noise to data. 
However, these ODE trajectory-level OPD still distills knowledge only within each action chunk. 
It does not explicitly capture the \textit{sequence} of executed chunks that determines final task success. 
Moreover, flow-based action experts do not directly provide clean-action likelihoods $\pi_\theta(\cdot\mid\bm{s}_t)$~\citep{FlowMatching}, making sequence-level OPD over action chunks difficult to apply directly. 
VLA-OPD~\citep{VLA-OPD} instead derives dense teacher-based distillation rewards on student rollouts, but requires closed-loop environment interaction and repeated teacher queries.
This work therefore aims to overcome these challenges, thereby enabling task-success signals to be incorporated across action-chunk sequences.
\section{Preliminary}
\label{sec:prelim}

\subsection{On-Policy Distillation}
\label{ssec:prelim-OPD}

On-policy distillation (OPD) minimizes the reverse Kullback-Leibler
divergence (KLD) between a model distribution $q_\theta$ and a target
distribution $p$ over model-generated sequences~\citep{MiniLLM}:
\begin{equation}
    \mathcal{L}_{\mathrm{OPD}}(\theta)
    = \mathbb{E}_{c \sim \mu,\,\bm{x} \sim q_\theta(\cdot \mid c)}
    \left[\log\frac{q_\theta(\bm{x}\mid c)}
                   {p(\bm{x}\mid c)}
    \right],
\label{eq:OPD-KL}
\end{equation}
where $c$ is the conditioning context and
$\bm{x}=(x_1,\ldots,x_T)$ is an autoregressively generated sequence.
A useful property of sequence-level reverse-KL optimization is that its
gradient separates the discrepancy at the current prediction from the effect
of that prediction on subsequent sequence-level discrepancy~\citep{MiniLLM}.
This decomposition is expressed through step-wise log-density ratios,
\begin{equation}
    r_t
    = \log\frac{p(z_t \mid c,\bm{z}_{<t})}
                {q_\theta(z_t \mid c,\bm{z}_{<t})},\qquad
    R_{t+1}
    = \sum_{k=t+1}^{T} r_k,
\label{eq:OPD-step-ratio}
\end{equation}
where $r_t$ represents the current-step discrepancy and $R_{t+1}$ accumulates
the discrepancy over subsequent steps.
For discrete autoregressive models such as LLMs, these terms are tractable
because the step-wise conditional distributions are explicitly available
from model logits. In contrast, continuous generative policies generally do
not provide tractable clean-sample densities, making the same formulation
difficult to apply directly.

\subsection{Drifting Models}
\label{ssec:prelim-drifting}

Drifting models~\citep{Drifting} are updated by a sample-space vector field induced by the discrepancy between a target distribution and a model distribution. Given a condition $\bm{s}$, let $p(a_t\mid\bm{s})$ and $q_{\theta}(a_t\mid\bm{s})$ denote the target and model distributions over data samples $a_t$, respectively.
Then, their reverse KLD is
\begin{equation}
    D_{\mathrm{KL}}\left[q_{\theta}(a_t\mid\bm{s})\,\|\,
                         p(a_t\mid\bm{s})\right]
    = \mathbb{E}_{a_t\sim q_{\theta}}
      \left[\log\frac{q_{\theta}(a_t\mid\bm{s})}
                        {p(a_t\mid\bm{s})}
      \right].
\label{eq:drifting-KL}
\end{equation}
The corresponding local descent direction in the data space is given by
the difference between their score functions:
\begin{equation}
    \bm{V}^{*}(a_t,\bm{s})
    = \nabla_{a_t}\log p(a_t\mid\bm{s})
    - \nabla_{a_t}\log q_{\theta}(a_t\mid\bm{s}).
\label{eq:drifting-score}
\end{equation}
When the densities are not directly available, they can be approximated
from samples using kernel density estimation (KDE)~\citep{MeanShift,AlignFewStep},
\begin{equation}
    p_h(a_t\mid\bm{s})
    = \mathbb{E}_{a_t^p \sim p(\cdot\mid\bm{s})}
    \left[K_h(a_t,a_t^{p})\right],\qquad
    q_{\theta,h}(a_t\mid\bm{s}) 
    = \mathbb{E}_{a_t^{q} \sim q_{\theta}(\cdot\mid\bm{s})}
    \left[K_h(a_t,a_t^{q})\right].
\label{eq:drifting-kde}
\end{equation}
where $K_h(\cdot,\cdot)$ denotes a kernel function with bandwidth $h>0$.
This yields the sample-based drifting field $\bm{V}_{\theta,h}(a_t,\bm{s})$:
\begin{equation}
    \begin{aligned}
        \bm{V}_{\theta,h}(a_t,\bm{s})
        =& \nabla_{a_t}\log p_h(a_t\mid\bm{s})
        - \nabla_{a_t}\log q_{\theta,h}(a_t\mid\bm{s}) \\
        =& \nabla_{a_t}\log
        \mathbb{E}_{a_t^{p} \sim p(\cdot\mid\bm{s})}
        \left[K_h(a_t,a_t^{p})\right]
        - \nabla_{a_t}\log
        \mathbb{E}_{a_t^{q} \sim q_{\theta}(\cdot\mid\bm{s})}
        \left[K_h(a_t,a_t^{q})\right] \\
        =& \frac{\mathbb{E}_{a_t^{p} \sim p(\cdot\mid\bm{s})}
            \left[\nabla_{a_t} K_h(a_t,a_t^{p})\right]
        }{\mathbb{E}_{a_t^{p} \sim p(\cdot\mid\bm{s})}
            \left[K_h(a_t,a_t^{p})\right]
        }
        - \frac{
            \mathbb{E}_{a_t^{q} \sim q_{\theta}(\cdot\mid\bm{s})}
            \left[\nabla_{a_t} K_h(a_t,a_t^{q})\right]
        }{\mathbb{E}_{a_t^{q} \sim q_{\theta}(\cdot\mid\bm{s})}
            \left[K_h(a_t,a_t^{q})\right]
        }.
    \end{aligned}
\label{eq:drifting-field}
\end{equation}
A drifting model $\bm{G}_{\theta}$ can then be learned by updating the current generated samples toward this field,
\begin{equation}
    \mathcal{L}_{\mathrm{drift}}(\theta)
    = \mathbb{E}_{\bm{s},\,a_t}
    \left[\left\|
    \bm{G}_{\theta}(a_t,\bm{s}) 
    - \mathrm{stopgrad}\left(\bm{G}_{\theta}(a_t,\bm{s}) 
    + \eta\bm{V}_{\theta,h}(a_t,\bm{s})\right)
    \right\|_2^2\right].
\label{eq:drifting-objective}
\end{equation}
For the full derivation, please refer to \Secref{appendix:drifting-rkl}.
\begin{figure}[t]
  \centering
   \includegraphics[width=\linewidth]{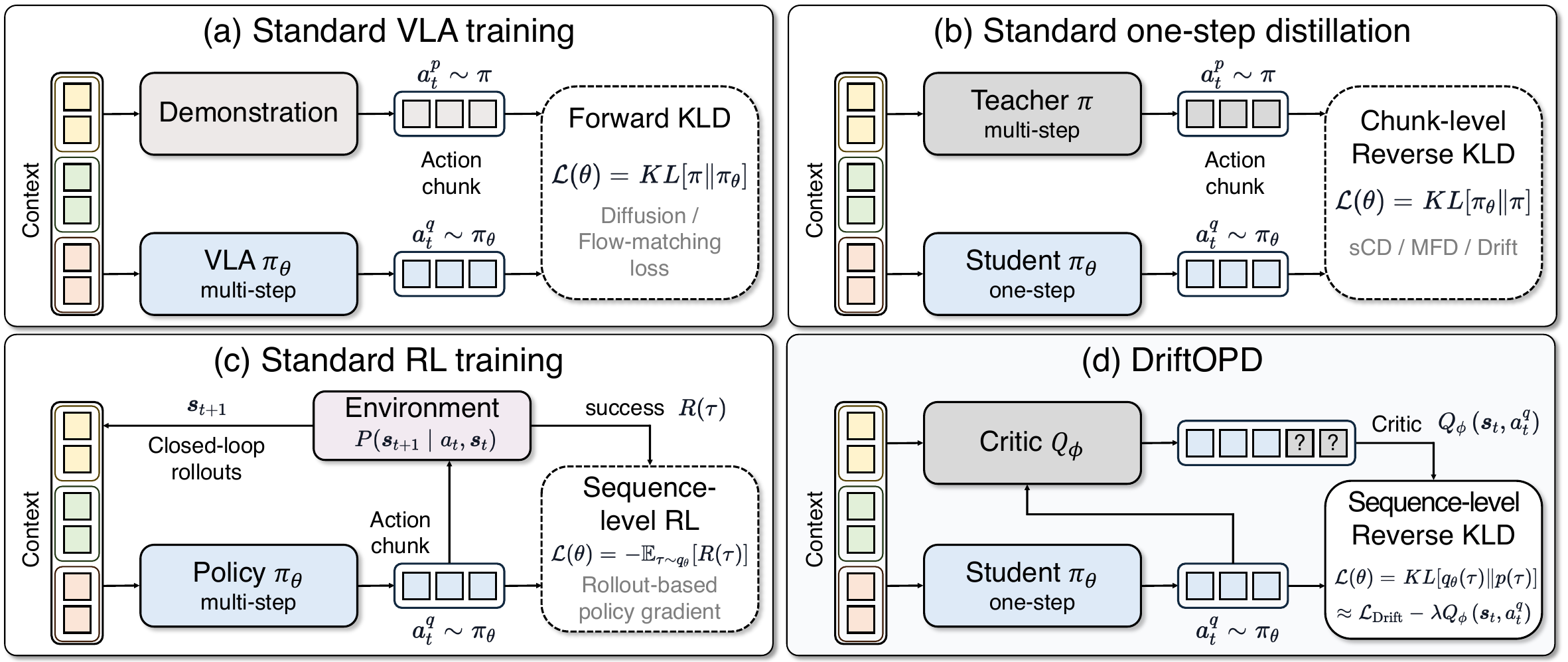}
    \vspace{-15pt}
   \caption{Comparison of VLA training paradigms. (a) Standard VLA training uses demonstration-supervised forward KLD, (b) existing one-step distillation optimizes chunk-level reverse KLD, and (c) RL relies on closed-loop rollouts, whereas (d) DriftOPD approximates sequence-level reverse KLD using a drifting objective and critic guidance without online rollouts.}
   \vspace{-12pt}
   \label{fig:main}
\end{figure}

\section{Method}
\label{sec:method}

\subsection{Sequence-Level Reverse KL Divergence in VLA}
\label{ssec:method-reverseKL}

We formulate VLA learning through the reverse KLD between the distributions over action \textit{sequences} induced by the model and target policies.
This sequence-level perspective is illustrated in \Figref{fig:main}.
Let $t$ index the action-chunk sequence, with $\bm{s}_t$ denoting the current VLA context (e.g., observation, instruction, and action state) and $a_t$ the corresponding action chunk.
Accordingly, let
$\tau=(\bm{s}_1,a_1,\bm{s}_2,\ldots,a_T,\bm{s}_{T+1})$
denote a trajectory composed of a sequence of action chunks and their corresponding conditioning contexts.
The trajectory distributions induced by the model policy $\pi_\theta$ and the target policy $\pi$ are respectively given by
\begin{equation}
    q_\theta(\tau)
    = \rho(\bm{s}_1)
      \prod_{t=1}^{T}
        \pi_\theta(a_t\mid\bm{s}_t)
        P(\bm{s}_{t+1}\mid\bm{s}_t,a_t),\qquad 
    p(\tau) 
    = \rho(\bm{s}_1)
      \prod_{t=1}^{T}
        \pi(a_t\mid\bm{s}_t)
        P(\bm{s}_{t+1}\mid\bm{s}_t,a_t),
\label{eq:vla-traj-dist}
\end{equation}
where $\rho(\bm{s}_1)$ denotes the initial context distribution and
$P(\bm{s}_{t+1}\mid\bm{s}_t,a_t)$ denotes the transition induced by executing $a_t$ under the current context.
Since the two trajectory distributions share the same initial context distribution and transition dynamics, their reverse KLD reduces to the cumulative policy discrepancy along model-generated trajectories (full derivation in \Secref{appendix:rkd-vla}):
\begin{equation}
    \begin{aligned}
        \mathcal{L}_\mathrm{RKL}(\theta)
        = D_{\mathrm{KL}}
            \left[q_\theta(\tau)\,\|\,p(\tau)
            \right] = \mathbb{E}_{\tau\sim q_\theta}
        \left[\log\frac{q_\theta(\tau)}
                       {p(\tau)}
        \right] 
        = \mathbb{E}_{\tau\sim q_\theta}
          \left[\sum_{t=1}^{T}\log
                \frac{\pi_\theta(a_t\mid\bm{s}_t)}
                     {\pi(a_t\mid\bm{s}_t)}
          \right].
    \end{aligned}
\label{eq:vla-reverse-kl}
\end{equation}

To expose the sequential structure of this objective, we define
$J_t(\bm{s}_t)$ as the reverse-KL cost-to-go from sequence index $t$.
Its Bellman recursion is
\begin{equation}
    \begin{aligned}
        J_t(\bm{s}_t)
        &=
        \underbrace{
        D_{\mathrm{KL}}
        \left[
            \pi_\theta(\cdot\mid\bm{s}_t)
            \,\|\,
            \pi(\cdot\mid\bm{s}_t)
        \right]
        }_{\textstyle \mathcal{L}_{\mathrm{Local}}}
        +
        \underbrace{
        \mathbb{E}_{\substack{
            a_t\sim\pi_\theta(\cdot\mid\bm{s}_t)\\
            \bm{s}_{t+1}\sim P(\cdot\mid\bm{s}_t,a_t)
        }}
        \left[
            J_{t+1}(\bm{s}_{t+1})
        \right]
        }_{\textstyle \mathcal{L}_{\mathrm{Future}}}.
    \end{aligned}
\label{eq:vla-single-long}
\end{equation}
Here, $\mathcal{L}_{\mathrm{Local}}$ captures the discrepancy between
the model and target distributions over the current action chunk,
whereas $\mathcal{L}_{\mathrm{Future}}$ captures how the current action
chunk affects the distributional discrepancy over subsequent action
chunks through the induced future contexts.
Thus, sequence-level reverse-KL in VLAs accounts for both the current
action-distribution discrepancy and its effect on future behavior, as in
language-model OPD~\citep{MiniLLM}.

\subsection{Single-Positive Local Term: Drifting Field}
\label{ssec:method-drifting}

\textbf{The local term} $\mathcal{L}_{\mathrm{Local}}$ in
\Eqref{eq:vla-single-long} requires minimizing the reverse KLD between
the model and target action-chunk distributions,
$D_{\mathrm{KL}}[
\pi_\theta(\cdot\mid\bm{s}_t)
\,\|\, 
\pi(\cdot\mid\bm{s}_t)]$.
However, unlike autoregressive LLMs, VLA action experts do not provide
explicit densities $\pi_\theta(\cdot\mid\bm{s}_t)$ over action chunks.
We therefore optimize a kernel-smoothed local reverse KLD using the
sample-space drifting field (\Secref{ssec:prelim-drifting}).
Let $K_h(a,a')=\kappa_h(a-a')$ be a normalized, nonnegative, differentiable density with bandwidth $h$, and define the kernel-smoothed distributions
as
\begin{equation}
    \pi_h(a\mid\bm{s}_t)
    = \mathbb{E}_{a^p\sim\pi(\cdot\mid\bm{s}_t)}
        \left[K_h(a,a^p)\right],\qquad
        \pi_{\theta,h}(a\mid\bm{s}_t)
    =\mathbb{E}_{a^q\sim\pi_\theta(\cdot\mid\bm{s}_t)}
        \left[K_h(a,a^q)\right].
\label{eq:local-kernel-density}
\end{equation}
The corresponding smoothed local objective is
\begin{equation}
    \mathcal{L}_{\mathrm{Local}}^{h}(\theta;\bm{s}_t)
    = D_{\mathrm{KL}}
        \left[\pi_{\theta,h}(\cdot\mid\bm{s}_t)\,\|\,
              \pi_h(\cdot\mid\bm{s}_t)
        \right].
\label{eq:local-smoothed-rkl}
\end{equation}
Its drifting field is defined by the difference between the target and
model scores,
\begin{equation}
    \bm{V}_{\theta,h}(a,\bm{s}_t)
    = \nabla_a\log\pi_h(a\mid\bm{s}_t)
    - \nabla_a\log\pi_{\theta,h}(a\mid\bm{s}_t).
\label{eq:local-drift-field}
\end{equation}

\begin{proposition}[Local reverse KL as drifting]
\label{prop:local-drift}
    Let
    $a_t^q=\bm{G}_{\theta}(\bm{z},\bm{s}_t)$ and
    $\widetilde a_t^q=a_t^q+\bm{\epsilon}_h$,
    where $\bm{z}\sim\mathcal{N}(\bm{0},\bm{I})$ and
    $\bm{\epsilon}_h\sim\kappa_h$.
    Define
    \begin{equation}
        \mathcal{L}_{\mathrm{Drift}}(\theta;\bm{s}_t)
        = \mathbb{E}_{\bm{z},\bm{\epsilon}_h}
            \left[\left\|
            \bm{G}_{\theta}(\bm{z},\bm{s}_t)
            - \mathrm{stopgrad}(a_t^q + \eta\,
                \bm{V}_{\theta,h}(\widetilde a_t^q,\bm{s}_t))
            \right\|_2^2\right].
    \label{eq:local-drift-loss}
    \end{equation}
    Then, for any $\eta>0$,
    \begin{equation}
        \nabla_{\theta}
        \mathcal{L}_{\mathrm{Local}}^{h}(\theta;\bm{s}_t)
        = \frac{1}{2\eta}
            \nabla_{\theta}
            \mathcal{L}_{\mathrm{Drift}}(\theta;\bm{s}_t).
        \label{eq:local-drift-equivalence}
    \end{equation}
\end{proposition}
\begin{proof}
    Using the pathwise derivative through
    $\widetilde a_t^q=
    \bm{G}_{\theta}(\bm{z},\bm{s}_t)+\bm{\epsilon}_h$,
    \begin{equation}
        \nabla_{\theta}\mathcal{L}_{\mathrm{Local}}^{h}
        = -\mathbb{E}_{\bm{z},\bm{\epsilon}_h}
            \left[(\partial_{\theta}\bm{G}_{\theta})^\top
                \bm{V}_{\theta,h}(\widetilde a_t^q,\bm{s}_t)
            \right].
    \end{equation}
    Since $a_t^q+\eta\bm{V}_{\theta,h}(\widetilde a_t^q,\bm{s}_t)$ is treated as a stop-gradient target,
    \begin{equation}
        \nabla_{\theta}\mathcal{L}_{\mathrm{Drift}}
        = -2\eta\,\mathbb{E}_{\bm{z},\bm{\epsilon}_h}
            \left[(\partial_{\theta}\bm{G}_{\theta})^\top
                \bm{V}_{\theta,h}(\widetilde a_t^q,\bm{s}_t)
            \right],
    \end{equation}
    which proves \Eqref{eq:local-drift-equivalence}. A detailed derivation is provided in \Secref{appendix:local-drift}.
\end{proof}

In our setting, the target policy is accessible only through offline
demonstrations rather than repeated queries at a fixed context
$\bm{s}_t$. For each demonstration context, we therefore use its single
recorded action chunk $a_t^p$ as the positive sample, following prior work~\citep{DriftQlearning,DriftWorld},
\begin{equation}
    \pi(\cdot\mid\bm{s}_t)
    \approx \delta_{a_t^p}(\cdot), \qquad
    \pi_h(a\mid\bm{s}_t) \approx K_h(a,a_t^p),
\label{eq:single-positive-target}
\end{equation}
where $\delta_{a_t^p}(\cdot)$ denotes the Dirac delta at $a_t^p$.
The drifting field then becomes
\begin{equation}
    \bm{V}_{\theta,h}(a,\bm{s}_t)
    = \nabla_a\log K_h(a,a_t^p)
    - \frac{\mathbb{E}_{a^q\sim\pi_\theta(\cdot\mid\bm{s}_t)}
            \left[\nabla_a K_h(a,a^q)\right]
           }{\mathbb{E}_{a^q\sim\pi_\theta(\cdot\mid\bm{s}_t)}
                \left[K_h(a,a^q)\right]
           }.
\label{eq:single-positive-drift}
\end{equation}
Hence, the local reverse-KL term can be optimized using a single positive
demonstration chunk together with samples from the current VLA policy.
A detailed analysis of the single-positive drifting field is provided in \Secref{appendix:mse}.

\subsection{Action-Value Future Term: Critic Surrogate}
\label{ssec:method-critic}

\textbf{The future term} $\mathcal{L}_{\mathrm{Future}}$ in \Eqref{eq:vla-single-long} accounts for how the current action chunk affects the subsequent trajectory through future observations. Following the Bellman recursion, we define the divergence-to-go after taking $a_t$ at $\bm{s}_t$ as
\begin{equation}
    Q_t^\theta(\bm{s}_t,a_t)
    = \mathbb{E}_{\bm{s}_{t+1}\sim P(\cdot\mid\bm{s}_t,a_t)}
        \left[J_{t+1}(\bm{s}_{t+1})\right],
\label{eq:future-q}
\end{equation}
such that the future term is
\begin{equation}
    \mathcal{L}_{\mathrm{Future}}(\theta;\bm{s}_t)
    = \mathbb{E}_{a_t\sim\pi_\theta(\cdot\mid\bm{s}_t)}
        \left[Q_t^\theta(\bm{s}_t,a_t)\right].
\label{eq:future-term}
\end{equation}
In the sequence-level gradient decomposition, $Q_t^\theta$ is evaluated under the current policy and treated as fixed with respect to $\theta$ when differentiating the contribution at sequence index $t$ (\Secref{appendix:rkd-vla}). 
Using $a_t=\bm{G}_\theta(\bm{z},\bm{s}_t)$, the resulting future gradient contribution is
\begin{equation}
    \bm{g}_{\mathrm{Future}}(\theta;\bm{s}_t)
    =
    \mathbb{E}_{\bm{z}}
    \left[
    (\partial_\theta\bm{G}_\theta)^\top
    \nabla_{a_t}Q_t^\theta(\bm{s}_t,a_t)
    \right].
\label{eq:future-gradient}
\end{equation}
Directly evaluating $Q_t^\theta$ requires rolling out the current policy and evaluating the remaining divergence along its future trajectory. 
Instead, we use a frozen action-value critic
$Q_\phi(\bm{s}_t,a_t)$ to approximate the unavailable action gradient,
\begin{equation}
    \nabla_{a_t}Q_t^\theta(\bm{s}_t,a_t)
    \approx
    -\nabla_{a_t}Q_\phi(\bm{s}_t,a_t).
\label{eq:critic-q-substitution}
\end{equation}
Note that $Q_t^\theta$ denotes a divergence cost, whereas $Q_\phi$
is a value score for which larger values indicate better actions. 
The critic parameters $\phi$ are fixed during policy optimization, and the critic is independent of the VLA policy parameterization.
Accordingly, the future gradient contribution is approximated as
\begin{equation}
    \bm{g}_{\mathrm{Future}}(\theta;\bm{s}_t)
    \approx
    -\mathbb{E}_{\bm{z}}
    \left[
    (\partial_\theta\bm{G}_\theta)^\top
    \nabla_{a_t}Q_\phi(\bm{s}_t,a_t)
    \right].
\label{eq:future-critic-gradient}
\end{equation}
Equivalently, this update can be obtained from the critic surrogate
\begin{equation}
    \mathcal{L}_{\mathrm{Critic}}(\theta;\bm{s}_t)
    = -\mathbb{E}_{\bm{z}}
        \left[Q_\phi\!\left(\bm{s}_t,\bm{G}_\theta(\bm{z},\bm{s}_t)\right)
        \right],
\label{eq:critic-surrogate}
\end{equation}
where gradients are propagated through the generated action chunk while $Q_\phi$ remains frozen. Thus, the critic provides a tractable surrogate for the long-horizon future gradient contribution without requiring closed-loop rollouts during policy optimization.
A detailed analysis and justification of the critic surrogate, together with its error analysis, are provided in \Secref{appendix:future-critic}.

\subsection{Objective Function}
\label{ssec:method-objective}

Combining the local and future terms, we introduce
\textbf{Drift On-Policy Distillation (DriftOPD)}, a unified policy
optimization framework that accounts for both local action-distribution
matching and long-horizon action quality. The overall framework is summarized in \Figref{fig:main}. Specifically, the local
reverse KLD term is realized by the drifting objective in
\Secref{ssec:method-drifting}, while the future gradient contribution is approximated by the frozen critic surrogate in \Secref{ssec:method-critic}. 
Our overall objective is
\begin{equation}
    \mathcal{L}_{\mathrm{DriftOPD}}(\theta)
    = \mathbb{E}_{(\bm{s}_t,a_t^p)\sim\mathcal{D}}
    \left[\mathcal{L}_{\mathrm{Drift}}(\theta;a_t^p,\bm{s}_t)
        + \lambda
          \mathcal{L}_{\mathrm{Critic}}(\theta;\bm{s}_t)
    \right],
\label{eq:overall-objective}
\end{equation}
where $\mathcal{D}$ denotes the offline demonstration dataset, and $\lambda>0$ absorbs the constant scale $\eta$ induced by the drifting step and controls the relative contribution of the future term.

\begin{wrapfigure}[13]{r}{0.53\textwidth}
\vspace{-2.2em}
\begin{minipage}{0.53\textwidth}
\centering
\begin{algorithm}[H]
\caption{Drift On-Policy Distillation (DriftOPD)}
\label{alg:gtpos}
\small
\begin{algorithmic}[1]
    \Require offline demonstrations $\mathcal{D}$, critic
    $Q_\phi$, weight $\lambda>0$
    \Ensure distilled VLA policy $\bm{G}_\theta$
    \For{$k=1,\dots,N_{\mathrm{train}}$}
    \State Sample $(\bm{s}_t,a_t^p)\sim\mathcal{D}$ and
    $\bm{z}\sim\mathcal{N}(\bm{0},\bm{I})$
    \State Generate $a_t^q=\bm{G}_\theta(\bm{z},\bm{s}_t)$
    \State Estimate $\widehat{\bm{V}}_{\theta,h}(\widetilde a_t^q,\bm{s}_t)$
    using $a_t^p$ and policy samples
    \State $\displaystyle
    \mathcal{L}_{\mathrm{Drift}}
    = \left\|a_t^q
      - \mathrm{stopgrad}\!
        \left(a_t^q+\eta\widehat{\bm{V}}_{\theta,h}
        \right)
    \right\|_2^2$
    \State $\displaystyle
    \mathcal{L}_{\mathrm{DriftOPD}}
    = \mathcal{L}_{\mathrm{Drift}}
    - \lambda
        Q_\phi(\bm{s}_t,a_t^q)$
    \State $\displaystyle
    \theta\leftarrow
    \theta-\alpha
    \nabla_\theta\mathcal{L}_{\mathrm{DriftOPD}}$
    \Comment{$Q_\phi$ is frozen}
    \EndFor
\end{algorithmic}
\end{algorithm}
\end{minipage}
\end{wrapfigure}
\textbf{Algorithm.}
The complete
training procedure of DriftOPD is summarized in
Algorithm~\ref{alg:gtpos}. 
For each demonstration pair $(\bm{s}_t,a_t^p)$, we sample $M$ action chunks $a_t^q=\bm{G}_{\theta}(\bm{z},\bm{s}_t)$ from the current VLA policy and estimate the drifting field in \Eqref{eq:single-positive-drift}.
The resulting drifting update locally aligns the model action distribution with the demonstrated target action while retaining the model-distribution score term of reverse-KL optimization. 
In parallel, the frozen critic $Q_\phi$ provides long-horizon guidance by propagating $\nabla_a Q_\phi(\bm{s}_t,a)$ through the generated action chunk.
The critic parameters $\phi$ remain fixed throughout policy optimization.

\textbf{Takeaway.}
Consequently, the two components of DriftOPD directly correspond to
the decomposition in \Eqref{eq:vla-single-long}. 
The drifting term $\mathcal{L}_{\mathrm{Drift}}$ realizes local action-distribution matching, while the critic term $\mathcal{L}_{\mathrm{Critic}}$ provides a tractable surrogate for the future gradient contribution..
This formulation enables on-policy distillation of VLA action experts directly from generated action chunks, without requiring explicit action densities or closed-loop rollouts during policy optimization.

\section{Experiment}
\label{sec:exp}

We experimentally evaluate and analyze the effectiveness of DriftOPD on robot manipulation tasks. We use the $\pi_{0.5}$~\citep{pi05}, GR00T N1.5/N1.6~\citep{Gr00tN1}, and ABot-M0~\citep{ABot-M0} vision language action models, covering fully finetuned models, LoRA~\citep{LoRA} finetuned models, models finetuned on selected VLM layers, and a JiT~\citep{JiT}-style model that directly predicts clean actions. For the given embodiment-specific finetuned models, we further train each VLA using the DriftOPD objective (\Eqref{eq:overall-objective}) for only 10k steps. For the action-value critic $Q_\phi(\bm{s}_t,a_t)$ in \Secref{ssec:method-critic}, we adapt VGAS~\citep{VGAS}, a Q-function trained solely on demonstration data to avoid requiring rollout data for critic training.
To compare one-step distillation in terms of performance, stability, and applicability, we use the multi-step teacher, its one-step prediction, and existing one-step distillation methods as baselines, including sCD~\citep{sCM}, MFD~\citep{MFD}, and Drift~\citep{DriftingDistill}.

\subsection{Simulation Experiments}
\label{ssec:exp-sim}

\begin{table}[t]
\centering
\caption{Average task success rates on RoboCasa365, reported separately for each evaluation suite. Each task uses 30 episodes, totaling 540, 480, and 480 per suite.}
\vspace{-6pt}
\setlength{\tabcolsep}{5pt}
\renewcommand{\arraystretch}{1.05}
    \resizebox{\textwidth}{!}{%
        \begin{tabular}{c c c c l c c c c c c}
            \specialrule{0.5pt}{0pt}{3pt}
            \multirow{2}{*}{\textbf{Model}} &
            \multirow{2}{*}{\textbf{\shortstack{Action\\prediction}}} &
            \multirow{2}{*}{\textbf{\shortstack{Fine-tuning\\strategy}}} &
            \multirow{2}{*}{\textbf{\shortstack{Teacher\\steps}}} &
            \multirow{2}{*}{\textbf{Suite}} &
            \multirow{2}{*}{\textbf{Teacher}} &
            \multicolumn{5}{c}{\textbf{1-step}} \\
            \cmidrule(lr){7-11}
            & & & & & &
            \textbf{Teacher} & \textbf{sCD} & \textbf{MFD} &
            \textbf{Drift} & \textbf{DriftOPD} \\
            \specialrule{0.7pt}{1pt}{3pt}
            \multirow{3}{*}{$\pi_{0.5}$} & \multirow{3}{*}{\shortstack{Flow\\matching}} & \multirow{3}{*}{Full FT} & \multirow{3}{*}{10} & \texttt{atomic\_seen} (18) & 39.6 & 31.30 & 28.33 & 34.44 & 32.22 & \textbf{38.33} \\
            & & & & \texttt{composite\_seen} (16) & 7.1 & 0.83 & 1.04 & 2.50 & 5.00 & \textbf{5.42} \\
            & & & & \texttt{composite\_unseen} (16) & 1.2 & 0.83 & 0.83 & 0.83 & 1.04 & \textbf{1.67} \\
            \specialrule{0.3pt}{1pt}{3pt}
            \multirow{3}{*}{$\pi_{0.5}$} & \multirow{3}{*}{\shortstack{Flow\\matching}} & \multirow{3}{*}{\shortstack{LoRA\\($r=16$)}} & \multirow{3}{*}{10} & \texttt{atomic\_seen} (18) & 39.6 & 31.30 & 29.26 & 28.33 & 34.07 & \textbf{39.26} \\
            & & & & \texttt{composite\_seen} (16) & 7.1 & 0.83 & 1.67 & 2.71 & 3.12 & \textbf{5.00} \\
            & & & & \texttt{composite\_unseen} (16) & 1.2 & 0.83 & 0.83 & 1.25 & \textbf{1.46} & \textbf{1.46} \\
            \specialrule{0.3pt}{1pt}{3pt}
            \multirow{3}{*}{GR00T N1.5} & \multirow{3}{*}{\shortstack{Flow\\matching}} & \multirow{3}{*}{Action expert} & \multirow{3}{*}{4} & \texttt{atomic\_seen} (18) & 50.7 & 42.96 & 39.63 & 43.52 & 43.89 & \textbf{45.37} \\
            & & & & \texttt{composite\_seen} (16) & 14.8 & 7.71 & 4.17 & \textbf{8.54} & 8.12 & \textbf{8.54} \\
            & & & & \texttt{composite\_unseen} (16) & 2.7 & 2.50 & 1.67 & 2.71 & 2.29 & \textbf{2.92} \\
            \specialrule{0.3pt}{1pt}{3pt}
            \multirow{3}{*}{GR00T N1.6} & \multirow{3}{*}{\shortstack{Flow\\matching}} & \multirow{3}{*}{\shortstack{Action head +\\top-4 VLM layers}} & \multirow{3}{*}{4} & \texttt{atomic\_seen} (18) & 51.1 & 44.81 & \textbf{51.67} & 32.96 & 47.22 & 49.44 \\
            & & & & \texttt{composite\_seen} (16) & 9.4 & 2.71 & 3.54 & 0.83 & 3.96 & \textbf{5.00} \\
            & & & & \texttt{composite\_unseen} (16) & 1.7 & 0.83 & 0.62 & \textbf{1.04} & 0.62 & \textbf{1.04} \\
            \specialrule{0.5pt}{1pt}{0pt}
        \end{tabular}%
    }
\vspace{-6pt}
\label{tab:robocasa365}
\end{table}

\begin{figure}[t]
    \centering
    \vspace{-3pt}
    \includegraphics[width=0.9\textwidth]{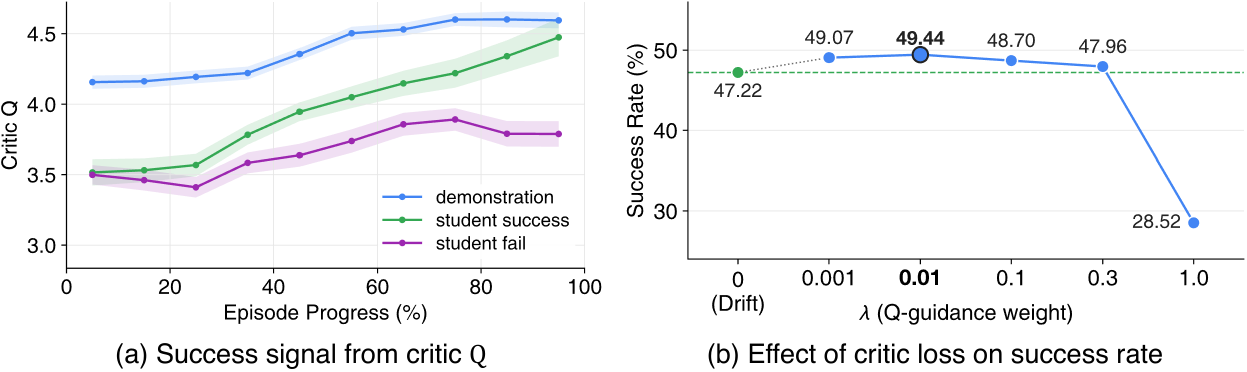}
    \vspace{-4pt}
    \caption{Analysis of critic guidance for sequence-level VLA policy optimization on RoboCasa365 with GR00T N1.6. (a) Critic $Q$ over episode progress for demonstrations, successful rollouts, and failed student rollouts. (b) Task success rate versus $\lambda$ on the \texttt{atomic\_seen} suite.}
    \vspace{-16pt}
    \label{fig:critic}
\end{figure}

Our simulation benchmarks include RoboCasa365~\citep{RoboCasa365} for large-scale household manipulation and RoboTwin 2.0~\citep{RoboTwin2} for bimanual manipulation under diverse domain variations. We additionally evaluate on LIBERO~\citep{Libero} and LIBERO-Plus~\citep{LiberoPlus} in \Secref{appendix:results}. All evaluations follow the official publicly released protocols for reproducibility. Further implementation details are provided in \Secref{appendix:implementation}.

\textbf{Main results.}
As shown in Table~\ref{tab:robocasa365} for RoboCasa365 and Table~\ref{tab:robotwin2} for RoboTwin 2.0, DriftOPD achieves task success rates that are mostly superior or comparable to sCD, MFD, and Drift. In several suites, it even surpasses the multi-step teacher despite using only one-step inference. Meanwhile, sCD and MFD sometimes underperform the one-step teacher prediction. We hypothesize that limited embodiment-specific data for VLAs exacerbate the difficulty of low-noise estimation, where the target distribution becomes sharply concentrated and score estimation is known to be unstable~\citep{LowNoise1,LowNoise2,LowNoise3}. This instability is particularly evident in the training collapse of MFD on ABot-M0, which directly predicts clean actions. In contrast, the single-positive drifting objectives of Drift and DriftOPD in \Eqref{eq:drifting-objective} may mitigate this issue through multi-bandwidth kernel density estimation~\citep{Drifting}.
DriftOPD further complements this robust chunk-level objective with sequence-level critic guidance, which accounts for the effect of the action chunk on future task progress and generally improves over chunk-level drifting distillation.

\textbf{Effect of critic guidance.}
To assess whether the critic provides a meaningful success-aware signal and its sensitivity to critic weighting, we analyze trajectory-level critic values and sweep $\lambda$. As shown in \Figref{fig:critic}a, the critic clearly separates successful from failed trajectories, supporting its use for sequence-level optimization. In \Figref{fig:critic}b, a broad range of $\lambda$ values ($0.001$-$0.3$) consistently improves success over the drifting-only objective ($\lambda=0$). An excessively large weight of $1.0$ causes performance collapse, likely because RMS normalization fixes the scale of the multi-bandwidth drifting loss~\citep{Drifting}. These two results support that the critic provides useful sequence-level guidance without requiring delicate tuning.

\textbf{GPU time to success.}
Unlike image or video generation, reducing action generation to a single inference step does not necessarily translate into faster task completion. A one-step policy may reduce success rate or require more policy calls because of noisier action sequences. We therefore measure GPU time to success, defined as the number of policy calls multiplied by the per-call latency, on 10 randomly selected RoboCasa365 \texttt{atomic\_seen} tasks with $\pi_{0.5}$ (\Figref{fig:gpu-time}). DriftOPD is $2.60\times$ faster than the 10-step teacher, thereby reducing GPU time to task success.

\begin{table}[t]
\centering
\caption{Average task success rates on RoboTwin 2.0, grouped by task horizon. Each task is evaluated over 30 episodes under paired environment settings. \texttt{short}, \texttt{medium}, and \texttt{long} horizons correspond to average trajectory lengths of $<150$, $150$-$279$, and $\geq280$ steps, respectively.}
\vspace{-6pt}
\setlength{\tabcolsep}{5pt}
\renewcommand{\arraystretch}{1.05}
    \resizebox{\textwidth}{!}{%
        \begin{tabular}{c c c c l c c c c c c}
            \specialrule{0.5pt}{0pt}{3pt}
            \multirow{2}{*}{\textbf{Model}} &
            \multirow{2}{*}{\textbf{\shortstack{Action\\objective}}} &
            \multirow{2}{*}{\textbf{\shortstack{Fine-tuning\\strategy}}} &
            \multirow{2}{*}{\textbf{\shortstack{Teacher\\steps}}} &
            \multirow{2}{*}{\textbf{Suite}} &
            \multirow{2}{*}{\textbf{Teacher}} &
            \multicolumn{5}{c}{\textbf{1-step}} \\
            \cmidrule(lr){7-11}
            & & & & & &
            \textbf{Teacher} & \textbf{sCD} & \textbf{MFD} &
            \textbf{Drift} & \textbf{DriftOPD} \\
            \specialrule{0.7pt}{1pt}{3pt}
            \multirow{3}{*}{$\pi_{0.5}$-LoRA} &
            \multirow{3}{*}{\shortstack{Flow\\matching}} &
            \multirow{3}{*}{\shortstack{LoRA\\($r=16$)}} &
            \multirow{3}{*}{10} &
            \texttt{short} (18) & 60.74 & 61.30 & 64.81 & 54.44 & 66.30 & \textbf{67.78} \\
            & & & & \texttt{medium} (21) & 67.30 & 56.35 & 61.11 & 48.25 & 61.90 & \textbf{64.13} \\
            & & & & \texttt{long} (11) & 55.15 & 40.61 & 36.97 & 16.36 & \textbf{48.48} & 46.67 \\
            \specialrule{0.3pt}{1pt}{3pt}
            \multirow{3}{*}{ABot-M0} &
            \multirow{3}{*}{\shortstack{Direct\\clean action}} &
            \multirow{3}{*}{\shortstack{Joint FT\\(VLM + action expert)}} &
            \multirow{3}{*}{10} &
            \texttt{short} (18) & 70.19 & 70.37 & 67.41 & 2.41 & \textbf{69.07} & 64.63 \\
            & & & & \texttt{medium} (21) & 73.97 & 67.78 & 66.83 & 4.13 & 65.71 & \textbf{68.10} \\
            & & & & \texttt{long} (11) & 36.97 & 36.97 & 30.61 & 0.91 & 35.76 & \textbf{38.79} \\
            \specialrule{0.5pt}{1pt}{0pt}
        \end{tabular}%
    }
\vspace{-5pt}
\label{tab:robotwin2}
\end{table}

\begin{figure}[t]
\centering
\begin{minipage}[t]{0.43\textwidth}
        \vspace{0pt}
        \centering
        \captionsetup{hypcap=false}
        \captionof{figure}{GPU time to success: policy calls $\times$ per-call latency (idle B200, batch 1) (Bars: total GPU time per success, including failed episodes; Diamonds (parenthesized): mean over successful episodes).}
    \vspace{-5pt}
        \includegraphics[width=\linewidth]{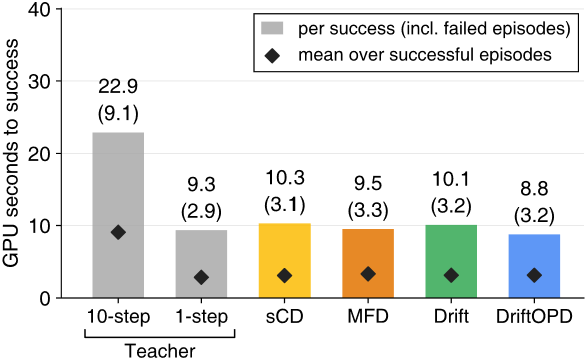}
        \label{fig:gpu-time}
    \vspace{-20pt}
\end{minipage}
\hfill
\begin{minipage}[t]{0.55\textwidth}
    \vspace{0pt}
    \centering
    \small
    \setlength{\tabcolsep}{6pt}
    \captionsetup{hypcap=false}
    \captionof{table}{Real-world evaluation in paired environments on unimanual pick-and-place and bimanual handover tasks. (parentheses: partial success rate (e.g., picking the object but failing to place or handover it)).}
    \vspace{-5pt}
    \renewcommand{\arraystretch}{1.1}
        \resizebox{\linewidth}{!}{%
            \begin{tabular}{l cccccc}
            \specialrule{0.5pt}{0pt}{3pt}
            \multirow{2}{*}{\textbf{Task}} &
            \multicolumn{1}{c}{\textbf{4-step}} &
            \multicolumn{5}{c}{\textbf{1-step}} \\
            \cmidrule(lr){3-7}
            & \textbf{Teacher} &
            \textbf{Teacher} &
            \textbf{sCD} &
            \textbf{MFD} &
            \textbf{Drift} &
            \textbf{DriftOPD} \\
            \specialrule{0.7pt}{1pt}{3pt}
            \multicolumn{7}{l}{\textit{\texttt{Unimanual (right arm)}}} \\
            Easy        & 93.3 & 40.0 & 0.0 (60.0) & 0.0 (0.0) & 86.7 & \textbf{93.3} \\
            Moderate    & 73.3 & 46.7 & 6.7 (20.0) & 0.0 (0.0) & \textbf{60.0} & \textbf{60.0} \\
            Hard        & 66.7 & 13.3 & 0.0 (0.0) & 0.0 (0.0) & 60.0 & \textbf{66.7} \\
            Challenging & 46.7 & 0.0  & 0.0 (0.0) & 0.0 (0.0) & \textbf{13.3} & \textbf{13.3} \\
            \cmidrule(lr){2-7}
            Avg & 70.0 & 25.0 & 1.7 & 0.0 & 55.0 & \textbf{58.3} \\
            \specialrule{0.3pt}{1pt}{3pt}
            \multicolumn{7}{l}{\textit{\texttt{Bimanual}}} \\
            \texttt{Handover (LR)} & 73.3 & 60.0 & 0.0 (0.0) & 0.0 (0.0) & 53.3 & \textbf{80.0} \\
            \texttt{Handover (RL)} & 86.7 & 53.3 & 0.0 (26.7) & 0.0 (20.0) & 40.0 & \textbf{73.3} \\
            \cmidrule(lr){2-7}
            \textit{Avg} & 80.0 & 56.7 & 0.0 & 0.0 & 46.7 & \textbf{76.7} \\
            \specialrule{0.5pt}{1pt}{0pt}
            \end{tabular}%
        }
        \label{tab:realworld-gr00t}
    \vspace{-20pt}
\end{minipage}
\end{figure}

\subsection{Real-World Experiments}
\label{ssec:exp-real}

For real-world evaluation, we use the MolmoAct 2 Research Kit\footnote{\url{https://i2rt.com/products/molmoact-2-research-kit}} from I2RT Robotics. 
We finetune the cross-embodiment pretrained GR00T N1.6 base model on 638 teleoperated episodes comprising 366k frames, and subsequently distill the resulting policy into one-step students for 10k training steps. 
We evaluate four unimanual pick-and-place tasks with increasing difficulty using the right YAM robot arm, as well as two bimanual pick-and-handover-and-place tasks using two YAM robot arms. As summarized in Table~\ref{tab:realworld-gr00t}, DriftOPD achieves the highest one-step success rates on both the unimanual and bimanual tasks. A detailed analysis of action oscillation is provided in \Secref{appendix:oscilation}, while the hardware setup and evaluation protocol are described in \Secref{appendix:implementation}.

\section{Conclusion}
\label{sec:conc}

We introduce DriftOPD, a sequence-level one-step distillation framework for VLAs. Starting from the reverse KL divergence over action-sequence distributions, we decompose sequence-level OPD into a local action-distribution term and a future-sequence term. 
DriftOPD realizes the local term with a single-positive drifting objective and approximates the future gradient contribution with a demonstration-trained critic, enabling rollout-free one-step policy optimization from offline demonstrations. 
Across diverse VLA configurations in simulation and real-world manipulation, DriftOPD generally outperforms existing one-step distillation baselines while achieving performance comparable to multi-step policies. 
Our results further suggest that limited embodiment-specific data can exacerbate low-noise estimation, while the drifting objective provides stable distribution matching and the critic supplies effective sequence-level guidance. 
Overall, DriftOPD provides an efficient and sequence-aware approach to one-step VLA distillation without online rollouts or a separate multi-step teacher.
Further discussion is provided in \Secref{appendix:limitation}.

\subsubsection*{Author Contributions}
Y. Jun and J.C. Ye contributed to the conceptualization and study design.
Research execution was carried out by Y. Jun.
For the robot experiments, J. Park set up the I2RT robot hardware, and S. Jin, Y. Kim, and S. Park developed the software for robot data acquisition and action execution.
Robot data collection was performed by Y. Jun, K. Choi, and Y. Kim through teleoperation using leader arms.
Real-robot execution and evaluation were conducted by Y. Jun and K. Choi.
The manuscript was written by Y. Jun and J.C. Ye, and all authors reviewed and approved the final manuscript.

\clearpage
\bibliography{bib}
\bibliographystyle{iclr2027_conference}

\clearpage
\appendix

\section{Theoretical Demonstration}
\label{appendix:theory}

\subsection{Drifting Field and Reverse KL}
\label{appendix:drifting-rkl}

We provide a derivation connecting the drifting field in \Secref{ssec:prelim-drifting} to the local descent direction of the reverse KLD. Throughout this section, we fix the conditioning context $\bm{s}$ and assume that $p(a_t\mid\bm{s})$ and $q_\theta(a_t\mid\bm{s})$ are positive and continuously differentiable, with sufficient decay at the boundary to justify integration by parts.

\paragraph{Reverse-KL descent in sample space.}
For a fixed condition $\bm{s}$, consider the reverse KLD
\begin{equation}
    \mathcal{F}(q_\theta)=D_{\mathrm{KL}}\!\left[q_\theta(a_t\mid\bm{s})\,\|\,p(a_t\mid\bm{s})\right]=\int q_\theta(a_t\mid\bm{s})\log\frac{q_\theta(a_t\mid\bm{s})}{p(a_t\mid\bm{s})}\,da_t.
\label{eq:app-rkl}
\end{equation}
To characterize a descent direction directly in the sample space, let $\bm{u}(a_t,\bm{s})$ be a smooth vector field and perturb each sample infinitesimally as
\begin{equation}
    T_\epsilon(a_t)=a_t+\epsilon\bm{u}(a_t,\bm{s}), \qquad q_{\theta,\epsilon}=(T_\epsilon)_{\#}q_\theta,
\label{eq:app-transport}
\end{equation}
where $(T_\epsilon)_{\#}q_\theta$ denotes the pushforward of $q_\theta$ through $T_\epsilon$. The induced density satisfies the continuity equation
\begin{equation}
    \left.\frac{\partial q_{\theta,\epsilon}}{\partial\epsilon}\right|_{\epsilon=0}=-\nabla_{a_t}\!\cdot\!\left(q_\theta(a_t\mid\bm{s})\bm{u}(a_t,\bm{s})\right).
\label{eq:app-continuity}
\end{equation}
Differentiating \Eqref{eq:app-rkl} along this perturbation gives
\begin{equation}
    \begin{aligned}
        \left.\frac{d}{d\epsilon}\mathcal{F}(q_{\theta,\epsilon})\right|_{\epsilon=0}
        &=\int \left.\frac{\partial q_{\theta,\epsilon}}{\partial\epsilon}\right|_{\epsilon=0}
        \left(\log\frac{q_\theta}{p}+1\right)da_t \\
        &=-\int \nabla_{a_t}\!\cdot(q_\theta\bm{u})\left(\log\frac{q_\theta}{p}+1\right)da_t \\
        &=\int q_\theta\,\bm{u}^{\top}\nabla_{a_t}\log\frac{q_\theta}{p}\,da_t \\
        &=-\mathbb{E}_{a_t\sim q_\theta(\cdot\mid\bm{s})}
        \left[\bm{u}(a_t,\bm{s})^{\top}\bm{V}^{*}(a_t,\bm{s})\right],
    \end{aligned}
\label{eq:app-rkl-directional}
\end{equation}
where
\begin{equation}
    \bm{V}^{*}(a_t,\bm{s})=\nabla_{a_t}\log p(a_t\mid\bm{s})-\nabla_{a_t}\log q_\theta(a_t\mid\bm{s}).
\label{eq:app-optimal-field}
\end{equation}
Here, the third equality follows from integration by parts. Choosing $\bm{u}=\bm{V}^{*}$ yields
\begin{equation}
    \left.\frac{d}{d\epsilon}\mathcal{F}(q_{\theta,\epsilon})\right|_{\epsilon=0}
    =-\mathbb{E}_{a_t\sim q_\theta(\cdot\mid\bm{s})}\!\left[\|\bm{V}^{*}(a_t,\bm{s})\|_2^2\right]\leq 0.
\label{eq:app-rkl-decrease}
\end{equation}
Therefore, the difference between the target and model scores defines a local descent direction of the reverse KLD in sample space.

\paragraph{KDE-based drifting field.}
When $p(a_t\mid\bm{s})$ and $q_\theta(a_t\mid\bm{s})$ cannot be evaluated explicitly, we replace them by the kernel-smoothed densities
\begin{equation}
    p_h(a_t\mid\bm{s})=\mathbb{E}_{a_t^p\sim p(\cdot\mid\bm{s})}[K_h(a_t,a_t^p)], \qquad q_{\theta,h}(a_t\mid\bm{s})=\mathbb{E}_{a_t^q\sim q_\theta(\cdot\mid\bm{s})}[K_h(a_t,a_t^q)].
\label{eq:app-kde}
\end{equation}
Assuming that differentiation and expectation can be interchanged, the score of the smoothed target distribution is
\begin{equation}
    \nabla_{a_t}\log p_h(a_t\mid\bm{s})
    =
    \frac{\mathbb{E}_{a_t^p\sim p(\cdot\mid\bm{s})}[\nabla_{a_t}K_h(a_t,a_t^p)]}
    {\mathbb{E}_{a_t^p\sim p(\cdot\mid\bm{s})}[K_h(a_t,a_t^p)]},
\label{eq:app-kde-p-score}
\end{equation}
and analogously,
\begin{equation}
    \nabla_{a_t}\log q_{\theta,h}(a_t\mid\bm{s})
    =
    \frac{\mathbb{E}_{a_t^q\sim q_\theta(\cdot\mid\bm{s})}[\nabla_{a_t}K_h(a_t,a_t^q)]}
    {\mathbb{E}_{a_t^q\sim q_\theta(\cdot\mid\bm{s})}[K_h(a_t,a_t^q)]}.
\label{eq:app-kde-q-score}
\end{equation}
Applying the reverse-KL descent result in \Eqref{eq:app-optimal-field} to $p_h$ and $q_{\theta,h}$ therefore gives
\begin{equation}
    \bm{V}_{\theta,h}(a_t,\bm{s})=\frac{\mathbb{E}_{a_t^p\sim p(\cdot\mid\bm{s})}[\nabla_{a_t}K_h(a_t,a_t^p)]}{\mathbb{E}_{a_t^p\sim p(\cdot\mid\bm{s})}[K_h(a_t,a_t^p)]}-\frac{\mathbb{E}_{a_t^q\sim q_\theta(\cdot\mid\bm{s})}[\nabla_{a_t}K_h(a_t,a_t^q)]}{\mathbb{E}_{a_t^q\sim q_\theta(\cdot\mid\bm{s})}[K_h(a_t,a_t^q)]},
\label{eq:app-kde-field}
\end{equation}
which recovers the sample-based drifting field in \Eqref{eq:drifting-field}.

For example, with a Gaussian kernel $K_h(a,a')\propto\exp(-\|a-a'\|_2^2/(2h^2))$, we have
\begin{equation}
    \nabla_a K_h(a,a')=K_h(a,a')\frac{a'-a}{h^2}.
\label{eq:app-gaussian-kernel-grad}
\end{equation}
Hence, given target samples $\{a_{t,i}^{p}\}_{i=1}^{N_p}$ and model samples $\{a_{t,j}^{q}\}_{j=1}^{N_q}$, the empirical drifting field can be written as
\begin{equation}
    \widehat{\bm{V}}_{\theta,h}(a_t,\bm{s})=\frac{1}{h^2}\left(\sum_{i=1}^{N_p}w_i^p a_{t,i}^p-\sum_{j=1}^{N_q}w_j^q a_{t,j}^q\right),
\label{eq:app-empirical-field}
\end{equation}
where
\begin{equation}
    w_i^p=\frac{K_h(a_t,a_{t,i}^p)}{\sum_{k=1}^{N_p}K_h(a_t,a_{t,k}^p)}, \qquad w_j^q=\frac{K_h(a_t,a_{t,j}^q)}{\sum_{k=1}^{N_q}K_h(a_t,a_{t,k}^q)}.
\label{eq:app-kde-weights}
\end{equation}
Thus, the KDE score difference has the intuitive form of attraction toward nearby target samples and repulsion from nearby model samples.

\paragraph{Drifting objective as reverse-KL optimization.}
We finally show why moving generated samples toward the drifting field yields the gradient of the smoothed reverse KLD. Let $a_t=G_\theta(z,\bm{s})$, where $z$ is sampled from the base distribution, and let $\tilde{a}_t=a_t+\epsilon_h$ with $\epsilon_h\sim\kappa_h$. Then $\tilde{a}_t$ follows the kernel-smoothed model distribution $q_{\theta,h}(\cdot\mid\bm{s})$. Define
\begin{equation}
    \mathcal{L}_{\mathrm{RKL}}^{h}(\theta;\bm{s})=D_{\mathrm{KL}}\!\left[q_{\theta,h}(\cdot\mid\bm{s})\,\|\,p_h(\cdot\mid\bm{s})\right].
\label{eq:app-smoothed-rkl}
\end{equation}
Using the reparameterization $\tilde{a}_t=G_\theta(z,\bm{s})+\epsilon_h$, its gradient is
\begin{equation}
    \nabla_\theta\mathcal{L}_{\mathrm{RKL}}^{h}
    =-\mathbb{E}_{z,\epsilon_h}\!\left[\left(\partial_\theta G_\theta(z,\bm{s})\right)^\top\bm{V}_{\theta,h}(\tilde{a}_t,\bm{s})\right].
\label{eq:app-smoothed-rkl-grad}
\end{equation}
To see this explicitly, differentiating the reparameterized objective gives
\begin{equation}
\begin{aligned}
    \nabla_\theta\mathcal{L}_{\mathrm{RKL}}^{h}
    &=\mathbb{E}_{z,\epsilon_h}\!\left[\partial_\theta\log q_{\theta,h}(\tilde{a}_t\mid\bm{s})+\left(\partial_\theta G_\theta\right)^\top\nabla_{\tilde{a}_t}\log\frac{q_{\theta,h}(\tilde{a}_t\mid\bm{s})}{p_h(\tilde{a}_t\mid\bm{s})}\right] \\
    &=\mathbb{E}_{z,\epsilon_h}\!\left[\left(\partial_\theta G_\theta\right)^\top\left(\nabla_{\tilde{a}_t}\log q_{\theta,h}-\nabla_{\tilde{a}_t}\log p_h\right)\right],
\end{aligned}
\label{eq:app-pathwise-rkl}
\end{equation}
where the explicit density derivative vanishes because
\begin{equation}
    \mathbb{E}_{\tilde{a}_t\sim q_{\theta,h}}\!\left[\partial_\theta\log q_{\theta,h}(\tilde{a}_t\mid\bm{s})\right]=\int \partial_\theta q_{\theta,h}(a\mid\bm{s})\,da=\partial_\theta 1=0.
\label{eq:app-score-identity}
\end{equation}
Substituting the definition of $\bm{V}_{\theta,h}$ then yields \Eqref{eq:app-smoothed-rkl-grad}.

Now define the drifting objective with a stop-gradient target,
\begin{equation}
    \mathcal{L}_{\mathrm{drift}}(\theta;\bm{s})=\mathbb{E}_{z,\epsilon_h}\!\left[\left\|G_\theta(z,\bm{s})-\mathrm{stopgrad}\!\left(a_t+\eta\bm{V}_{\theta,h}(\tilde{a}_t,\bm{s})\right)\right\|_2^2\right].
\label{eq:app-drift-loss}
\end{equation}
Since the target inside $\mathrm{stopgrad}(\cdot)$ is treated as constant during differentiation and $a_t=G_\theta(z,\bm{s})$ numerically at the current iteration,
\begin{equation}
    \begin{aligned}
        \nabla_\theta\mathcal{L}_{\mathrm{drift}}
        &=2\mathbb{E}_{z,\epsilon_h}\!\left[\left(\partial_\theta G_\theta\right)^\top\left(G_\theta-\mathrm{stopgrad}\!\left(a_t+\eta\bm{V}_{\theta,h}\right)\right)\right] \\
        &=-2\eta\,\mathbb{E}_{z,\epsilon_h}\!\left[\left(\partial_\theta G_\theta\right)^\top\bm{V}_{\theta,h}(\tilde{a}_t,\bm{s})\right] \\
        &=2\eta\,\nabla_\theta\mathcal{L}_{\mathrm{RKL}}^{h}(\theta;\bm{s}).
    \end{aligned}
\label{eq:app-drift-rkl-equivalence}
\end{equation}
Therefore, for any $\eta>0$, minimizing the drifting objective follows the same parameter-space descent direction as minimizing the KDE-smoothed reverse KLD, up to the positive scaling factor $2\eta$.

\subsection{Sequence-Level Reverse KL Divergence in VLA}
\label{appendix:rkd-vla}

We provide the full derivation of the sequence-level reverse-KL decomposition used in \Secref{ssec:method-reverseKL}. We assume that the model and target policies share the same initial context distribution $\rho$ and transition dynamics $P$, and that $\pi_\theta(\cdot\mid\bm{s}_t)$ is absolutely continuous with respect to $\pi(\cdot\mid\bm{s}_t)$ whenever an unsmoothed density ratio is used. All required exchanges of differentiation and expectation are assumed valid.

From \Eqref{eq:vla-traj-dist}, the ratio between the model and target trajectory distributions is
\begin{equation}
    \frac{q_\theta(\tau)}{p(\tau)}
    =
    \prod_{t=1}^{T}
    \frac{\pi_\theta(a_t\mid\bm{s}_t)}
    {\pi(a_t\mid\bm{s}_t)},
\label{eq:appendix-traj-ratio}
\end{equation}
because the common factors $\rho(\bm{s}_1)$ and $P(\bm{s}_{t+1}\mid\bm{s}_t,a_t)$ cancel. Taking the logarithm and expectation under $q_\theta$ gives
\begin{equation}
    \mathcal{L}_{\mathrm{RKL}}(\theta)
    =
    \mathbb{E}_{\tau\sim q_\theta}
    \left[
    \sum_{t=1}^{T}
    \log
    \frac{\pi_\theta(a_t\mid\bm{s}_t)}
    {\pi(a_t\mid\bm{s}_t)}
    \right].
\label{eq:appendix-seq-rkl}
\end{equation}

To derive the Bellman recursion, define the reverse-KL cost-to-go as
\begin{equation}
    J_t(\bm{s}_t)
    =
    \mathbb{E}_{\tau\sim q_\theta}
    \left[
    \left.
    \sum_{k=t}^{T}
    \log
    \frac{\pi_\theta(a_k\mid\bm{s}_k)}
    {\pi(a_k\mid\bm{s}_k)}
    \,\right|\,\bm{s}_t
    \right],
    \qquad
    J_{T+1}\equiv0.
\label{eq:appendix-cost-to-go}
\end{equation}
Conditioning first on $a_t$ and then on $\bm{s}_{t+1}$ yields
\begin{equation}
    \begin{aligned}
        J_t(\bm{s}_t)
        &=
        \mathbb{E}_{a_t\sim\pi_\theta(\cdot\mid\bm{s}_t)}
        \left[
        \log\frac{\pi_\theta(a_t\mid\bm{s}_t)}
        {\pi(a_t\mid\bm{s}_t)}
        +
        \mathbb{E}_{\bm{s}_{t+1}\sim P(\cdot\mid\bm{s}_t,a_t)}
        [J_{t+1}(\bm{s}_{t+1})]
        \right] \\
        &=
        D_{\mathrm{KL}}
        \left[
        \pi_\theta(\cdot\mid\bm{s}_t)
        \,\|\,\pi(\cdot\mid\bm{s}_t)
        \right]
        +
        \mathbb{E}_{\substack{
        a_t\sim\pi_\theta(\cdot\mid\bm{s}_t)\\
        \bm{s}_{t+1}\sim P(\cdot\mid\bm{s}_t,a_t)}}
        [J_{t+1}(\bm{s}_{t+1})],
    \end{aligned}
\label{eq:appendix-bellman}
\end{equation}
which recovers \Eqref{eq:vla-single-long}. In particular, $\mathcal{L}_{\mathrm{RKL}}(\theta)=\mathbb{E}_{\bm{s}_1\sim\rho}[J_1(\bm{s}_1)]$.

We next derive the corresponding gradient decomposition. Let $d_t^\theta$ denote the marginal distribution of $\bm{s}_t$ induced by $q_\theta$. Differentiating \Eqref{eq:appendix-seq-rkl} gives
\begin{equation}
    \nabla_\theta\mathcal{L}_{\mathrm{RKL}}
    =
    \sum_{t=1}^{T}
    \mathbb{E}_{\tau\sim q_\theta}
    \left[
    \nabla_\theta\log\pi_\theta(a_t\mid\bm{s}_t)
    \sum_{k=t}^{T}
    \log
    \frac{\pi_\theta(a_k\mid\bm{s}_k)}
    {\pi(a_k\mid\bm{s}_k)}
    \right].
\label{eq:appendix-seq-gradient}
\end{equation}
To obtain \Eqref{eq:appendix-seq-gradient}, the explicit derivative of the log-ratio sum vanishes in expectation by the score identity, $\mathbb{E}_{a_t\sim\pi_\theta}[\nabla_\theta\log\pi_\theta(a_t\mid\bm{s}_t)]=0$, while terms preceding $t$ vanish by causality. Conditioning the remaining return on $(\bm{s}_t,a_t)$ gives the current log-ratio plus the divergence-to-go
\begin{equation}
    Q_t^\theta(\bm{s}_t,a_t)
    =
    \mathbb{E}_{\bm{s}_{t+1}\sim P(\cdot\mid\bm{s}_t,a_t)}
    [J_{t+1}(\bm{s}_{t+1})].
\label{eq:appendix-future-q}
\end{equation}
Therefore,
\begin{equation}
    \nabla_\theta\mathcal{L}_{\mathrm{RKL}}
    =
    \sum_{t=1}^{T}
    \mathbb{E}_{\bm{s}_t\sim d_t^\theta}
    \left[
    \nabla_\theta\mathcal{L}_{\mathrm{Local}}(\theta;\bm{s}_t)
    +
    \nabla_\theta
    \mathbb{E}_{a_t\sim\pi_\theta(\cdot\mid\bm{s}_t)}
    \left[
    \mathrm{stopgrad}\!\left(Q_t^\theta(\bm{s}_t,a_t)\right)
    \right]
    \right].
\label{eq:appendix-local-future-gradient}
\end{equation}
Here, $Q_t^\theta$ is evaluated at the current policy but held fixed when differentiating the contribution of $a_t$. The dependence of future actions on $\theta$ is accounted for by their corresponding terms at later sequence indices. Using $a_t=\bm{G}_\theta(\bm{z},\bm{s}_t)$, the future component admits the pathwise form
\begin{equation}
    \bm{g}_{\mathrm{Future}}(\theta;\bm{s}_t)
    :=
    \mathbb{E}_{\bm{z}}
    \left[
    (\partial_\theta\bm{G}_\theta)^\top
    \nabla_{a_t}Q_t^\theta(\bm{s}_t,a_t)
    \right].
\label{eq:appendix-future-pathwise}
\end{equation}

\paragraph{Offline context substitution.}
The exact gradient in \Eqref{eq:appendix-local-future-gradient} is evaluated under the model-induced context distribution $d_t^\theta$, whereas DriftOPD uses contexts from the offline demonstrations. Let $d_t^\pi$ denote the target-policy context distribution and let $\bm{g}_t(\bm{s}_t;\theta)$ denote the quantity inside the brackets of \Eqref{eq:appendix-local-future-gradient}. If $\|\bm{g}_t(\bm{s}_t;\theta)\|_2\le M$ for all $t$ and $\bm{s}_t$, then
\begin{equation}
    \left\|
    \sum_{t=1}^{T}
    \left(
    \mathbb{E}_{d_t^\theta}
    -
    \mathbb{E}_{d_t^\pi}
    \right)
    [\bm{g}_t]
    \right\|_2
    \le
    M\sum_{t=1}^{T}
    \sqrt{
    2D_{\mathrm{KL}}
    \left[
    q_\theta^{1:t}\,\|\,p^{1:t}
    \right]
    }
    \le
    TM\sqrt{2\mathcal{L}_{\mathrm{RKL}}(\theta)},
\label{eq:appendix-occupancy-bound}
\end{equation}
where $q_\theta^{1:t}$ and $p^{1:t}$ denote the corresponding trajectory-prefix distributions. The first inequality follows from Pinsker's inequality and the contraction of KLD under marginalization, and the second from $D_{\mathrm{KL}}[q_\theta^{1:t}\|p^{1:t}]\le D_{\mathrm{KL}}[q_\theta\|p]$. Thus, replacing model-generated contexts by demonstration contexts introduces an occupancy error that vanishes as the model trajectory distribution approaches the target distribution, although the bound can be loose when the two policies are far apart.

\subsection{Local Term and Drifting Field}
\label{appendix:local-drift}

We next derive the drifting update for the local term without assuming a particular kernel family. Let $K_h(a,a')=\kappa_h(a-a')$, where $\kappa_h$ is a normalized, nonnegative, differentiable density. The smoothed target and model distributions in \Eqref{eq:local-kernel-density} can equivalently be written as convolutions, $\pi_h=\pi*\kappa_h$ and $\pi_{\theta,h}=\pi_\theta*\kappa_h$. Accordingly, if $a_t=\bm{G}_\theta(\bm{z},\bm{s}_t)$ and $\bm{\epsilon}_h\sim\kappa_h$, then $\widetilde a_t=a_t+\bm{\epsilon}_h$ follows $\pi_{\theta,h}(\cdot\mid\bm{s}_t)$.

\paragraph{KDE-based score estimation.}
Given samples $\{a_i^p\}_{i=1}^{N_p}$ from $\pi(\cdot\mid\bm{s}_t)$ and $\{a_j^q\}_{j=1}^{N_q}$ from $\pi_\theta(\cdot\mid\bm{s}_t)$, the empirical KDEs are
\begin{equation}
    \widehat{\pi}_h(a\mid\bm{s}_t)
    =
    \frac{1}{N_p}\sum_{i=1}^{N_p}K_h(a,a_i^p),
    \qquad
    \widehat{\pi}_{\theta,h}(a\mid\bm{s}_t)
    =
    \frac{1}{N_q}\sum_{j=1}^{N_q}K_h(a,a_j^q).
\label{eq:appendix-empirical-kde}
\end{equation}
Their score estimates follow directly by differentiating the KDE and dividing by its density,
\begin{equation}
    \nabla_a\log\widehat{\pi}_h(a\mid\bm{s}_t)
    =
    \frac{\sum_{i=1}^{N_p}\nabla_aK_h(a,a_i^p)}
    {\sum_{i=1}^{N_p}K_h(a,a_i^p)},
    \qquad
    \nabla_a\log\widehat{\pi}_{\theta,h}(a\mid\bm{s}_t)
    =
    \frac{\sum_{j=1}^{N_q}\nabla_aK_h(a,a_j^q)}
    {\sum_{j=1}^{N_q}K_h(a,a_j^q)}.
\label{eq:appendix-kde-score}
\end{equation}
Thus, the population field in \Eqref{eq:local-drift-field} can be estimated entirely from samples. For fixed $h$, the empirical density and its derivative converge to their population counterparts under standard regularity conditions, and hence the score estimate converges whenever the limiting density is positive. With a vanishing bandwidth, standard KDE conditions further recover the unsmoothed score; for a Gaussian kernel in $d$ dimensions, a sufficient condition is $h\rightarrow0$ and $Nh^{d+2}\rightarrow\infty$~\citep{SmoothingTechniques,DensitiyEstimation}.

\newtheorem*{propdriftrestate}{Proposition~\ref{prop:local-drift}}

\begin{propdriftrestate}[Local reverse KL as drifting, restated]
\label{prop:appendix-local-drift}
    Let $a_t=\bm{G}_{\theta}(\bm{z},\bm{s}_t)$ and $\widetilde a_t=a_t+\bm{\epsilon}_h$, where $\bm{z}\sim\mathcal{N}(\bm{0},\bm{I})$ and $\bm{\epsilon}_h\sim\kappa_h$ are independent. Define
    \begin{equation}
        \mathcal{L}_{\mathrm{drift}}(\theta;\bm{s}_t)
        =
        \mathbb{E}_{\bm{z},\bm{\epsilon}_h}
        \left[
        \left\|
        \bm{G}_{\theta}(\bm{z},\bm{s}_t)
        -
        \mathrm{stopgrad}
        \left(
        a_t+\eta\bm{V}_{\theta,h}(\widetilde a_t,\bm{s}_t)
        \right)
        \right\|_2^2
        \right].
    \label{eq:appendix-drift-loss}
    \end{equation}
    Then, for any $\eta>0$,
    \begin{equation}
        \nabla_\theta\mathcal{L}_{\mathrm{Local}}^h(\theta;\bm{s}_t)
        =
        \frac{1}{2\eta}
        \nabla_\theta\mathcal{L}_{\mathrm{drift}}(\theta;\bm{s}_t).
    \label{eq:appendix-drift-equivalence}
    \end{equation}
\end{propdriftrestate}

\begin{proof}
    Since $\widetilde a_t=\bm{G}_\theta(\bm{z},\bm{s}_t)+\bm{\epsilon}_h$ follows $\pi_{\theta,h}(\cdot\mid\bm{s}_t)$,
    \begin{equation}
        \mathcal{L}_{\mathrm{Local}}^h
        =
        \mathbb{E}_{\bm{z},\bm{\epsilon}_h}
        \left[
        \log\pi_{\theta,h}(\widetilde a_t\mid\bm{s}_t)
        -
        \log\pi_h(\widetilde a_t\mid\bm{s}_t)
        \right].
    \label{eq:appendix-smoothed-kl-reparam}
    \end{equation}
    Differentiating with respect to $\theta$ gives
    \begin{equation}
        \begin{aligned}
            \nabla_\theta\mathcal{L}_{\mathrm{Local}}^h
            &=
            \mathbb{E}_{\bm{z},\bm{\epsilon}_h}
            \left[
            \partial_\theta\log\pi_{\theta,h}(\widetilde a_t\mid\bm{s}_t)
            \right]
            +
            \mathbb{E}_{\bm{z},\bm{\epsilon}_h}
            \left[
            (\partial_\theta\bm{G}_\theta)^\top
            \nabla_{\widetilde a_t}
            \log
            \frac{\pi_{\theta,h}(\widetilde a_t\mid\bm{s}_t)}
            {\pi_h(\widetilde a_t\mid\bm{s}_t)}
            \right] \\
            &=
            -\mathbb{E}_{\bm{z},\bm{\epsilon}_h}
            \left[
            (\partial_\theta\bm{G}_\theta)^\top
            \bm{V}_{\theta,h}(\widetilde a_t,\bm{s}_t)
            \right].
        \end{aligned}
    \label{eq:appendix-local-pathwise}
    \end{equation}
    The first term vanishes by the score identity,
    \begin{equation}
        \mathbb{E}_{\widetilde a_t\sim\pi_{\theta,h}}
        \left[
        \partial_\theta\log\pi_{\theta,h}(\widetilde a_t\mid\bm{s}_t)
        \right]
        =
        \int\partial_\theta\pi_{\theta,h}(a\mid\bm{s}_t)\,da
        =
        0.
    \label{eq:appendix-score-zero}
    \end{equation}
    Now define $\widehat a_t=a_t+\eta\bm{V}_{\theta,h}(\widetilde a_t,\bm{s}_t)$. Since $\widehat a_t$ is treated as a stop-gradient target,
    \begin{equation}
        \begin{aligned}
            \nabla_\theta\mathcal{L}_{\mathrm{drift}}
            &=
            2\mathbb{E}_{\bm{z},\bm{\epsilon}_h}
            \left[
            (\partial_\theta\bm{G}_\theta)^\top
            \left(
            \bm{G}_\theta(\bm{z},\bm{s}_t)
            -
            \mathrm{stopgrad}(\widehat a_t)
            \right)
            \right] \\
            &=
            -2\eta
            \mathbb{E}_{\bm{z},\bm{\epsilon}_h}
            \left[
            (\partial_\theta\bm{G}_\theta)^\top
            \bm{V}_{\theta,h}(\widetilde a_t,\bm{s}_t)
            \right].
        \end{aligned}
    \label{eq:appendix-drift-gradient}
    \end{equation}
    Comparing \Eqref{eq:appendix-local-pathwise} and \Eqref{eq:appendix-drift-gradient} proves \Eqref{eq:appendix-drift-equivalence}. The derivation uses only that $\kappa_h$ is a differentiable probability kernel and does not require a Gaussian kernel.
\end{proof}

\paragraph{Single-positive target.}
For a demonstration context $\bm{s}_t$, only one target action chunk $a_t^p$ is observed. Replacing the target conditional by $\delta_{a_t^p}$ gives $\pi_h(a\mid\bm{s}_t)=K_h(a,a_t^p)$ and
\begin{equation}
\bm{V}_{\theta,h}(a,\bm{s}_t)
=
\nabla_a\log K_h(a,a_t^p)
-
\frac{
\mathbb{E}_{a^q\sim\pi_\theta(\cdot\mid\bm{s}_t)}
[\nabla_aK_h(a,a^q)]
}{
\mathbb{E}_{a^q\sim\pi_\theta(\cdot\mid\bm{s}_t)}
[K_h(a,a^q)]
},
\label{eq:appendix-single-positive-field}
\end{equation}
which recovers \Eqref{eq:single-positive-drift}. The first term attracts the generated action toward the demonstrated action, while the second retains the model-distribution score required by reverse-KL optimization.

The single-positive construction also has a direct relation to the full smoothed target distribution. Assume $K_h(a,a')>0$ and define
\begin{equation}
    \mathcal{U}_h(\theta;\bm{s}_t)
    =
    \mathbb{E}_{a_t^p\sim\pi(\cdot\mid\bm{s}_t)}
    \left[
    D_{\mathrm{KL}}
    \left[
    \pi_{\theta,h}(\cdot\mid\bm{s}_t)
    \,\|\,K_h(\cdot,a_t^p)
    \right]
    \right].
\label{eq:appendix-single-positive-objective}
\end{equation}
Since $\pi_h(a\mid\bm{s}_t)=\mathbb{E}_{a_t^p}[K_h(a,a_t^p)]$,
\begin{equation}
    \mathcal{U}_h-\mathcal{L}_{\mathrm{Local}}^h
    =
    \mathbb{E}_{a\sim\pi_{\theta,h}}
    \left[
    \log
    \mathbb{E}_{a_t^p}
    K_h(a,a_t^p)
    -
    \mathbb{E}_{a_t^p}
    \log K_h(a,a_t^p)
    \right]
    \ge0,
\label{eq:appendix-single-positive-gap}
\end{equation}
where the inequality follows from Jensen's inequality. Thus, averaging the single-positive objective over demonstration actions upper-bounds the reverse KLD to the full kernel-smoothed target distribution. In particular, the gap vanishes when the target conditional is deterministic.

\paragraph{Finite-sample implementation.}
Proposition~\ref{prop:appendix-local-drift} is an exact population identity when the field is evaluated at the kernel-perturbed sample $\widetilde a_t$ using the population scores. In practice, \Eqref{eq:appendix-kde-score} is estimated from a finite set of policy samples, and Algorithm~\ref{alg:gtpos} evaluates the resulting field at the generated action. These introduce, respectively, finite-sample KDE error and an evaluation-point approximation. Under standard KDE regularity conditions the former vanishes as the number of policy samples increases, while the latter vanishes as the kernel bandwidth tends to zero for a continuous drifting field.

\subsection{Future Term and Critic Surrogate}
\label{appendix:future-critic}

We now show that the future gradient contribution can be represented exactly by an ideal frozen action-value function, and then characterize the approximation error introduced by a practical learned critic.
Recall
\begin{equation}
    Q_t^\theta(\bm{s}_t,a_t)
    =
    \mathbb{E}_{\bm{s}_{t+1}\sim P(\cdot\mid\bm{s}_t,a_t)}
    \left[
    J_{t+1}(\bm{s}_{t+1})
    \right].
\label{eq:appendix-future-definition}
\end{equation}
When computing the contribution at sequence index $t$, $Q_t^\theta$ is evaluated under the current policy but treated as fixed with respect to $\theta$, as in \Eqref{eq:appendix-local-future-gradient}.

\paragraph{Structure of the future action gradient.}
Assume that $P(\bm{s}_{t+1}\mid\bm{s}_t,a_t)$ admits a density differentiable with respect to $a_t$.
Differentiating under the integral gives
\begin{equation}
    \nabla_{a_t}Q_t^\theta(\bm{s}_t,a_t)
    =
    \mathbb{E}_{\bm{s}_{t+1}\sim P(\cdot\mid\bm{s}_t,a_t)}
    \left[
    J_{t+1}(\bm{s}_{t+1})
    \nabla_{a_t}\log P(\bm{s}_{t+1}\mid\bm{s}_t,a_t)
    \right].
\label{eq:appendix-q-action-gradient}
\end{equation}
Define
\begin{equation}
    \begin{aligned}
        I_P(\bm{s}_t,a_t)
        &=
        \mathbb{E}_{\bm{s}_{t+1}\sim P(\cdot\mid\bm{s}_t,a_t)}
        \left[
        \left\|
        \nabla_{a_t}\log P(\bm{s}_{t+1}\mid\bm{s}_t,a_t)
        \right\|_2^2
        \right], \\
        M_{t+1}(\bm{s}_t,a_t)
        &=
        \mathbb{E}_{\bm{s}_{t+1}\sim P(\cdot\mid\bm{s}_t,a_t)}
        \left[
        J_{t+1}(\bm{s}_{t+1})^2
        \right].
    \end{aligned}
\label{eq:appendix-transition-fisher}
\end{equation}
By Cauchy--Schwarz,
\begin{equation}
    \left\|
    \nabla_{a_t}Q_t^\theta(\bm{s}_t,a_t)
    \right\|_2
    \le
    \sqrt{
    I_P(\bm{s}_t,a_t)
    M_{t+1}(\bm{s}_t,a_t)
    }.
\label{eq:appendix-future-magnitude}
\end{equation}
If
\begin{equation}
    \overline D(\theta)
    =
    \sup_{\bm{s}}
    \mathcal{L}_{\mathrm{Local}}(\theta;\bm{s})
    <
    \infty,
\end{equation}
then
$0\le J_{t+1}(\bm{s}_{t+1})
\le (T-t)\overline D(\theta)$, and therefore
\begin{equation}
    \left\|
    \nabla_{a_t}Q_t^\theta(\bm{s}_t,a_t)
    \right\|_2
    \le
    (T-t)\overline D(\theta)
    \sqrt{I_P(\bm{s}_t,a_t)}.
\label{eq:appendix-future-local-bound}
\end{equation}
Hence, the future contribution can be significant only when the current action changes the successor distribution and non-negligible policy discrepancy remains in the future.
In particular, the exact future contribution vanishes as the remaining local discrepancy vanishes.

\paragraph{Ideal frozen critic.}
Using
$a_t=\bm{G}_\theta(\bm{z},\bm{s}_t)$, the exact future gradient contribution is
\begin{equation}
    \bm{g}_{\mathrm{Future}}(\theta;\bm{s}_t)
    =
    \mathbb{E}_{\bm{z}}
    \left[
    (\partial_\theta\bm{G}_\theta)^\top
    \nabla_{a_t}Q_t^\theta(\bm{s}_t,a_t)
    \right].
\label{eq:appendix-exact-future-gradient}
\end{equation}
Consider a reference policy $\pi_{\bar\theta}$ and define the ideal frozen critic
\begin{equation}
    Q^{\mathrm{ideal}}_{t,\bar\theta}(\bm{s}_t,a_t)
    =
    -
    Q_t^{\bar\theta}(\bm{s}_t,a_t)
    +
    b_t(\bm{s}_t),
\label{eq:appendix-ideal-critic}
\end{equation}
where $\bar\theta$ is fixed and $b_t(\bm{s}_t)$ is any action-independent function.
Importantly, $Q^{\mathrm{ideal}}_{t,\bar\theta}$ is treated as a fixed function and has no differentiable dependence on the optimized policy parameters $\theta$.

Define
\begin{equation}
    \mathcal{L}_{\mathrm{Critic}}^{\mathrm{ideal}}
    (\theta;\bm{s}_t)
    =
    -
    \mathbb{E}_{\bm{z}}
    \left[
    Q^{\mathrm{ideal}}_{t,\bar\theta}
    \left(
    \bm{s}_t,
    \bm{G}_\theta(\bm{z},\bm{s}_t)
    \right)
    \right].
\label{eq:appendix-ideal-critic-objective}
\end{equation}
Its gradient is
\begin{equation}
    \nabla_\theta
    \mathcal{L}_{\mathrm{Critic}}^{\mathrm{ideal}}
    =
    -
    \mathbb{E}_{\bm{z}}
    \left[
    (\partial_\theta\bm{G}_\theta)^\top
    \nabla_{a_t}
    Q^{\mathrm{ideal}}_{t,\bar\theta}(\bm{s}_t,a_t)
    \right].
\label{eq:appendix-ideal-critic-gradient}
\end{equation}
Since
\begin{equation}
    \nabla_{a_t}
    Q^{\mathrm{ideal}}_{t,\bar\theta}(\bm{s}_t,a_t)
    =
    -
    \nabla_{a_t}
    Q_t^{\bar\theta}(\bm{s}_t,a_t),
\end{equation}
evaluating \Eqref{eq:appendix-ideal-critic-gradient} at $\theta=\bar\theta$ yields
\begin{equation}
    \left.
    \nabla_\theta
    \mathcal{L}_{\mathrm{Critic}}^{\mathrm{ideal}}
    (\theta;\bm{s}_t)
    \right|_{\theta=\bar\theta}
    =
    \bm{g}_{\mathrm{Future}}(\bar\theta;\bm{s}_t).
\label{eq:appendix-ideal-critic-exact}
\end{equation}
Thus, an ideal frozen critic exactly realizes the future gradient contribution at the reference policy.
The scalar critic value itself need not match $-Q_t^{\bar\theta}$: any action-independent offset $b_t(\bm{s}_t)$ leaves the policy update unchanged.
This establishes that the critic term is not an auxiliary regularizer in the ideal case, but an exact frozen representation of the sequence-level future contribution.

\paragraph{Learned critic and approximation error.}
In practice, the ideal critic in \Eqref{eq:appendix-ideal-critic} is unavailable.
DriftOPD therefore uses a frozen learned critic $Q_\phi$, for which larger values indicate more desirable actions,
\begin{equation}
    \mathcal{L}_{\mathrm{Critic}}(\theta;\bm{s}_t)
    =
    -
    \mathbb{E}_{\bm{z}}
    \left[
    Q_\phi
    \left(
    \bm{s}_t,
    \bm{G}_\theta(\bm{z},\bm{s}_t)
    \right)
    \right],
\label{eq:appendix-critic-objective}
\end{equation}
with
\begin{equation}
    \nabla_\theta
    \mathcal{L}_{\mathrm{Critic}}
    =
    -
    \mathbb{E}_{\bm{z}}
    \left[
    (\partial_\theta\bm{G}_\theta)^\top
    \nabla_{a_t}Q_\phi(\bm{s}_t,a_t)
    \right].
\label{eq:appendix-critic-gradient}
\end{equation}
At any current policy $\theta$, define the action-gradient mismatch
\begin{equation}
    \Delta_t^\phi(\bm{s}_t,a_t)
    =
    \nabla_{a_t}Q_t^\theta(\bm{s}_t,a_t)
    +
    \nabla_{a_t}Q_\phi(\bm{s}_t,a_t),
\label{eq:critic-gradient-mismatch}
\end{equation}
and
\begin{equation}
    \epsilon_t(\phi;\bm{s}_t)
    =
    \left(
    \mathbb{E}_{\bm{z}}
    \left[
    \left\|
    \Delta_t^\phi(\bm{s}_t,a_t)
    \right\|_2^2
    \right]
    \right)^{1/2}.
\label{eq:critic-action-gradient-error}
\end{equation}
If
\begin{equation}
    \left\|
    \partial_\theta
    \bm{G}_\theta(\bm{z},\bm{s}_t)
    \right\|_{\mathrm{op}}
    \le
    L_G,
\end{equation}
then
\begin{equation}
    \begin{aligned}
        \left\|
        \bm{g}_{\mathrm{Future}}(\theta;\bm{s}_t)
        -
        \nabla_\theta
        \mathcal{L}_{\mathrm{Critic}}(\theta;\bm{s}_t)
        \right\|_2
        &=
        \left\|
        \mathbb{E}_{\bm{z}}
        \left[
        (\partial_\theta\bm{G}_\theta)^\top
        \Delta_t^\phi(\bm{s}_t,a_t)
        \right]
        \right\|_2 \\
        &\le
        L_G\,
        \epsilon_t(\phi;\bm{s}_t).
    \end{aligned}
\label{eq:critic-gradient-error-bound}
\end{equation}
Therefore, the practical critic differs from the ideal frozen surrogate only through its action-gradient mismatch on model-generated action chunks.
\Eqref{eq:critic-gradient-error-bound} characterizes this approximation without requiring the learned critic to recover the absolute value of the future divergence, or assuming that $\epsilon_t(\phi;\bm{s}_t)$ is necessarily small.
The particular demonstration-trained critic used in our experiments is one practical instantiation of this frozen surrogate.
\section{Experiment Protocol}
\label{appendix:implementation}

\subsection{Model Configurations}
\label{appendix:model-data}

\begin{table}[H]
    \centering
    \caption{Model configurations used in our experiments}
    \setlength{\tabcolsep}{6pt}
    \renewcommand{\arraystretch}{2}
    \resizebox{\textwidth}{!}{%
        \begin{tabular}{c c c c c c c c}
            \toprule
            \textbf{Model} & \textbf{Benchmark} & \textbf{\shortstack{Action\\objective}} & \textbf{\shortstack{Distillation\\trainable parameters}} & \textbf{\shortstack{Action\\dim.}} & \textbf{\shortstack{Chunk\\length}} & \textbf{\shortstack{Teacher\\steps}} & \textbf{\shortstack{One-step\\parameterization}} \\
            \midrule
            $\pi_{0.5}$ & RoboCasa365 & Flow matching & \shortstack{Action expert\\693.4M} & 32 & 50 & 10 & $z-v(z,t{=}1)$ \\
            \midrule
            \multirow{3}{*}{$\pi_{0.5}$-LoRA} & RoboCasa365 & \multirow{3}{*}{Flow matching} & \multirow{3}{*}{\shortstack{LoRA + input/output proj.\\+ timestep MLP, 9.1M}} & \multirow{3}{*}{32} & 50 & \multirow{3}{*}{10} & \multirow{3}{*}{$z-v(z,t{=}1)$} \\
            & LIBERO & & & & 10 & & \\
            & RoboTwin 2.0 & & & & 50 & & \\
            \midrule
            GR00T N1.5 & RoboCasa365 & Flow matching & \shortstack{Action-generation head\\1,068.8M} & 32 & 16 & 4 & $z+v(z,t{=}0)$ \\
            \midrule
            GR00T N1.6 & RoboCasa365 & Flow matching & \shortstack{Action-generation head\\1,418.6M} & 128 & 50 & 4 & $z+v(z,t{=}0)$ \\
            \midrule
            \multirow{2}{*}{ABot-M0} & RoboTwin 2.0 & \multirow{2}{*}{Direct clean action} & \multirow{2}{*}{\shortstack{Action-generation head\\164.2M}} & \multirow{2}{*}{14} & 50 & \multirow{2}{*}{10} & \multirow{2}{*}{Direct $x_0$} \\
            & LIBERO & & & & 10 & & \\
            \midrule
            SmolVLA & LIBERO & Flow matching & \shortstack{Action-generation head\\97.5M} & 32 & 50 & 10 & $z-v(z,t{=}1)$ \\
            \bottomrule
        \end{tabular}%
    }
    \label{tab:model-config}
\end{table}

\begin{table}[H]
    \centering
    \tiny
    \caption{Starting checkpoints used for each model and benchmark}
    \renewcommand{\arraystretch}{1.2}
    \resizebox{\textwidth}{!}{%
        \begin{tabular}{c c >{\raggedright\arraybackslash}p{0.60\textwidth}}
            \toprule
            \textbf{Model} & \textbf{Benchmark / setting} & \textbf{Starting checkpoint} \\
            \midrule
            \multirow{3}{*}{$\pi_{0.5}$} & RoboCasa365 & \href{https://huggingface.co/robocasa/robocasa365_checkpoints}{\texttt{robocasa/robocasa365\_checkpoints}} \\
            & LIBERO & \texttt{gs://openpi-assets/checkpoints/pi05\_libero/} \\
            & RoboTwin 2.0 & \texttt{gs://openpi-assets/checkpoints/pi05\_base/params} \newline (fine-tuned by us for 75k steps with batch size 64 on \href{https://huggingface.co/datasets/TianxingChen/RoboTwin2.0}{\texttt{TianxingChen/RoboTwin2.0}}, 50 tasks) \\
            \midrule
            GR00T N1.5 & RoboCasa365 & \href{https://huggingface.co/robocasa/robocasa365_checkpoints}{\texttt{robocasa/robocasa365\_checkpoints}} \\
            \midrule
            \multirow{2}{*}{GR00T N1.6} & RoboCasa365 & \href{https://huggingface.co/Abhi03/grootn16_robocasa365_multitask_learning}{\texttt{Abhi03/grootn16\_robocasa365\_multitask\_learning}} \\
            & Real world & \href{https://huggingface.co/nvidia/GR00T-N1.6-3B}{\texttt{nvidia/GR00T-N1.6-3B}} \\
            \midrule
            \multirow{2}{*}{ABot-M0} & LIBERO & \href{https://huggingface.co/acvlab/ABot-M0-LIBERO}{\texttt{acvlab/ABot-M0-LIBERO}} \\
            & RoboTwin 2.0 & \href{https://huggingface.co/acvlab/ABot-M0-RoboTwin2}{\texttt{acvlab/ABot-M0-RoboTwin2}} \\
            \midrule
            SmolVLA & LIBERO & \href{https://huggingface.co/HuggingFaceVLA/smolvla_libero}{\texttt{HuggingFaceVLA/smolvla\_libero}} \\
            \bottomrule
        \end{tabular}%
    }
    \label{tab:starting-checkpoints}
\end{table}

\paragraph{Model configurations.} 
Tables~\ref{tab:model-config} and~\ref{tab:starting-checkpoints} summarize the policy configurations and starting checkpoints used throughout our experiments.
Depending on the benchmark, we either use a released task-specific checkpoint directly or fine-tune a base checkpoint before distillation, as detailed in Table~\ref{tab:starting-checkpoints}.
Each one-step student is initialized from its corresponding teacher checkpoint, and during distillation only the parameters listed in Table~\ref{tab:model-config} are updated while the remaining policy parameters are kept frozen.

\paragraph{LoRA fine-tuning.}
For $\pi_{0.5}$-LoRA, we insert LoRA~\citep{LoRA} modules with $r=16$ and $\alpha=16$ into the query, key, value, output, gate, up, and down projections of the action-expert transformer.
This replaces 126 linear layers with LoRA-augmented layers and contributes 6.93M trainable LoRA parameters.
The action input/output projections and timestep MLP remain fully trainable, resulting in 9.1M trainable parameters in total.
We use no LoRA dropout or bias term, and the LoRA scaling is $\alpha/r=1$.

\subsection{DriftOPD Training Details}
\label{appendix:distillation-details}

\begin{table}[H]
\centering
\small
\caption{Common student optimization settings}
\renewcommand{\arraystretch}{1.2}
    \resizebox{0.95\textwidth}{!}{%
    \begin{tabular}{p{0.31\textwidth} p{0.61\textwidth}}
        \toprule
        \textbf{Hyperparameter} & \textbf{Setting} \\
        \midrule
        Training steps & 10,000 \\
        Optimizer & AdamW, $\beta_1=0.9$, $\beta_2=0.95$ \\
        Learning rate & $3\times10^{-5}\rightarrow3\times10^{-7}$ \\
        LR schedule & Linear decay; no learning-rate warmup \\
        Weight decay & $0$ \\
        Gradient clipping & $1.0$ \\
        Precision & bfloat16 autocast with FP32 master weights \\
        EMA & Maximum decay $0.9995$ with 1,000-step EMA warmup \\
        Initialization & Full copy of the corresponding starting policy checkpoint \\
        Default training seed & 42 \\
        Noise & $z\sim\mathcal{N}(0,I)$, resampled at every microstep \\
        Antithetic sampling & Disabled \\
        Noise scale & $1.0$ \\
        \bottomrule
    \end{tabular}
    }
\label{tab:student-optimization}
\end{table}

\paragraph{One-step distillation settings.}
The common optimization settings used for DriftOPD and the one-step baselines are summarized in Table~\ref{tab:student-optimization}.
Unless otherwise specified, all students are trained for 10k optimization steps and evaluated using their EMA parameters.

\begin{table}[H]
\centering
\small
\caption{DriftOPD-specific hyperparameters}
\renewcommand{\arraystretch}{1.2}
    \resizebox{0.95\textwidth}{!}{%
    \begin{tabular}{p{0.31\textwidth} p{0.61\textwidth}}
        \toprule
        \textbf{Hyperparameter} & \textbf{Setting} \\
        \midrule
        Kernel & Laplacian, $K_R(a,a')\propto\exp(-\lVert a-a'\rVert/R)$ \\
        Bandwidths & $R\in\{0.02,0.05,0.2\}$ \\
        Current-policy samples $M$ & 64 by default (8 gradient-carrying + 56 detached) \newline 32 (8 gradient-carrying + 24 detached) on RoboCasa365 for $\pi_{0.5}$  \\
        Positive sample & One aligned demonstration action, separate from the $M$ policy samples \\
        Default KDE pool & 64 policy samples + 1 demonstration sample \\
        Input perturbation & Laplace noise $\epsilon_R$ with scale $R$ on generated policy samples \\
        Drifting step $\eta$ & $1.0$ for all bandwidths \\
        Force normalization & Current-step global RMS over batch, generated-sample, and flattened action dimensions \\
        RMS floor & $10^{-8}$ before the square root \\
        Multi-band aggregation & Compute one normalized loss per bandwidth and average the losses \\
        $\lambda$ & $0.01$, fixed \\
        Critic warmup & linear warmup over the first 500 steps \\
        Critic gradient normalization & Drift-to-critic gradient-norm ratio with EMA $0.99$ \\
        Gradient-ratio range & $[0.1,10]$ \\
        Simulation Q clipping & Disabled \\
        Real-world Q clipping & Hard clip when generated-action $Q$ exceeds demonstration $Q$ \\
        \bottomrule
    \end{tabular}
    }
\label{tab:driftopd-hparams}
\end{table}

\paragraph{DriftOPD-specific settings.}
Table~\ref{tab:driftopd-hparams} summarizes the implementation settings specific to the drifting and critic objectives.
Following prior work \citep{Drifting,DriftingDistill}, we use a Laplacian affinity kernel and the three-bandwidth set $R\in\{0.02,0.05,0.2\}$.
As in \citet{DriftingDistill}, the policy-sample set is built from both gradient-carrying and detached generations. $M$ counts current-policy samples only, of which eight retain gradients and the remainder are detached, and the aligned demonstration action enters the kernel pool as the single positive.
We use $M=64$ (8 gradient-carrying $+$ 56 detached) by default, giving a kernel pool of 65 samples. On RoboCasa365, $\pi_{0.5}$ students instead use $M=32$ (8 $+$ 24), i.e.\ a pool of 33.

\paragraph{Details for the drifting field.}
Before the kernel is evaluated, Laplace noise $\epsilon_R$ with scale $R$ is added to the generated and detached policy samples \citep{DriftingDistill}, so the drifting force is evaluated at perturbed actions $\tilde{a}=a+\epsilon_R$. The demonstration positive is left unperturbed, and the same noise realization enters the regression target of the one-step student.
For each bandwidth, the drifting force is normalized using statistics from the current optimization step rather than a running RMS:
\begin{equation}
    \widehat{\bm V}_{R}
    =
    \frac{\bm V_R}
    {\sqrt{\max\!\left(\mathbb{E}_{b,c,s}\!\left[\bm V_R^2\right],10^{-8}\right)}}.
\label{eq:rms-drift}
\end{equation}
Here, $b$, $c$, and $s$ index the conditioning batch, generated policy samples, and flattened action coordinates, respectively.
The drifting loss is computed independently for each bandwidth and then averaged across the three bandwidths. Normalizing each bandwidth before averaging keeps the target displacement independent of the number of bandwidths, whereas summing the normalized fields would scale it with $|R|$ and destabilize early training.

\paragraph{Implementation details for the critic loss.}
The critic term uses a base weight $\lambda=0.01$ with no decay schedule, ramped linearly from zero over the first 500 optimization steps so that the critic does not steer the student while its one-step outputs are still far from the data manifold. To keep the two terms commensurable, the critic weight is rescaled by the ratio of gradient norms measured at the generated actions, $r=\lVert\nabla_a\mathcal{L}_{\mathrm{Drift}}\rVert_2/\lVert\nabla_a Q_\phi\rVert_2$, clipped to $[0.1,10]$ and smoothed by an exponential moving average with coefficient $0.99$. 
The effective weight is $\lambda_{\mathrm{eff}}(t)=\lambda\,w(t)\,\bar r(t)$, where $w(t)$ denotes the warmup factor. The critic-weight analysis in \Figref{fig:critic} evaluates $\lambda\in\{0,0.001,0.01,0.1,0.3,1.0\}$ while holding all other optimization and stabilization settings fixed, and the final simulation experiments use $\lambda=0.01$.
No trust region or clipping of the critic score is applied in simulation. The real-world runs additionally hard-clip the critic term when the generated-action $Q_\phi$ exceeds the demonstration-action $Q_\phi$.

\begin{table}[H]
    \centering
    \small
    \caption{Model-specific batch configurations. $N_c$ denotes the number of task contexts per optimization step, and batch size denotes the number of gradient-carrying policy samples per context. Students evaluated on LIBERO-Plus
    are the LIBERO students, so they share the LIBERO configuration.}
    \renewcommand{\arraystretch}{1.2}
    \resizebox{0.8\textwidth}{!}{%
        \begin{tabular}{l l c c c}
            \toprule
            \textbf{Benchmark} & \textbf{Model} & \textbf{$N_c$ (\# tasks)} &
            \textbf{Batch size} & \textbf{\shortstack{Gradient-carrying\\samples / step}} \\
            \midrule
            \multirow{4}{*}{RoboCasa365}
              & $\pi_{0.5}$ (full fine-tuning) &  4 & 8 &  32 \\
              & $\pi_{0.5}$ (LoRA $r{=}16$)    &  4 & 8 &  32 \\
              & GR00T N1.5                     & 16 & 8 & 128 \\
              & GR00T N1.6                     & 16 & 8 & 128 \\
            \midrule
            \multirow{2}{*}{RoboTwin 2.0}
              & $\pi_{0.5}$ (LoRA $r{=}16$)    &  4 & 8 &  32 \\
              & ABot-M0                        & 64 & 8 & 512 \\
            \midrule
            \multirow{3}{*}{\shortstack[l]{LIBERO\\/ LIBERO-Plus}}
              & $\pi_{0.5}$ (LoRA $r{=}16$)    &  4 & 8 &  32 \\
              & ABot-M0                        & 64 & 8 & 512 \\
              & SmolVLA                        & 16 & 8 & 128 \\
            \bottomrule
        \end{tabular}%
    }
    \label{tab:model-batch}
\end{table}

\paragraph{Mini-batch configurations.}
Table~\ref{tab:model-batch} summarizes the batch configuration for each model and benchmark. Every optimization step draws $N_c$ task contexts and generates eight gradient-carrying one-step chunks per context, so a step is backpropagated through $8N_c$ policy samples. We keep the per-context count at eight everywhere and vary only $N_c$, which ranges from $4$ for $\pi_{0.5}$ to $64$ for ABot-M0, because the drifting loss is a per-task quantity and widening $N_c$ marginalizes over more of the task distribution per step. Within a given model and benchmark all four students share the same $N_c$ and batch size, so sCD, MFD, Drift, and DriftOPD see an identical sample budget and any gap between them reflects the objective alone. Each context also supplies one cached ground-truth chunk and $56$ or $24$ detached negatives, which shape the loss but carry no gradient and are excluded from the count above. We realize $N_c$ on a single device as the per-microstep task count times the number of gradient accumulation steps, holding the product fixed across students, so no multi-GPU training is needed to reproduce these settings. Students evaluated on LIBERO-Plus are the LIBERO students and inherit the LIBERO configuration.

\clearpage
\subsection{Critic Training}
\label{appendix:critic-training}

\begin{figure}[H]
    \centering
    \includegraphics[width=0.85\linewidth]{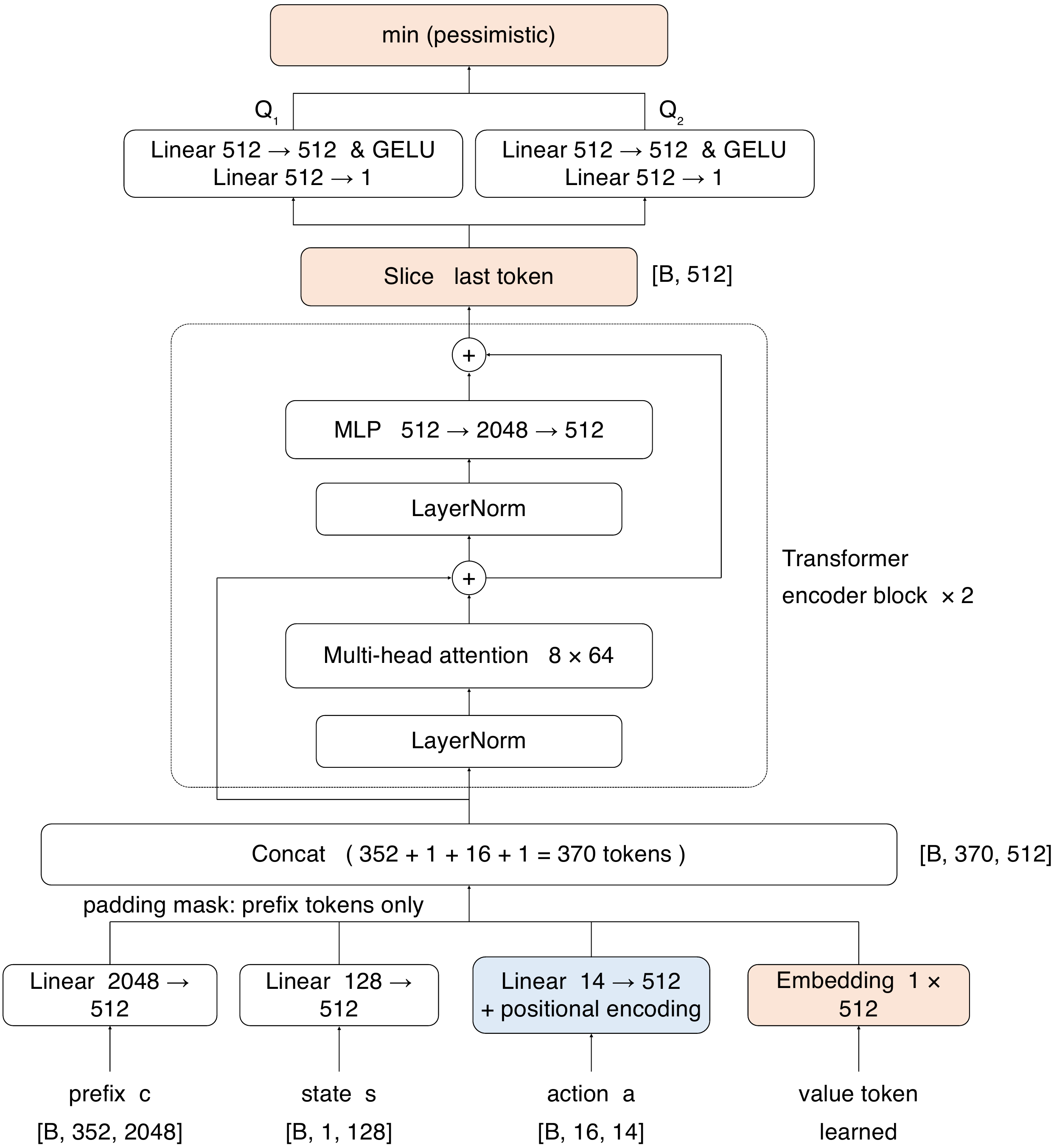}
    \caption{Architecture of the demonstration-trained critic used in DriftOPD.
    The frozen VLA prefix, the robot state, and the candidate action chunk are
    encoded jointly, and a learned value token is read out by twin Q-heads.}
    \label{fig:critic-architecture}
\end{figure}

\paragraph{Overview.}
DriftOPD queries the critic for an action gradient rather than an absolute value, so the critic must vary informatively with the action chunk in a neighborhood of demonstrated behavior. 
We therefore build on the Q-chunk-Former formulation of VGAS~\citep{VGAS} but adapt both the architecture and the training objective to that requirement. We do not reuse a released VGAS critic.
Instead we train one critic per policy backbone and embodiment from offline demonstrations alone, so that the critic reads the same frozen prefix the student is distilled against. 
As shown in \Figref{fig:critic-architecture}, the critic consumes the frozen VLA prefix together with the robot state and an action chunk, encodes them with a lightweight bidirectional Transformer, and reads an action-conditioned value out of a learned value token through twin Q-heads.
The critic is trained once and then frozen for the whole of policy distillation. Table~\ref{tab:critic-config} lists the architecture and optimization settings in full.

\begin{table}[t]
\centering
\small
\caption{Architecture and optimization settings of the demonstration-trained critic.}
\renewcommand{\arraystretch}{1.2}
\resizebox{0.95\textwidth}{!}{%
\begin{tabular}{p{0.31\textwidth} p{0.61\textwidth}}
    \toprule
    \textbf{Component} & \textbf{Setting} \\
    \midrule
    Input & Frozen VLA prefix tokens + state token + action-chunk tokens + learned value token \\
    Encoder & 2-layer pre-norm bidirectional Transformer encoder \\
    Hidden dimension & 512 \\
    Attention heads & 8 \\
    FFN dimension & 2048 \\
    Dropout & $0$ \\
    Q-heads & Two independent MLP heads; $\min(Q_1,Q_2)$ used for policy optimization \\
    Q-head MLP & $512\rightarrow512\rightarrow1$ with GELU and LayerNorm \\
    Training data & Offline demonstrations only \\
    Sampling & Task-balanced uniform sampling over tasks, then a uniform row within the sampled task \\
    Held-out split & Simulation: tasks satisfying \texttt{task\_index \% 10 == 0} excluded from critic training \newline Real world: episodes satisfying \texttt{episode\_index \% 10 == 0} excluded from critic training \\
    Target & Monte-Carlo return \\
    Discount factor & $\gamma=0.998$ \\
    Optimizer & AdamW, $\beta_1=0.9$, $\beta_2=0.95$ \\
    Learning rate & $10^{-4}$ \\
    Weight decay & $10^{-6}$ \\
    Batch size & 64 \\
    Training steps & 20,000 \\
    LR schedule & 1,000-step warmup followed by cosine decay \\
    Gradient clipping & 10 \\
    Critic seed & 0 \\
    Q-output normalization & None \\
    Objective & MC regression + Cal-QL + EGR \\
    Cal-QL & Coefficient $1.0$, eight OOD action candidates \\
    OOD candidates & Four local perturbations with $\sigma=0.3$ + four marginal samples \\
    EGR & Anchor weight $1.0$, ranking weight $5.0$ \\
    EGR scales & $\beta$ and $d_{\max}$ automatically calibrated from each cache \\
    \bottomrule
\end{tabular}
}
\label{tab:critic-config}
\end{table}

\paragraph{Training objective.}
Each training sample is a triple $(s_t,a_t,R_t)\sim\mathcal{D}$, where $s_t$ collects the frozen VLA prefix tokens, the robot state, and the language instruction, $a_t$ is the demonstrated action chunk executed at $s_t$, and $R_t$ is its Monte-Carlo return. The two Q-heads are combined pessimistically,
\begin{equation}
    Q_\phi(s_t,a_t)
    =
    \min\big(Q_1(s_t,a_t),Q_2(s_t,a_t)\big),
    \label{eq:critic-min}
\end{equation}
and the objective pairs a value-regression term with two terms that induce action dependence,
\begin{equation}
    \mathcal{L}_{\mathrm{Critic}}(\phi)
    =
    \mathcal{L}_{\mathrm{MC}}(\phi)
    +
    \alpha\,\mathcal{L}_{\mathrm{Cal\text{-}QL}}(\phi)
    +
    \mathcal{L}_{\mathrm{EGR}}(\phi),
    \qquad
    \alpha=1.0.
    \label{eq:critic-total}
\end{equation}

\paragraph{Monte-Carlo regression.}
Our demonstration caches retain only a few frames per episode, so consecutive chunk pairs $(s_t,s_{t+H})$ are rarely both present and temporal-difference bootstrapping is impractical. All demonstrations are successful and the reward is defined by frame position, so the return is computed directly from the episode and used as a regression target,
\begin{equation}
    \mathcal{L}_{\mathrm{MC}}(\phi)
    =
    \mathbb{E}_{(s_t,a_t,R_t)\sim\mathcal{D}}
    \left[
        \left(Q_\phi(s_t,a_t)-R_t\right)^2
    \right],
    \qquad
    R_t
    =
    \sum_{k\geq 0}\gamma^k r_{t+k},
    \quad
    \gamma=0.998.
    \label{eq:critic-mc}
\end{equation}
No target network or bootstrapping is used, so the critic is trained by pure regression onto returns available from the recorded episodes.

\paragraph{Negative candidates.}
Monte-Carlo regression alone constrains the critic only at demonstrated state--action pairs, which leaves its local variation in the action argument unsupervised. For each sample we therefore draw eight out-of-distribution (OOD) action chunks from an equal mixture of local and marginal candidates,
\begin{equation}
    a'
    \sim
    \nu(\cdot\mid s_t,a_t)
    =
    \frac{1}{2}
    \underbrace{
        \mathcal{N}\!\left(a_t,\sigma^2 I\right)
    }_{\text{local},\ \sigma=0.3}
    +
    \frac{1}{2}
    \underbrace{
        \mathcal{D}_a(\cdot)
    }_{\text{marginal}},
    \label{eq:critic-nu}
\end{equation}
where $\mathcal{D}_a$ is the empirical distribution of demonstrated action chunks recorded at other observations, and four candidates are drawn from each component. The local component populates a neighborhood of the demonstrated chunk, while the marginal component supplies chunks that are plausible in the dataset but mismatched to the current context.

\paragraph{Action dependence.}
The first of the two candidate-based terms is a conservative objective in the style of Cal-QL~\citep{Cal-QL}, which suppresses OOD action values while calibrating them against the Monte-Carlo return,
\begin{equation}
    \mathcal{L}_{\mathrm{Cal\text{-}QL}}(\phi)
    =
    \mathbb{E}_{s_t}
    \left[
        \log
        \sum_{a'\sim\nu}
        \exp
        \left(
            \max\left(Q_\phi(s_t,a'),R_t\right)
        \right)
        -
        Q_\phi(s_t,a_t)
    \right].
    \label{eq:critic-calql}
\end{equation}
The second, EGR, adds distance-aware supervision so that the critic varies smoothly with the distance from the demonstrated chunk, 
\begin{equation}
    \mathcal{L}_{\mathrm{EGR}}(\phi)
    =
    w_{\mathrm{anchor}}\,
    \ell_{\mathrm{anchor}}(\phi)
    +
    w_{\mathrm{rank}}\,
    \ell_{\mathrm{rank}}(\phi),
    \qquad
    w_{\mathrm{anchor}}=1.0,
    \quad
    w_{\mathrm{rank}}=5.0.
    \label{eq:critic-egr}
\end{equation}
Here $\ell_{\mathrm{anchor}}$ ties the value of a candidate at distance $d(a',a_t)$ to the margin-shifted anchor $Q_\phi(s_t,a_t)-\beta\,d(a',a_t)/d_{\max}$, while $\ell_{\mathrm{rank}}$ penalizes candidate pairs whose predicted ordering disagrees with their distance ordering, encouraging $Q_\phi(s_t,a'_i)>Q_\phi(s_t,a'_j)$ whenever $d(a'_i,a_t)<d(a'_j,a_t)$. The scale $\beta$ and the distance normalization $d_{\max}$ are calibrated from each critic cache rather than fixed across embodiments, since the action spaces differ in magnitude.

Together, \Eqref{eq:critic-calql} and \Eqref{eq:critic-egr} make $Q_\phi$ action-dependent in a neighborhood of demonstrated behavior. 
This is the property required by the critic surrogate in \Secref{ssec:method-critic}, which consumes the action gradient $\nabla_{a_t}Q_\phi(s_t,a_t)$ rather than the absolute critic value.

\subsection{Baseline Implementations}
\label{appendix:baseline-details}

All one-step baselines start from the same teacher checkpoint, read the same cached demonstrations and conditioning, and use the optimizer and effective context batch reported in Table~\ref{tab:student-optimization} and \ref{tab:model-batch}. We follow the hyperparameters of the original methods where they are specified and run no
baseline-specific sweep. The implemented settings are listed in Table~\ref{tab:baseline-config}.

\begin{table}[H]
\centering
\small
\caption{Implementation details of the one-step distillation baselines.}
\renewcommand{\arraystretch}{1.2}
\resizebox{\textwidth}{!}{%
\begin{tabular}{p{0.16\textwidth} p{0.76\textwidth}}
    \toprule
    \textbf{Method} & \textbf{Implementation} \\
    \midrule
    One-step Teacher &
    Direct one-step generation from the corresponding teacher checkpoint without
    additional distillation. For ABot-M0, the baseline uses its dedicated
    one-step generator. \\
    sCD &
    $P_{\mathrm{mean}}=-1.0$, $P_{\mathrm{std}}=1.4$, $\sigma_d=1$, tangent normalization $c=0.1$, JVP tangent
    estimation with central finite-difference fallback $\epsilon=10^{-3}$, and
    adaptive-weighting learning rate $10^{-4}$. \\
    MFD &
    Maximum step size $0.02$, solver-step cap 32, minimum interval $0.02$.
    TTUR ratio $K=10$ with logit-normal time sampling on LIBERO and
    RoboTwin 2.0; $K=50$ with uniform time sampling on RoboCasa365. \\
    Drift &
    Identical kernel, bandwidth, perturbation, normalization, sampling and
    student-optimization settings as DriftOPD, with the critic term disabled. \\
    DriftOPD &
    Drift objective combined with the frozen demonstration-trained critic using
    the settings in Tables~\ref{tab:driftopd-hparams} and~\ref{tab:critic-config}. \\
    \bottomrule
\end{tabular}
}
\label{tab:baseline-config}
\end{table}

\paragraph{One-step Teacher.}
The teacher is evaluated with a single generation step and no additional distillation, which isolates how much of the one-step gap is already closed by the pretrained flow head. For ABot-M0 this baseline uses the dedicated one-step generator shipped with the backbone rather than a truncated sampler.

\paragraph{sCD.}
We use the distillation form of the continuous-time consistency model \citep{sCM}, in which the probability-flow ODE is supplied by the frozen pretrained teacher $F^{\mathrm{Q}}$ rather than by an unbiased noise estimator.
With $\sigma_d=1$, fixed by the teachers having been trained with $\epsilon\sim\mathcal{N}(0,I)$, each step draws $x_0\sim\mathcal{D}$, $z\sim\mathcal{N}(0,\sigma_d^2 I)$ and a time from the log-normal proposal
\begin{equation}
    \tau\sim\mathcal{N}(P_{\mathrm{mean}},P_{\mathrm{std}}^2),
    \qquad
    t=\arctan\!\big(e^{\tau}/\sigma_d\big),
    \qquad
    x_t=\cos(t)\,x_0+\sin(t)\,z,
    \label{eq:scd-proposal}
\end{equation}
with $P_{\mathrm{mean}}=-1.0$ and $P_{\mathrm{std}}=1.4$. The tangent that the consistency condition requires is
\begin{equation}
    g
    =
    -\cos^2(t)\Big(\sigma_d F_{\theta^-}-\tfrac{\mathrm{d}x_t}{\mathrm{d}t}\Big)
    -
    r\,\cos(t)\sin(t)\Big(x_t+\sigma_d\,\tfrac{\mathrm{d}F_{\theta^-}}{\mathrm{d}t}\Big),
    \qquad
    \tfrac{\mathrm{d}x_t}{\mathrm{d}t}=\sigma_d F^{\mathrm{Q}}\!\big(x_t/\sigma_d,t\big),
    \label{eq:scd-tangent}
\end{equation}
where $r\in[0,1]$ is a linear warmup factor. The term $\mathrm{d}F_{\theta^-}/\mathrm{d}t$ is a directional derivative in time and is obtained with a Jacobian-vector product; since sCM removes the EMA target and sets $\theta^-=\mathrm{sg}[\theta]$, the primal output of that JVP is exactly $F_{\theta^-}$, so no separate forward pass is needed for it. 
If forward-mode automatic differentiation is unavailable for a backbone, the implementation falls back permanently to central finite differences with $\epsilon=10^{-3}$. The tangent is normalized as $g\leftarrow g/(\lVert g\rVert+c)$ with $c=0.1$, and the loss uses the adaptive weighting $w_\varphi$ of the original method,
\begin{equation}
    \mathcal{L}_{\mathrm{sCD}}(\theta,\varphi)
    =
    \frac{e^{w_\varphi(t)}}{D}
    \big\lVert
        F_\theta-F_{\theta^-}-g
    \big\rVert^2
    -
    w_\varphi(t),
    \label{eq:scd-loss}
\end{equation}
with $D$ the action-chunk dimension and $w_\varphi$ trained at learning rate $10^{-4}$.

\paragraph{MFD.}
Mean Flow Distillation \citep{MFD} is data-free, so the demonstration bank that Drift attracts toward is unused and only the teacher, the conditioning and the step budget are shared. Three heads sit on the frozen backbone: the frozen teacher $u^{\mathrm{Q}}$, the student generator $G_\theta$, and an auxiliary flow $v_\phi$ that tracks the student's own output distribution. Each outer iteration performs $K$ auxiliary updates followed by one student update. The auxiliary head is fit by flow matching on student samples,
\begin{equation}
    \mathcal{L}_{\mathrm{aux}}(\phi)
    =
    \mathbb{E}\Big[
        \big\lVert v_\phi(X_\tau,\tau)-(X_1-X_0')\big\rVert^2
    \Big],
    \quad
    X_1=\mathrm{sg}\!\left[G_\theta(X_0)\right],
    \quad
    X_\tau=(1-\tau)X_0'+\tau X_1,
    \label{eq:mfd-aux}
\end{equation}
with $\tau\sim\mathcal{U}[0,1]$ and $X_0,X_0'$ independent noise draws. 
The two-timescale ratio $K$ keeps $v_\phi$ an accurate model of the student distribution before each student step. For the student update an interval $0\leq s\leq t\leq 1$ is drawn, the state $X_s=s\,G_\theta(X_0)+(1-s)X_0'$ is formed, and both flows are integrated from $X_s$ over $[s,t]$ to give mean velocities
\begin{equation}
    \hat{U}^{\mathrm{P}}=\frac{\Phi(v_\phi,X_s,s,t)-X_s}{t-s},
    \qquad
    \hat{U}^{\mathrm{Q}}=\frac{\Phi(u^{\mathrm{Q}},X_s,s,t)-X_s}{t-s},
    \qquad
    \Delta=\hat{U}^{\mathrm{P}}-\hat{U}^{\mathrm{Q}},
    \label{eq:mfd-meanvel}
\end{equation}
where $\Phi$ denotes the ODE solver. Following the original derivation we drop the Jacobian of the solver with respect to its input, which gives the student gradient
\begin{equation}
    \nabla_\theta\mathcal{L}_{\mathrm{MFD}}
    =
    \mathbb{E}\Big[
        s\,\big(\nabla_\theta G_\theta(X_0)\big)^{\!\top}\Delta
    \Big].
    \label{eq:mfd-grad}
\end{equation}
We implement \Eqref{eq:mfd-grad} as the surrogate $\mathcal{L}=\langle \mathrm{sg}[\Delta],\,X_s\rangle$, whose autograd gradient is exactly the expression above. A mean-squared error on $\Delta$ would instead backpropagate through the solver and optimize a different objective.

\paragraph{Drift.}
Drift is the drifting distillation objective of \citet{DriftingDistill}, applied to action chunks. Generated samples are pushed by a drift field built from an affinity kernel over a single positive sample drawn from the demonstration cache and detached negatives drawn from the student's own output, so the student is attracted toward the demonstration distribution without a per-sample regression target. We use the Laplace kernel $k(d)=\exp(-d/R)$ on scale-normalized distances, compute the loss at each temperature and average across temperatures, and keep the three modifications of the reference implementation: a single softmax, Laplace perturbation of the inputs, and partially detached negatives. The feature over which the kernel acts is the flattened action chunk, and one drift problem is instantiated per task in the micro-batch. Kernel, bandwidth, perturbation, normalization, time sampling and student-optimization settings are listed in Table~\ref{tab:driftopd-hparams}.

\paragraph{DriftOPD.}
DriftOPD augments the Drift objective with the off-policy distillation term introduced in \Secref{ssec:method-critic}, which supplies the action gradient $\nabla_{a}Q_\phi$ of the frozen demonstration-trained critic. Every other setting is shared with Drift, so the two differ by exactly this one term and their gap is attributable to the critic. The critic is trained as described in \Secref{appendix:critic-training} and the coupling weight and normalization are given in Table~\ref{tab:driftopd-hparams}.

\clearpage
\subsection{Simulation Evaluation Protocol}
\label{appendix:simulation-protocol}

We follow the official benchmark protocols for replanning intervals and episode termination, while environment initialization follows the evaluation code released with the corresponding teacher models. Task sets, episode counts and the resulting settings are summarized in Table~\ref{tab:benchmark-config}.

\begin{table}[H]
    \centering
    \caption{Simulation benchmark configurations.}
    \label{tab:benchmark-config}
    \setlength{\tabcolsep}{4pt}
    \renewcommand{\arraystretch}{1.2}
    \resizebox{\textwidth}{!}{%
        \begin{tabular}{c c c c c c c}
            \toprule
            \textbf{Benchmark} & \textbf{Task set} & \textbf{\# Tasks} & \textbf{Episodes per task} & \textbf{Initialization} & \textbf{Actions before replanning} & \textbf{Termination} \\
            \midrule

            \multirow{3}{*}{RoboCasa365} & \texttt{pretrain} split & \multirow{3}{*}{50} & \multirow{3}{*}{30} & $\pi_{0.5}$: seed 7 & $\pi_{0.5}$: 5 & \multirow{3}{*}{Success or task horizon} \\
            & (leaderboard) & & & GR00T: no explicit seed & GR00T N1.5: 16 & \\
            & & & & & GR00T N1.6: 8 & \\
            \midrule

            \multirow{2}{*}{RoboTwin 2.0} & 50 tasks & \multirow{2}{*}{50} & \multirow{2}{*}{30} & Accepted-seed & \multirow{2}{*}{25} & Success or \\
            & \texttt{demo\_clean} & & & manifest & & task-specific limit \\
            \midrule

            \multirow{2}{*}{LIBERO} & 4 suites & \multirow{2}{*}{40} & \multirow{2}{*}{50} & Seed 7 & $\pi_{0.5}$: 5 & Success or \\
            & $\times$ 10 tasks & & & 50 official initial states & ABot / SmolVLA: 10 & suite horizon \\
            \midrule

            \multirow{2}{*}{LIBERO-Plus} & Perturbed & \multirow{2}{*}{10,030} & \multirow{2}{*}{1} & Seed 7 & $\pi_{0.5}$: 5 & Success or \\
            & LIBERO tasks & & & first official initial state & ABot / SmolVLA: 10 & suite horizon \\
            \bottomrule
        \end{tabular}%
    }
\end{table}

\paragraph{RoboCasa365.}
We use the \texttt{pretrain} evaluation split with 18 \texttt{atomic\_seen}, 16 \texttt{composite\_seen} and 16 \texttt{composite\_unseen} tasks, each evaluated over 30 episodes. 
For $\pi_{0.5}$ the environment is initialized with seed 7, so the compared methods share the same initial conditions and differ only in the stochasticity of policy generation. 
The official GR00T evaluators do not fix an environment seed, so those evaluations are not treated as paired. 
Thus, for Table~\ref{tab:robocasa365}, we follow the official GR00T evaluation protocol, while paired episode comparisons are provided as demo videos on the project page.
An episode ends when the environment reports success or when the task horizon from the benchmark registry is exhausted.

\paragraph{RoboTwin 2.0.}
We evaluate 50 tasks with the ALOHA-AgileX embodiment under the \texttt{demo\_clean} configuration, which disables background, lighting, camera, table-height and clutter randomization. The simulator runs at a physics timestep of $1/250\,\mathrm{s}$. Each predicted action chunk is executed for 25 environment actions before replanning, and an episode ends on success or when the task-specific action limit is reached.

To make initial conditions reproducible we build an accepted-seed manifest per task. Candidate seeds are drawn in increasing order from $100000$ and run with the benchmark expert, and a seed is accepted only when both its motion plan and its task-success check pass. Seeds the expert cannot solve are discarded, because such episodes fail for every policy and separate no methods while still consuming the budget. Filling 30 accepted episodes takes about 36 candidates per task at the median, so the seeds actually executed span $100000$ to $100190$ and form a non-contiguous subset of that interval rather than a fixed block. The manifest is shared across the compared methods, so a given episode index refers to the same initial condition for each of them, and the executed seed and outcome of every episode are recorded so that two runs can be aligned episode by episode. The short, medium and long groups reported in Table~\ref{tab:robotwin2} correspond to average demonstration trajectory lengths of ${<}150$, $150$--$279$ and ${\geq}280$ steps.

\paragraph{LIBERO and LIBERO-Plus.}
For LIBERO we evaluate the four standard suites with 10 tasks per suite and 50 episodes per task. Each task uses seed 7 and the 50 benchmark-provided initial states exactly once, so the compared methods again share initial conditions. The policy consumes agent-view and wrist RGB observations resized with padding to $224\times224$. Every episode begins with 10 no-op steps for physical stabilization, and the maximum horizons are 220, 280, 300 and 520 steps for LIBERO-Spatial, LIBERO-Object, LIBERO-Goal and LIBERO-10. An episode ends immediately on success and is otherwise counted as a failure once its horizon is reached.

LIBERO-Plus reuses the LIBERO students without further training and follows the same execution protocol and suite horizons. Its 10,030 task instances are perturbed variants of the standard LIBERO tasks, so each instance is already a distinct condition and is evaluated for a single episode from the first benchmark-provided initial state. All LIBERO-Plus evaluations are run in one compute environment so that rendering does not differ across nodes.

\clearpage
\subsection{Real-World Data and Evaluation}
\label{appendix:realworld-protocol}

\begin{figure}[H]
    \centering
    \includegraphics[width=0.85\linewidth]{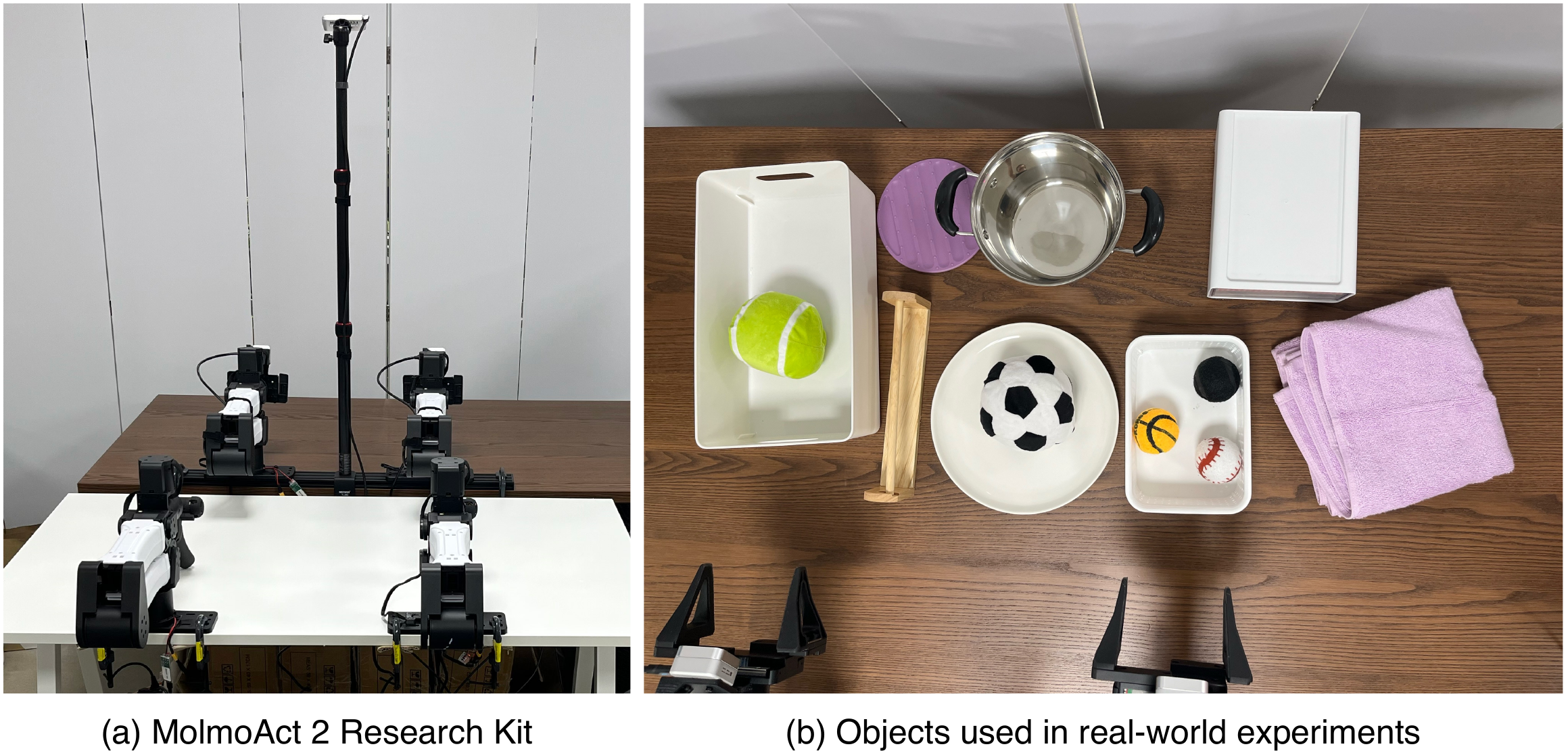}
    \vspace{-5pt}
    \caption{Real-world experimental setup. (a) MolmoAct 2 Research Kit on the YAM leader--follower bimanual platform, with top and wrist RealSense cameras providing RGB observations. The right arm is used for unimanual tasks and both arms for bimanual tasks. (b) Objects and containers used for the unimanual pick-and-place and bimanual handover tasks.}
    \label{fig:realworld-setup}
    \vspace{-15pt}
\end{figure}

\paragraph{Platform and demonstrations.}
All real-world experiments are conducted on the YAM bimanual platform shown in \Figref{fig:realworld-setup}a, using the right arm for unimanual tasks and both arms for bimanual tasks. Two operators collected 638 demonstration episodes spanning twelve tasks, teleoperating the follower arms at roughly $200\,\mathrm{Hz}$ while the three RealSense cameras recorded at approximately $30\,\mathrm{Hz}$. Robot state and action streams are linearly interpolated onto camera timestamps so that every frame carries a matched observation-action pair, yielding 366,373 aligned frames. The policy observes RGB only, letterboxed from $640\times480$ to $256\times256$, and predicts a seven-dimensional action per arm, consisting of six arm joint values and one continuous gripper value. Actions are normalized to $[-1,1]$ using 1$^\mathrm{st}$/99$^\mathrm{th}$ percentile min-max scaling, preventing rare teleoperation outliers from compressing the effective action range. The full collection and GR00T N1.6 teacher fine-tuning recipe are listed in Table~\ref{tab:realworld-data}.

\begin{table}[H]
\centering
\small
\caption{Real-world demonstration data and GR00T N1.6 teacher fine-tuning}
\vspace{-8pt}
\renewcommand{\arraystretch}{1.2}
\resizebox{0.95\textwidth}{!}{%
\begin{tabular}{p{0.31\textwidth} p{0.61\textwidth}}
    \toprule
    \textbf{Item} & \textbf{Setting} \\
    \midrule
    Robot & I2RT YAM leader--follower system \\
    Demonstrations & Unimanual (right arm): 424 episodes, 269,089 frames, 8 tasks\newline Bimanual: 214 episodes, 97,284 frames, 4 tasks \\
    Teleoperation rate & $\approx 200\,\mathrm{Hz}$ \\
    Camera rate & $\approx 30\,\mathrm{Hz}$ \\
    Policy cameras & Top, right, and left Intel RealSense, RGB only \\
    Raw resolution & $640\times480$ \\
    Model image input & $256\times256$ with letterbox preprocessing \\
    Robot state & Follower-arm state \\
    Action target & Leader-arm action \\
    Action dimension & Unimanual: 7 (six arm joints and one gripper value)\newline Bimanual: 14 (twelve arm joints and two gripper values) \\
    Gripper & Continuous value in $[0,1]$ \\
    Temporal alignment & Linear interpolation of state and action to camera timestamps \\
    Normalization & 1$^\mathrm{st}$/99$^\mathrm{th}$ percentile min-max scaling to $[-1,1]$ \\
    Data usage & All demonstrations used for teacher fine-tuning; no held-out split \\
    Teacher model & GR00T N1.6 \\
    Teacher fine-tuning & 20,000 steps, batch size 64, AdamW, learning rate
    $10^{-4}$, weight decay $10^{-5}$, warmup ratio $0.05$, cosine decay \\
    Hardware & 1$\times$ NVIDIA B200 \\
    \bottomrule
\end{tabular}
}
\label{tab:realworld-data}
\end{table}

\begin{table}[H]
\centering
\small
\caption{Objects used in the real-world experiments. Dimensions are height
$\times$ width $\times$ depth in centimeters; for spherical and circular items
the diameter is reported instead.}
\renewcommand{\arraystretch}{1.2}
\resizebox{0.95\textwidth}{!}{%
\begin{tabular}{p{0.30\textwidth} p{0.35\textwidth} p{0.25\textwidth}}
    \toprule
    \textbf{Object} & \textbf{Size} & \textbf{Role} \\
    \midrule
    Small basket (white, shallow) & $20.2\times14.5\times6$ & Target container \\
    Plate (white, round) & diameter $22.5$ & Target container \\
    Pot (stainless steel) & diameter $16$, height $9.2$ & Target container \\
    Plush tennis ball & diameter $\approx12$ & Grasped object \\
    Plush soccer ball & diameter $\approx12$ & Grasped object \\
    Plush basketball & diameter $\approx5$ & Grasped object \\
    Plush baseball & diameter $\approx5$ & Grasped object \\
    Plush black ball & diameter $\approx5$ & Grasped object \\
    \bottomrule
\end{tabular}
}
\vspace{-10pt}
\label{tab:realworld-objects}
\end{table}

\paragraph{Task design.}
We construct four unimanual pick-and-place tasks with different arm motions and end-effector orientations, as well as two bimanual pick-and-handover-and-place tasks using the objects and containers shown in \Figref{fig:realworld-setup}b and Table~\ref{tab:realworld-objects}. The four unimanual tasks are qualitatively grouped into Easy, Moderate, Hard, and Challenging based on their task characteristics and the multi-step teacher success rates. For the two bimanual tasks, we swap the object and target locations to create left-to-right and right-to-left handover settings.

\begin{table}[H]
\centering
\small
\caption{Real-world evaluation configuration}
\renewcommand{\arraystretch}{1.2}
\resizebox{0.95\textwidth}{!}{%
\begin{tabular}{p{0.35\textwidth} p{0.56\textwidth}}
    \toprule
    \textbf{Item} & \textbf{Setting} \\
    \midrule
    Tasks & Unimanual: four pick-and-place tasks \newline Bimanual: two pick-and-handover-and-place tasks \\
    Evaluation settings & 15 fixed scenes per task \\
    Pairing & Teacher session captures the scenes; all methods replay the same scenes in the same order \\
    Control frequency & $30\,\mathrm{Hz}$ \\
    Inference interval & $0.3\,\mathrm{s}$, i.e.\ about 9 commanded control steps between replans \\
    Inference mode & Asynchronous rate-limited inference \\
    Inference hardware & Local NVIDIA GeForce RTX 4090 (24\,GB); no external inference server \\
    Success annotation & Manual binary success/failure by the operator \\
    Scene verification & RGB and depth similarity logged against the teacher-session reference \\
    Reset & Return to a predefined basic pose, then settle \\
    Episode timeout & $40\,\mathrm{s}$ \\
    Maximum commanded joint speed & $0.75\,\mathrm{rad/s}$ \\
    Command timeout & $0.5\,\mathrm{s}$ \\
    \bottomrule
\end{tabular}
}
\vspace{-10pt}
\label{tab:realworld-execution}
\end{table}

\paragraph{Paired scene replay.}
Each task is evaluated on 15 fixed scenes. The multi-step teacher session runs first and records those scenes; the one-step teacher, sCD, MFD, Drift, and DriftOPD then replay the same 15 scenes in the same order. To ensure paired initial conditions, we require both the top and wrist views to achieve PSNR $\geq 28$\,dB and SSIM $\geq 0.95$ relative to the corresponding reference scene. Thus, every method is evaluated under closely matched visual environments. Success is annotated manually as a binary outcome by the operator running the session. The remaining execution settings are given in Table~\ref{tab:realworld-execution}.

\paragraph{Safety and exclusions.}
Joint-speed and command-timeout limits are enforced by the robot controller independently of the evaluated policy, so they apply identically to every method. Joint-limit recovery occasionally requires an extra reset attempt between episodes; because this happens before policy execution begins, no accepted evaluation episode is discarded by it. Episodes affected by operator error are identified and repeated before evaluation rather than being dropped afterwards. Any episode in which the two robot arms collide with each other, or in which a robot arm collides with the table or the top-camera support structure, is counted as a task failure.

\subsection{Hardware and Software}
\label{appendix:hardware-software}

\begin{table}[H]
\centering
\small
\caption{Hardware used for training and evaluation}
\vspace{-8pt}
\setlength{\tabcolsep}{10pt}
\renewcommand{\arraystretch}{1.2}
\resizebox{0.95\textwidth}{!}{%
\begin{tabular}{p{0.29\textwidth} p{0.35\textwidth} p{0.25\textwidth}}
    \toprule
    \textbf{Workload} & \textbf{Hardware} & \textbf{Usage} \\
    \midrule
    Student distillation & NVIDIA B200, 180\,GB & 1 GPU per run \\
    Critic training & NVIDIA B200, 180\,GB & 1 GPU per run \\
    Real-world teacher fine-tuning & NVIDIA B200, 180\,GB & 1 GPU per run \\
    Simulation evaluation & RoboTwin 2.0 \& LIBERO-Plus: \newline \hspace*{1em} NVIDIA A100-SXM4-40GB \newline RoboCasa365 \& LIBERO: \newline \hspace*{1em} NVIDIA B200 & 1 evaluation lane per GPU \\
    Real-world inference & NVIDIA GeForce RTX 4090, 24\,GB & Local robot workstation \\
    \bottomrule
\end{tabular}
}
\vspace{-15pt}
\label{tab:hardware}
\end{table}

The hardware used for training and evaluation is summarized in Table~\ref{tab:hardware}.
All student-distillation and critic-training runs use one GPU per process.

\begin{table}[H]
\centering
\small
\caption{Benchmark-specific software environments. Each benchmark uses its own
Python environment because the simulator dependencies conflict.}
\vspace{-8pt}
\renewcommand{\arraystretch}{1.2}
\resizebox{\textwidth}{!}{%
\begin{tabular}{p{0.18\textwidth} p{0.35\textwidth} p{0.15\textwidth} p{0.30\textwidth}}
    \toprule
    \textbf{Component} & \textbf{RoboTwin 2.0} & \textbf{RoboCasa365} &
    \textbf{LIBERO / LIBERO-Plus} \\
    \midrule
    Benchmark code & RoboTwin, upstream commit \texttt{c3ddfa8b} (23 May 2026);
      camera images rendered with the rasterization shader instead of the
      upstream 32-spp ray-tracing path
      & RoboCasa 1.0.1 (vendored) & LIBERO (vendored) \\
    Simulator & SAPIEN 3.0.0b1 & MuJoCo 3.3.1 & MuJoCo 2.3.7
      ($\pi_{0.5}$, ABot-M0); 3.8.1 (SmolVLA) \\
    Physics timestep & $1/250\,\mathrm{s}$ & Benchmark default & Benchmark default \\
    robosuite & --- & 1.5.2 & 1.4.0 \\
    Motion planning & mplib 0.2.1, nvidia-curobo 0.7.8, warp-lang 1.15.0 & --- & --- \\
    Python & 3.10.20 & 3.11.15 & 3.11.15 \\
    PyTorch & 2.7.1+cu128 & 2.7.1+cu128 & 2.11.0+cu128 ($\pi_{0.5}$);\newline
      2.7.1+cu128 (ABot-M0, SmolVLA) \\
    NumPy & 1.26.4 & 2.2.5 & 1.26.4 \\
    Gymnasium & 0.29.1 & 0.29.1 & 0.29.1 \\
    \bottomrule
\end{tabular}%
}
\vspace{-15pt}
\label{tab:benchmark-software}
\end{table}

Each benchmark keeps a separate Python environment because their simulator dependencies conflict: RoboCasa365 requires robosuite 1.5.2 with MuJoCo 3.3.1, whereas LIBERO is pinned to robosuite 1.4.0 with MuJoCo 2.3.7. Within LIBERO-Plus the SmolVLA evaluator additionally requires a newer MuJoCo than the other two policies. The RoboCasa365 and LIBERO copies are vendored into our codebase without upstream revision metadata, so we report the versions recoverable from the environments rather than a commit hash.

\subsection{Oscillation Metric}
\label{appendix:oscillation-metric}

Task success rate is the primary metric for both simulation and real-world experiments. On the real robot we additionally report a command oscillation metric in Table~\ref{tab:oscilation-real}, because success rate alone does not say how a policy reaches the goal. A one-step student can complete a task while emitting commands that reverse direction from step to step, which stresses the hardware and is not visible in a binary outcome. We therefore measure oscillation as a stability metric for real-robot deployment. 

The commanded joint trajectory carries both the intended motion and any high-frequency jitter on top of it. We separate the two by smoothing the trajectory and taking the residual as the jitter. For each accepted episode, 
\begin{equation}
    \begin{aligned}
        \hat{x}_j &= \mathrm{SG}_{w=15,p=2}(x_j), \\
        \mathrm{Osc} &= \frac{180}{\pi}\cdot\frac{1}{J}\sum_{j=1}^{J}
        \sqrt{
        \frac{1}{N}\sum_{t=1}^{N}
        \left(x_j(t)-\hat{x}_j(t)\right)^2
        },
    \end{aligned}
\label{eq:oscillation}
\end{equation}
where $x_j(t)$ is the commanded position of arm joint $j$ at timestep $t$ and $\mathrm{SG}_{w=15,p=2}$ is a second-order Savitzky-Golay filter with a 15-sample window. At the $30\,\mathrm{Hz}$ command rate that window spans about $0.5\,\mathrm{s}$, so the filter passes deliberate reaching motion and leaves only faster reversals in the residual. The residual RMS is averaged over the $J$ arm joints, excluding the grippers ($J=6$ for unimanual and $J=12$ for bimanual), and converted from radians to degrees, so $\mathrm{Osc}$ is read as a typical command deviation in degrees.

We compute $\mathrm{Osc}$ per episode and report the unweighted mean over the 15 accepted episodes of each task--method pair. Failed episodes are included: a policy that oscillates its way into failure should not be rewarded by having those episodes dropped.

\section{Full Results}
\label{appendix:results}

This section expands the aggregate tables of the main paper into per-task results and adds the two LIBERO benchmarks, which the main paper reports only in summary. All success rates are per-task rates over the episode counts given in Table~\ref{tab:benchmark-config}.

\paragraph{RoboCasa365 (Tables~\ref{tab:robocasa365_pi05_full}--\ref{tab:robocasa365_n16}).}
These four tables are the per-task version of Table~\ref{tab:robocasa365}, for $\pi_{0.5}$ with full fine-tuning, $\pi_{0.5}$ with LoRA ($r{=}16$), GR00T N1.5 and GR00T N1.6. Each keeps the three \texttt{pretrain} suites separate, and the \textit{Avg.} rows reproduce the corresponding entries of the main table. The per-task view shows that the suite averages rest on a few tasks: most \texttt{composite\_unseen} tasks are at $0\%$ for every method.

\paragraph{RoboTwin 2.0 (Tables~\ref{tab:robotwin2_pi05_full} and~\ref{tab:robotwin2_abot_full}).}
These two tables are the per-task version of Table~\ref{tab:robotwin2}, for $\pi_{0.5}$-LoRA and ABot-M0. Tasks are ordered by average demonstration trajectory length and grouped into the same \texttt{short}, \texttt{medium} and \texttt{long} horizons as the main table, so the horizon trend reads directly down each column. The backbones differ at long horizons: $\pi_{0.5}$ loses $14.5$ points from the multi-step to the one-step teacher, which the students partly recover, whereas for ABot-M0 both teachers sit at $36.97\%$ and no such gap exists.

\paragraph{LIBERO (Table~\ref{tab:libero}).}
This table is not summarized in the main paper. It reports success rate for each of the four standard suites and the 40-task average. Excluding SmolVLA, both backbones are near saturation on \texttt{libero\_spatial}, \texttt{libero\_object} and \texttt{libero\_goal}, so methods separate mainly through the long-horizon \texttt{libero\_10}. The one exception is MFD on ABot-M0, which collapses to $0\%$ everywhere.

\paragraph{LIBERO-Plus (Table~\ref{tab:libero-plus}).}
This table is likewise new. It breaks the 10,030 perturbed task instances down by perturbation factor, so each row isolates one axis of robustness. Success rates track the perturbation rather than the distillation method: for $\pi_{0.5}$-LoRA every method stays within about one point of the teacher on each factor, and the spread across factors ($71\%$ on camera viewpoints versus $98\%$ on lighting) far exceeds the spread across methods.

\paragraph{Training curves (\Figref{fig:train_loss}).}
The figure shows one-step distillation losses on RoboCasa365. One caveat is needed to read it: the drifting objective $\mathcal{L}_{\mathrm{Drift}}$ is RMS-normalized, so its value stays near $1$ throughout training by construction and carries no information about optimization progress. A flat drift curve is the expected behavior, not a stalled run. For DriftOPD the total loss is plotted with its drift term (dashed) and critic term $-Q$ (red); only the critic term moves on a meaningful scale, and it is what separates DriftOPD from Drift.

\paragraph{Pre-normalization drift magnitude (\Figref{fig:train_drift_loss}).}
Because that normalization removes the scale of $\mathcal{L}_{\mathrm{Drift}}$, we additionally plot the bandwidth-wise RMS of the drifting term before normalization. This is the quantity that actually changes over training, and shows how the drift field's magnitude is distributed across kernel bandwidths.

\clearpage
\begin{table}[t]
    \centering
    \caption{Per-task success rates of $\pi_{0.5}$ (full fine-tuning) and its one-step students on RoboCasa365. Each task is evaluated over 30 episodes. The multi-step Teacher column is taken from the RoboCasa365 leaderboard (\url{https://robocasa.ai/leaderboard.html}).}
    \setlength{\tabcolsep}{6pt}
    \renewcommand{\arraystretch}{1.1}
    \resizebox{\textwidth}{!}{%
    \begin{tabular}{c c l c c c c c c}
        \toprule
        \multirow{2}{*}{\textbf{Model}} &
        \multirow{2}{*}{\textbf{\shortstack{Teacher\\steps}}} &
        \multirow{2}{*}{\textbf{Task}} &
        \multirow{2}{*}{\textbf{Teacher}} &
        \multicolumn{5}{c}{\textbf{1-step}} \\
        \cmidrule(lr){5-9}
        & & & &
        \textbf{Teacher} &
        \textbf{sCD} &
        \textbf{MFD} &
        \textbf{Drift} &
        \textbf{DriftOPD} \\
        \midrule

        \multicolumn{9}{l}{\textit{\texttt{atomic\_seen}} (18 tasks)} \\
        \multirow{18}{*}{$\pi_{0.5}$}
        & \multirow{18}{*}{10}
        & \texttt{CloseBlenderLid}    & -- & 6.67 & 0.00 & 3.33 & 3.33 & 3.33 \\
        & & \texttt{CloseFridge}    & -- & 23.33 & 46.67 & 46.67 & 30.00 & 56.67 \\
        & & \texttt{CloseToasterOvenDoor}    & -- & 0.00 & 0.00 & 6.67 & 3.33 & 0.00 \\
        & & \texttt{CoffeeSetupMug}    & -- & 13.33 & 3.33 & 10.00 & 13.33 & 6.67 \\
        & & \texttt{NavigateKitchen}    & -- & 0.00 & 0.00 & 3.33 & 0.00 & 0.00 \\
        & & \texttt{OpenCabinet}    & -- & 30.00 & 30.00 & 36.67 & 43.33 & 43.33 \\
        & & \texttt{OpenDrawer}    & -- & 53.33 & 63.33 & 53.33 & 66.67 & 83.33 \\
        & & \texttt{OpenStandMixerHead}    & -- & 63.33 & 30.00 & 26.67 & 30.00 & 53.33 \\
        & & \texttt{PickPlaceCounterToCabinet}    & -- & 43.33 & 46.67 & 66.67 & 70.00 & 76.67 \\
        & & \texttt{PickPlaceCounterToStove}    & -- & 76.67 & 73.33 & 70.00 & 70.00 & 66.67 \\
        & & \texttt{PickPlaceDrawerToCounter}    & -- & 36.67 & 33.33 & 46.67 & 36.67 & 53.33 \\
        & & \texttt{PickPlaceSinkToCounter}    & -- & 86.67 & 86.67 & 76.67 & 93.33 & 90.00 \\
        & & \texttt{PickPlaceToasterToCounter}    & -- & 0.00 & 0.00 & 6.67 & 3.33 & 3.33 \\
        & & \texttt{SlideDishwasherRack}    & -- & 73.33 & 66.67 & 80.00 & 60.00 & 76.67 \\
        & & \texttt{TurnOffStove}    & -- & 6.67 & 3.33 & 6.67 & 0.00 & 6.67 \\
        & & \texttt{TurnOnElectricKettle}    & -- & 0.00 & 0.00 & 3.33 & 0.00 & 6.67 \\
        & & \texttt{TurnOnMicrowave}    & -- & 0.00 & 0.00 & 16.67 & 3.33 & 13.33 \\
        & & \texttt{TurnOnSinkFaucet}    & -- & 50.00 & 26.67 & 60.00 & 53.33 & 50.00 \\
        \cmidrule(lr){3-9}
        & & \textit{Avg. (atomic\_seen)} & 39.6 & 31.30 & 28.33 & 34.44 & 32.22 & \textbf{38.33} \\
        \midrule

        \multicolumn{9}{l}{\textit{\texttt{composite\_seen}} (16 tasks)} \\
        \multirow{16}{*}{$\pi_{0.5}$}
        & \multirow{16}{*}{10}
        & \texttt{DeliverStraw}    & -- & 0.00 & 0.00 & 0.00 & 0.00 & 0.00 \\
        & & \texttt{GetToastedBread}    & -- & 0.00 & 0.00 & 0.00 & 0.00 & 0.00 \\
        & & \texttt{KettleBoiling}    & -- & 0.00 & 3.33 & 0.00 & 0.00 & 3.33 \\
        & & \texttt{LoadDishwasher}    & -- & 6.67 & 0.00 & 10.00 & 20.00 & 16.67 \\
        & & \texttt{PackIdenticalLunches}    & -- & 0.00 & 0.00 & 3.33 & 3.33 & 0.00 \\
        & & \texttt{PreSoakPan}    & -- & 3.33 & 10.00 & 3.33 & 23.33 & 40.00 \\
        & & \texttt{PrepareCoffee}    & -- & 0.00 & 0.00 & 3.33 & 0.00 & 0.00 \\
        & & \texttt{RinseSinkBasin}    & -- & 0.00 & 0.00 & 3.33 & 0.00 & 3.33 \\
        & & \texttt{ScrubCuttingBoard}    & -- & 0.00 & 0.00 & 0.00 & 0.00 & 0.00 \\
        & & \texttt{SearingMeat}    & -- & 0.00 & 0.00 & 0.00 & 0.00 & 0.00 \\
        & & \texttt{SetUpCuttingStation}    & -- & 0.00 & 0.00 & 6.67 & 0.00 & 0.00 \\
        & & \texttt{StackBowlsCabinet}    & -- & 3.33 & 3.33 & 10.00 & 16.67 & 16.67 \\
        & & \texttt{SteamInMicrowave}    & -- & 0.00 & 0.00 & 0.00 & 3.33 & 0.00 \\
        & & \texttt{StirVegetables}    & -- & 0.00 & 0.00 & 0.00 & 0.00 & 3.33 \\
        & & \texttt{StoreLeftoversInBowl}    & -- & 0.00 & 0.00 & 0.00 & 0.00 & 0.00 \\
        & & \texttt{WashLettuce}    & -- & 0.00 & 0.00 & 0.00 & 13.33 & 3.33 \\
        \cmidrule(lr){3-9}
        & & \textit{Avg. (composite\_seen)} & 7.1 & 0.83 & 1.04 & 2.50 & 5.00 & \textbf{5.42} \\
        \midrule

        \multicolumn{9}{l}{\textit{\texttt{composite\_unseen}} (16 tasks)} \\
        \multirow{16}{*}{$\pi_{0.5}$}
        & \multirow{16}{*}{10}
        & \texttt{ArrangeBreadBasket}    & -- & 0.00 & 0.00 & 0.00 & 0.00 & 0.00 \\
        & & \texttt{ArrangeTea}    & -- & 0.00 & 0.00 & 0.00 & 0.00 & 0.00 \\
        & & \texttt{BreadSelection}    & -- & 0.00 & 0.00 & 0.00 & 0.00 & 0.00 \\
        & & \texttt{CategorizeCondiments}    & -- & 0.00 & 0.00 & 0.00 & 0.00 & 0.00 \\
        & & \texttt{CuttingToolSelection}    & -- & 0.00 & 0.00 & 0.00 & 0.00 & 0.00 \\
        & & \texttt{GarnishPancake}    & -- & 0.00 & 0.00 & 0.00 & 0.00 & 0.00 \\
        & & \texttt{GatherTableware}    & -- & 0.00 & 0.00 & 0.00 & 0.00 & 0.00 \\
        & & \texttt{HeatKebabSandwich}    & -- & 0.00 & 0.00 & 0.00 & 0.00 & 0.00 \\
        & & \texttt{MakeIceLemonade}    & -- & 0.00 & 0.00 & 0.00 & 0.00 & 0.00 \\
        & & \texttt{PanTransfer}    & -- & 0.00 & 0.00 & 0.00 & 0.00 & 0.00 \\
        & & \texttt{PortionHotDogs}    & -- & 0.00 & 0.00 & 0.00 & 0.00 & 0.00 \\
        & & \texttt{RecycleBottlesByType}    & -- & 13.33 & 13.33 & 13.33 & 13.33 & 13.33 \\
        & & \texttt{SeparateFreezerRack}    & -- & 0.00 & 0.00 & 0.00 & 0.00 & 0.00 \\
        & & \texttt{WaffleReheat}    & -- & 0.00 & 0.00 & 0.00 & 0.00 & 0.00 \\
        & & \texttt{WashFruitColander}    & -- & 0.00 & 0.00 & 0.00 & 3.33 & 13.33 \\
        & & \texttt{WeighIngredients}    & -- & 0.00 & 0.00 & 0.00 & 0.00 & 0.00 \\
        \cmidrule(lr){3-9}
        & & \textit{Avg. (composite\_unseen)} & 1.2 & 0.83 & 0.83 & 0.83 & 1.04 & \textbf{1.67} \\

        \bottomrule
    \end{tabular}%
    }
\label{tab:robocasa365_pi05_full}
\end{table}

\clearpage
\begin{table}[t]
    \centering
    \caption{Per-task success rates of $\pi_{0.5}$ (LoRA $r{=}16$) and its one-step students on RoboCasa365. Each task is evaluated over 30 episodes. The multi-step Teacher column is taken from the RoboCasa365 leaderboard (\url{https://robocasa.ai/leaderboard.html}).}
    \setlength{\tabcolsep}{6pt}
    \renewcommand{\arraystretch}{1.1}
    \resizebox{\textwidth}{!}{%
    \begin{tabular}{c c l c c c c c c}
        \toprule
        \multirow{2}{*}{\textbf{Model}} &
        \multirow{2}{*}{\textbf{\shortstack{Teacher\\steps}}} &
        \multirow{2}{*}{\textbf{Task}} &
        \multirow{2}{*}{\textbf{Teacher}} &
        \multicolumn{5}{c}{\textbf{1-step}} \\
        \cmidrule(lr){5-9}
        & & & &
        \textbf{Teacher} &
        \textbf{sCD} &
        \textbf{MFD} &
        \textbf{Drift} &
        \textbf{DriftOPD} \\
        \midrule

        \multicolumn{9}{l}{\textit{\texttt{atomic\_seen}} (18 tasks)} \\
        \multirow{18}{*}{\shortstack{$\pi_{0.5}$\\(LoRA $r{=}16$)}}
        & \multirow{18}{*}{10}
        & \texttt{CloseBlenderLid}   & -- & 6.67 & 0.00 & 0.00 & 3.33 & 3.33 \\
        & & \texttt{CloseFridge}   & -- & 23.33 & 10.00 & 43.33 & 20.00 & 43.33 \\
        & & \texttt{CloseToasterOvenDoor}   & -- & 0.00 & 0.00 & 3.33 & 0.00 & 6.67 \\
        & & \texttt{CoffeeSetupMug}   & -- & 13.33 & 10.00 & 10.00 & 16.67 & 10.00 \\
        & & \texttt{NavigateKitchen}   & -- & 0.00 & 0.00 & 3.33 & 0.00 & 0.00 \\
        & & \texttt{OpenCabinet}   & -- & 30.00 & 23.33 & 16.67 & 36.67 & 33.33 \\
        & & \texttt{OpenDrawer}   & -- & 53.33 & 56.67 & 40.00 & 56.67 & 70.00 \\
        & & \texttt{OpenStandMixerHead}   & -- & 63.33 & 60.00 & 6.67 & 53.33 & 53.33 \\
        & & \texttt{PickPlaceCounterToCabinet}   & -- & 43.33 & 60.00 & 53.33 & 46.67 & 66.67 \\
        & & \texttt{PickPlaceCounterToStove}   & -- & 76.67 & 83.33 & 66.67 & 70.00 & 70.00 \\
        & & \texttt{PickPlaceDrawerToCounter}   & -- & 36.67 & 36.67 & 30.00 & 36.67 & 46.67 \\
        & & \texttt{PickPlaceSinkToCounter}   & -- & 86.67 & 70.00 & 76.67 & 93.33 & 96.67 \\
        & & \texttt{PickPlaceToasterToCounter}   & -- & 0.00 & 3.33 & 10.00 & 16.67 & 30.00 \\
        & & \texttt{SlideDishwasherRack}   & -- & 73.33 & 70.00 & 60.00 & 80.00 & 73.33 \\
        & & \texttt{TurnOffStove}   & -- & 6.67 & 0.00 & 10.00 & 6.67 & 10.00 \\
        & & \texttt{TurnOnElectricKettle}   & -- & 0.00 & 0.00 & 13.33 & 3.33 & 10.00 \\
        & & \texttt{TurnOnMicrowave}   & -- & 0.00 & 0.00 & 10.00 & 13.33 & 23.33 \\
        & & \texttt{TurnOnSinkFaucet}   & -- & 50.00 & 43.33 & 56.67 & 60.00 & 60.00 \\
        \cmidrule(lr){3-9}
        & & \textit{Avg. (atomic\_seen)} & 39.6 & 31.30 & 29.26 & 28.33 & 34.07 & \textbf{39.26} \\
        \midrule

        \multicolumn{9}{l}{\textit{\texttt{composite\_seen}} (16 tasks)} \\
        \multirow{16}{*}{\shortstack{$\pi_{0.5}$\\(LoRA $r{=}16$)}}
        & \multirow{16}{*}{10}
        & \texttt{DeliverStraw}   & -- & 0.00 & 0.00 & 0.00 & 0.00 & 0.00 \\
        & & \texttt{GetToastedBread}   & -- & 0.00 & 0.00 & 0.00 & 0.00 & 0.00 \\
        & & \texttt{KettleBoiling}   & -- & 0.00 & 0.00 & 3.33 & 0.00 & 10.00 \\
        & & \texttt{LoadDishwasher}   & -- & 6.67 & 6.67 & 10.00 & 20.00 & 23.33 \\
        & & \texttt{PackIdenticalLunches}   & -- & 0.00 & 0.00 & 0.00 & 0.00 & 0.00 \\
        & & \texttt{PreSoakPan}   & -- & 3.33 & 16.67 & 20.00 & 10.00 & 13.33 \\
        & & \texttt{PrepareCoffee}   & -- & 0.00 & 0.00 & 0.00 & 0.00 & 0.00 \\
        & & \texttt{RinseSinkBasin}   & -- & 0.00 & 3.33 & 0.00 & 3.33 & 3.33 \\
        & & \texttt{ScrubCuttingBoard}   & -- & 0.00 & 0.00 & 0.00 & 0.00 & 0.00 \\
        & & \texttt{SearingMeat}   & -- & 0.00 & 0.00 & 0.00 & 0.00 & 0.00 \\
        & & \texttt{SetUpCuttingStation}   & -- & 0.00 & 0.00 & 0.00 & 0.00 & 3.33 \\
        & & \texttt{StackBowlsCabinet}   & -- & 3.33 & 0.00 & 6.67 & 6.67 & 20.00 \\
        & & \texttt{SteamInMicrowave}   & -- & 0.00 & 0.00 & 0.00 & 0.00 & 0.00 \\
        & & \texttt{StirVegetables}   & -- & 0.00 & 0.00 & 0.00 & 6.67 & 3.33 \\
        & & \texttt{StoreLeftoversInBowl}   & -- & 0.00 & 0.00 & 0.00 & 0.00 & 0.00 \\
        & & \texttt{WashLettuce}   & -- & 0.00 & 0.00 & 3.33 & 3.33 & 3.33 \\
        \cmidrule(lr){3-9}
        & & \textit{Avg. (composite\_seen)} & 7.1 & 0.83 & 1.67 & 2.71 & 3.12 & \textbf{5.00} \\
        \midrule

        \multicolumn{9}{l}{\textit{\texttt{composite\_unseen}} (16 tasks)} \\
        \multirow{16}{*}{\shortstack{$\pi_{0.5}$\\(LoRA $r{=}16$)}}
        & \multirow{16}{*}{10}
        & \texttt{ArrangeBreadBasket}   & -- & 0.00 & 0.00 & 0.00 & 0.00 & 0.00 \\
        & & \texttt{ArrangeTea}   & -- & 0.00 & 0.00 & 0.00 & 0.00 & 0.00 \\
        & & \texttt{BreadSelection}   & -- & 0.00 & 0.00 & 0.00 & 0.00 & 0.00 \\
        & & \texttt{CategorizeCondiments}   & -- & 0.00 & 0.00 & 0.00 & 0.00 & 0.00 \\
        & & \texttt{CuttingToolSelection}   & -- & 0.00 & 0.00 & 3.33 & 6.67 & 6.67 \\
        & & \texttt{GarnishPancake}   & -- & 0.00 & 0.00 & 0.00 & 0.00 & 0.00 \\
        & & \texttt{GatherTableware}   & -- & 0.00 & 0.00 & 0.00 & 0.00 & 0.00 \\
        & & \texttt{HeatKebabSandwich}   & -- & 0.00 & 0.00 & 0.00 & 0.00 & 0.00 \\
        & & \texttt{MakeIceLemonade}   & -- & 0.00 & 0.00 & 0.00 & 0.00 & 0.00 \\
        & & \texttt{PanTransfer}   & -- & 0.00 & 0.00 & 0.00 & 0.00 & 0.00 \\
        & & \texttt{PortionHotDogs}   & -- & 0.00 & 0.00 & 0.00 & 0.00 & 3.33 \\
        & & \texttt{RecycleBottlesByType}   & -- & 13.33 & 13.33 & 13.33 & 13.33 & 13.33 \\
        & & \texttt{SeparateFreezerRack}   & -- & 0.00 & 0.00 & 0.00 & 0.00 & 0.00 \\
        & & \texttt{WaffleReheat}   & -- & 0.00 & 0.00 & 0.00 & 0.00 & 0.00 \\
        & & \texttt{WashFruitColander}   & -- & 0.00 & 0.00 & 3.33 & 3.33 & 0.00 \\
        & & \texttt{WeighIngredients}   & -- & 0.00 & 0.00 & 0.00 & 0.00 & 0.00 \\
        \cmidrule(lr){3-9}
        & & \textit{Avg. (composite\_unseen)} & 1.2 & 0.83 & 0.83 & 1.25 & \textbf{1.46} & \textbf{1.46} \\

        \bottomrule
    \end{tabular}%
    }
\label{tab:robocasa365_pi05_lora16}
\end{table}

\clearpage
\begin{table}[t]
    \centering
    \caption{Per-task success rates of GR00T N1.5 and its one-step students on RoboCasa365. Each task is evaluated over 30 episodes. The multi-step Teacher column is taken from the RoboCasa365 leaderboard (\url{https://robocasa.ai/leaderboard.html}).}
    \setlength{\tabcolsep}{6pt}
    \renewcommand{\arraystretch}{1.1}
    \resizebox{\textwidth}{!}{%
    \begin{tabular}{c c l c c c c c c}
        \toprule
        \multirow{2}{*}{\textbf{Model}} &
        \multirow{2}{*}{\textbf{\shortstack{Teacher\\steps}}} &
        \multirow{2}{*}{\textbf{Task}} &
        \multirow{2}{*}{\textbf{Teacher}} &
        \multicolumn{5}{c}{\textbf{1-step}} \\
        \cmidrule(lr){5-9}
        & & & &
        \textbf{Teacher} &
        \textbf{sCD} &
        \textbf{MFD} &
        \textbf{Drift} &
        \textbf{DriftOPD} \\
        \midrule

        \multicolumn{9}{l}{\textit{\texttt{atomic\_seen}} (18 tasks)} \\
        \multirow{18}{*}{GR00T N1.5}
        & \multirow{18}{*}{4}
        & \texttt{CloseBlenderLid}   & -- & 6.67 & 3.33 & 10.00 & 16.67 & 26.67 \\
        & & \texttt{CloseFridge}   & -- & 73.33 & 60.00 & 90.00 & 56.67 & 70.00 \\
        & & \texttt{CloseToasterOvenDoor}   & -- & 43.33 & 33.33 & 53.33 & 23.33 & 26.67 \\
        & & \texttt{CoffeeSetupMug}   & -- & 16.67 & 16.67 & 3.33 & 33.33 & 16.67 \\
        & & \texttt{NavigateKitchen}   & -- & 6.67 & 0.00 & 3.33 & 3.33 & 6.67 \\
        & & \texttt{OpenCabinet}   & -- & 56.67 & 36.67 & 63.33 & 73.33 & 53.33 \\
        & & \texttt{OpenDrawer}   & -- & 53.33 & 43.33 & 56.67 & 46.67 & 63.33 \\
        & & \texttt{OpenStandMixerHead}   & -- & 80.00 & 73.33 & 70.00 & 86.67 & 80.00 \\
        & & \texttt{PickPlaceCounterToCabinet}   & -- & 63.33 & 53.33 & 73.33 & 70.00 & 86.67 \\
        & & \texttt{PickPlaceCounterToStove}   & -- & 80.00 & 86.67 & 56.67 & 70.00 & 73.33 \\
        & & \texttt{PickPlaceDrawerToCounter}   & -- & 46.67 & 53.33 & 63.33 & 33.33 & 36.67 \\
        & & \texttt{PickPlaceSinkToCounter}   & -- & 80.00 & 73.33 & 83.33 & 90.00 & 96.67 \\
        & & \texttt{PickPlaceToasterToCounter}   & -- & 63.33 & 33.33 & 43.33 & 50.00 & 43.33 \\
        & & \texttt{SlideDishwasherRack}   & -- & 30.00 & 53.33 & 43.33 & 50.00 & 63.33 \\
        & & \texttt{TurnOffStove}   & -- & 6.67 & 10.00 & 3.33 & 3.33 & 13.33 \\
        & & \texttt{TurnOnElectricKettle}   & -- & 40.00 & 53.33 & 43.33 & 30.00 & 23.33 \\
        & & \texttt{TurnOnMicrowave}   & -- & 0.00 & 16.67 & 13.33 & 20.00 & 6.67 \\
        & & \texttt{TurnOnSinkFaucet}   & -- & 26.67 & 13.33 & 10.00 & 33.33 & 30.00 \\
        \cmidrule(lr){3-9}
        & & \textit{Avg. (atomic\_seen)} & 50.7 & 42.96 & 39.63 & 43.52 & 43.89 & \textbf{45.37} \\
        \midrule

        \multicolumn{9}{l}{\textit{\texttt{composite\_seen}} (16 tasks)} \\
        \multirow{16}{*}{GR00T N1.5}
        & \multirow{16}{*}{4}
        & \texttt{DeliverStraw}   & -- & 0.00 & 0.00 & 0.00 & 0.00 & 0.00 \\
        & & \texttt{GetToastedBread}   & -- & 0.00 & 0.00 & 0.00 & 0.00 & 0.00 \\
        & & \texttt{KettleBoiling}   & -- & 6.67 & 3.33 & 3.33 & 6.67 & 6.67 \\
        & & \texttt{LoadDishwasher}   & -- & 33.33 & 30.00 & 23.33 & 10.00 & 20.00 \\
        & & \texttt{PackIdenticalLunches}   & -- & 10.00 & 0.00 & 3.33 & 3.33 & 0.00 \\
        & & \texttt{PreSoakPan}   & -- & 10.00 & 10.00 & 33.33 & 13.33 & 13.33 \\
        & & \texttt{PrepareCoffee}   & -- & 0.00 & 0.00 & 0.00 & 0.00 & 0.00 \\
        & & \texttt{RinseSinkBasin}   & -- & 6.67 & 6.67 & 26.67 & 13.33 & 10.00 \\
        & & \texttt{ScrubCuttingBoard}   & -- & 36.67 & 6.67 & 13.33 & 13.33 & 20.00 \\
        & & \texttt{SearingMeat}   & -- & 0.00 & 0.00 & 0.00 & 0.00 & 0.00 \\
        & & \texttt{SetUpCuttingStation}   & -- & 0.00 & 6.67 & 13.33 & 10.00 & 6.67 \\
        & & \texttt{StackBowlsCabinet}   & -- & 6.67 & 0.00 & 6.67 & 20.00 & 26.67 \\
        & & \texttt{SteamInMicrowave}   & -- & 0.00 & 0.00 & 3.33 & 0.00 & 3.33 \\
        & & \texttt{StirVegetables}   & -- & 0.00 & 0.00 & 0.00 & 6.67 & 6.67 \\
        & & \texttt{StoreLeftoversInBowl}   & -- & 3.33 & 0.00 & 0.00 & 0.00 & 0.00 \\
        & & \texttt{WashLettuce}   & -- & 10.00 & 3.33 & 10.00 & 33.33 & 23.33 \\
        \cmidrule(lr){3-9}
        & & \textit{Avg. (composite\_seen)} & 14.8 & 7.71 & 4.17 & \textbf{8.54} & 8.12 & \textbf{8.54} \\
        \midrule

        \multicolumn{9}{l}{\textit{\texttt{composite\_unseen}} (16 tasks)} \\
        \multirow{16}{*}{GR00T N1.5}
        & \multirow{16}{*}{4}
        & \texttt{ArrangeBreadBasket}   & -- & 0.00 & 0.00 & 0.00 & 0.00 & 0.00 \\
        & & \texttt{ArrangeTea}   & -- & 0.00 & 0.00 & 0.00 & 0.00 & 0.00 \\
        & & \texttt{BreadSelection}   & -- & 16.67 & 0.00 & 0.00 & 3.33 & 13.33 \\
        & & \texttt{CategorizeCondiments}   & -- & 0.00 & 0.00 & 0.00 & 3.33 & 0.00 \\
        & & \texttt{CuttingToolSelection}   & -- & 0.00 & 0.00 & 6.67 & 0.00 & 0.00 \\
        & & \texttt{GarnishPancake}   & -- & 0.00 & 0.00 & 3.33 & 0.00 & 3.33 \\
        & & \texttt{GatherTableware}   & -- & 0.00 & 0.00 & 0.00 & 0.00 & 0.00 \\
        & & \texttt{HeatKebabSandwich}   & -- & 0.00 & 0.00 & 0.00 & 0.00 & 0.00 \\
        & & \texttt{MakeIceLemonade}   & -- & 3.33 & 0.00 & 3.33 & 0.00 & 0.00 \\
        & & \texttt{PanTransfer}   & -- & 0.00 & 0.00 & 0.00 & 0.00 & 0.00 \\
        & & \texttt{PortionHotDogs}   & -- & 0.00 & 0.00 & 6.67 & 3.33 & 6.67 \\
        & & \texttt{RecycleBottlesByType}   & -- & 6.67 & 16.67 & 3.33 & 16.67 & 16.67 \\
        & & \texttt{SeparateFreezerRack}   & -- & 0.00 & 0.00 & 0.00 & 0.00 & 0.00 \\
        & & \texttt{WaffleReheat}   & -- & 6.67 & 3.33 & 6.67 & 6.67 & 3.33 \\
        & & \texttt{WashFruitColander}   & -- & 3.33 & 0.00 & 10.00 & 0.00 & 3.33 \\
        & & \texttt{WeighIngredients}   & -- & 3.33 & 6.67 & 3.33 & 3.33 & 0.00 \\
        \cmidrule(lr){3-9}
        & & \textit{Avg. (composite\_unseen)} & 2.7 & 2.50 & 1.67 & 2.71 & 2.29 & \textbf{2.92} \\

        \bottomrule
    \end{tabular}%
    }
\label{tab:robocasa365_n15}
\end{table}

\clearpage
\begin{table}[t]
    \centering
    \caption{Per-task success rates of GR00T N1.6 and its one-step students on RoboCasa365. Each task is evaluated over 30 episodes. The multi-step Teacher column is taken from the RoboCasa365 leaderboard (\url{https://robocasa.ai/leaderboard.html}).}
    \setlength{\tabcolsep}{6pt}
    \renewcommand{\arraystretch}{1.1}
    \resizebox{\textwidth}{!}{%
    \begin{tabular}{c c l c c c c c c}
        \toprule
        \multirow{2}{*}{\textbf{Model}} &
        \multirow{2}{*}{\textbf{\shortstack{Teacher\\steps}}} &
        \multirow{2}{*}{\textbf{Task}} &
        \multirow{2}{*}{\textbf{Teacher}} &
        \multicolumn{5}{c}{\textbf{1-step}} \\
        \cmidrule(lr){5-9}
        & & & &
        \textbf{Teacher} &
        \textbf{sCD} &
        \textbf{MFD} &
        \textbf{Drift} &
        \textbf{DriftOPD} \\
        \midrule

        \multicolumn{9}{l}{\textit{\texttt{atomic\_seen}} (18 tasks)} \\
        \multirow{18}{*}{GR00T N1.6}
        & \multirow{18}{*}{4}
        & \texttt{CloseBlenderLid}   & -- & 3.33 & 6.67 & 0.00 & 10.00 & 13.33 \\
        & & \texttt{CloseFridge}   & -- & 66.67 & 80.00 & 63.33 & 90.00 & 76.67 \\
        & & \texttt{CloseToasterOvenDoor}   & -- & 56.67 & 60.00 & 30.00 & 66.67 & 50.00 \\
        & & \texttt{CoffeeSetupMug}   & -- & 36.67 & 36.67 & 3.33 & 20.00 & 26.67 \\
        & & \texttt{NavigateKitchen}   & -- & 6.67 & 0.00 & 0.00 & 16.67 & 16.67 \\
        & & \texttt{OpenCabinet}   & -- & 30.00 & 33.33 & 33.33 & 66.67 & 63.33 \\
        & & \texttt{OpenDrawer}   & -- & 53.33 & 43.33 & 6.67 & 23.33 & 26.67 \\
        & & \texttt{OpenStandMixerHead}   & -- & 83.33 & 86.67 & 80.00 & 86.67 & 93.33 \\
        & & \texttt{PickPlaceCounterToCabinet}   & -- & 56.67 & 63.33 & 53.33 & 60.00 & 53.33 \\
        & & \texttt{PickPlaceCounterToStove}   & -- & 60.00 & 63.33 & 43.33 & 56.67 & 60.00 \\
        & & \texttt{PickPlaceDrawerToCounter}   & -- & 23.33 & 36.67 & 10.00 & 30.00 & 30.00 \\
        & & \texttt{PickPlaceSinkToCounter}   & -- & 80.00 & 83.33 & 80.00 & 76.67 & 73.33 \\
        & & \texttt{PickPlaceToasterToCounter}   & -- & 50.00 & 56.67 & 43.33 & 40.00 & 60.00 \\
        & & \texttt{SlideDishwasherRack}   & -- & 46.67 & 80.00 & 43.33 & 70.00 & 66.67 \\
        & & \texttt{TurnOffStove}   & -- & 3.33 & 6.67 & 0.00 & 3.33 & 13.33 \\
        & & \texttt{TurnOnElectricKettle}   & -- & 70.00 & 70.00 & 73.33 & 56.67 & 90.00 \\
        & & \texttt{TurnOnMicrowave}   & -- & 43.33 & 63.33 & 23.33 & 46.67 & 33.33 \\
        & & \texttt{TurnOnSinkFaucet}   & -- & 36.67 & 60.00 & 6.67 & 30.00 & 43.33 \\
        \cmidrule(lr){3-9}
        & & \textit{Avg. (atomic\_seen)} & 51.1 & 44.81 & \textbf{51.67} & 32.96 & 47.22 & 49.44 \\
        \midrule

        \multicolumn{9}{l}{\textit{\texttt{composite\_seen}} (16 tasks)} \\
        \multirow{16}{*}{GR00T N1.6}
        & \multirow{16}{*}{4}
        & \texttt{DeliverStraw}   & -- & 0.00 & 0.00 & 0.00 & 0.00 & 0.00 \\
        & & \texttt{GetToastedBread}   & -- & 0.00 & 0.00 & 0.00 & 0.00 & 0.00 \\
        & & \texttt{KettleBoiling}   & -- & 0.00 & 6.67 & 0.00 & 0.00 & 6.67 \\
        & & \texttt{LoadDishwasher}   & -- & 10.00 & 3.33 & 0.00 & 6.67 & 10.00 \\
        & & \texttt{PackIdenticalLunches}   & -- & 0.00 & 0.00 & 0.00 & 10.00 & 6.67 \\
        & & \texttt{PreSoakPan}   & -- & 3.33 & 6.67 & 3.33 & 3.33 & 3.33 \\
        & & \texttt{PrepareCoffee}   & -- & 0.00 & 0.00 & 0.00 & 0.00 & 0.00 \\
        & & \texttt{RinseSinkBasin}   & -- & 13.33 & 13.33 & 3.33 & 10.00 & 16.67 \\
        & & \texttt{ScrubCuttingBoard}   & -- & 6.67 & 10.00 & 6.67 & 20.00 & 20.00 \\
        & & \texttt{SearingMeat}   & -- & 0.00 & 0.00 & 0.00 & 0.00 & 0.00 \\
        & & \texttt{SetUpCuttingStation}   & -- & 0.00 & 0.00 & 0.00 & 3.33 & 0.00 \\
        & & \texttt{StackBowlsCabinet}   & -- & 10.00 & 6.67 & 0.00 & 0.00 & 6.67 \\
        & & \texttt{SteamInMicrowave}   & -- & 0.00 & 0.00 & 0.00 & 3.33 & 3.33 \\
        & & \texttt{StirVegetables}   & -- & 0.00 & 3.33 & 0.00 & 0.00 & 3.33 \\
        & & \texttt{StoreLeftoversInBowl}   & -- & 0.00 & 0.00 & 0.00 & 0.00 & 0.00 \\
        & & \texttt{WashLettuce}   & -- & 0.00 & 6.67 & 0.00 & 6.67 & 3.33 \\
        \cmidrule(lr){3-9}
        & & \textit{Avg. (composite\_seen)} & 9.4 & 2.71 & 3.54 & 0.83 & 3.96 & \textbf{5.00} \\
        \midrule

        \multicolumn{9}{l}{\textit{\texttt{composite\_unseen}} (16 tasks)} \\
        \multirow{16}{*}{GR00T N1.6}
        & \multirow{16}{*}{4}
        & \texttt{ArrangeBreadBasket}   & -- & 0.00 & 0.00 & 0.00 & 0.00 & 0.00 \\
        & & \texttt{ArrangeTea}   & -- & 0.00 & 0.00 & 0.00 & 0.00 & 0.00 \\
        & & \texttt{BreadSelection}   & -- & 3.33 & 3.33 & 0.00 & 0.00 & 3.33 \\
        & & \texttt{CategorizeCondiments}   & -- & 0.00 & 0.00 & 0.00 & 0.00 & 0.00 \\
        & & \texttt{CuttingToolSelection}   & -- & 0.00 & 0.00 & 0.00 & 0.00 & 0.00 \\
        & & \texttt{GarnishPancake}   & -- & 0.00 & 0.00 & 0.00 & 0.00 & 0.00 \\
        & & \texttt{GatherTableware}   & -- & 0.00 & 0.00 & 0.00 & 0.00 & 0.00 \\
        & & \texttt{HeatKebabSandwich}   & -- & 0.00 & 0.00 & 0.00 & 0.00 & 0.00 \\
        & & \texttt{MakeIceLemonade}   & -- & 0.00 & 0.00 & 0.00 & 3.33 & 0.00 \\
        & & \texttt{PanTransfer}   & -- & 0.00 & 0.00 & 0.00 & 0.00 & 0.00 \\
        & & \texttt{PortionHotDogs}   & -- & 0.00 & 0.00 & 0.00 & 0.00 & 0.00 \\
        & & \texttt{RecycleBottlesByType}   & -- & 6.67 & 6.67 & 16.67 & 6.67 & 13.33 \\
        & & \texttt{SeparateFreezerRack}   & -- & 0.00 & 0.00 & 0.00 & 0.00 & 0.00 \\
        & & \texttt{WaffleReheat}   & -- & 0.00 & 0.00 & 0.00 & 0.00 & 0.00 \\
        & & \texttt{WashFruitColander}   & -- & 3.33 & 0.00 & 0.00 & 0.00 & 0.00 \\
        & & \texttt{WeighIngredients}   & -- & 0.00 & 0.00 & 0.00 & 0.00 & 0.00 \\
        \cmidrule(lr){3-9}
        & & \textit{Avg. (composite\_unseen)} & 1.7 & 0.83 & 0.62 & \textbf{1.04} & 0.62 & \textbf{1.04} \\

        \bottomrule
    \end{tabular}%
    }
\label{tab:robocasa365_n16}
\end{table}

\clearpage
\begin{table}[t]
    \centering
    \caption{Per-task success rates of $\pi_{0.5}$ and its LoRA-based ($r=16$) student policies on RoboTwin 2.0 with the ALOHA 14-DoF embodiment. Tasks are ordered by average demonstration trajectory length and grouped into \texttt{short}, \texttt{medium} and \texttt{long} horizons. Each task is evaluated over 30 episodes from a shared accepted-seed manifest.}
    \setlength{\tabcolsep}{6pt}
    \renewcommand{\arraystretch}{1.1}
    \resizebox{\textwidth}{!}{%
    \begin{tabular}{c c l c c c c c c}
        \toprule
        \multirow{2}{*}{\textbf{Model}} &
        \multirow{2}{*}{\textbf{Teacher steps}} &
        \multirow{2}{*}{\textbf{Task}} &
        \multirow{2}{*}{\textbf{Teacher}} &
        \multicolumn{5}{c}{\textbf{1-step}} \\
        \cmidrule(lr){5-9}
        & & & &
        \textbf{Teacher} &
        \textbf{sCD} &
        \textbf{MFD} &
        \textbf{Drift} &
        \textbf{DriftOPD} \\
        \midrule

        \multicolumn{9}{l}{\textit{\texttt{short}} (18 tasks, 77--147 demo steps)} \\
        \multirow{19}{*}{$\pi_{0.5}$-LoRA}
        & \multirow{19}{*}{10}
        & \texttt{click\_bell}                 & 30.00 & 80.00 & 90.00 & 100.00 & 86.67 & 86.67 \\
        & & \texttt{click\_alarmclock}           & 50.00 & 80.00 & 73.33 & 100.00 & 83.33 & 86.67 \\
        & & \texttt{grab\_roller}                & 100.00 & 100.00 & 100.00 & 100.00 & 100.00 & 100.00 \\
        & & \texttt{turn\_switch}                & 40.00 & 50.00 & 56.67 & 66.67 & 43.33 & 50.00 \\
        & & \texttt{lift\_pot}                   & 33.33 & 33.33 & 50.00 & 43.33 & 40.00 & 60.00 \\
        & & \texttt{beat\_block\_hammer}         & 53.33 & 36.67 & 16.67 & 26.67 & 63.33 & 40.00 \\
        & & \texttt{move\_playingcard\_away}     & 90.00 & 90.00 & 93.33 & 70.00 & 83.33 & 93.33 \\
        & & \texttt{press\_stapler}              & 90.00 & 86.67 & 96.67 & 100.00 & 93.33 & 100.00 \\
        & & \texttt{pick\_diverse\_bottles}      & 46.67 & 36.67 & 53.33 & 30.00 & 53.33 & 70.00 \\
        & & \texttt{pick\_dual\_bottles}         & 56.67 & 60.00 & 50.00 & 50.00 & 66.67 & 80.00 \\
        & & \texttt{place\_phone\_stand}         & 33.33 & 33.33 & 33.33 & 10.00 & 30.00 & 36.67 \\
        & & \texttt{place\_object\_stand}        & 90.00 & 83.33 & 90.00 & 66.67 & 96.67 & 96.67 \\
        & & \texttt{adjust\_bottle}              & 93.33 & 96.67 & 93.33 & 63.33 & 90.00 & 86.67 \\
        & & \texttt{place\_object\_scale}        & 56.67 & 60.00 & 60.00 & 33.33 & 70.00 & 56.67 \\
        & & \texttt{move\_pillbottle\_pad}       & 70.00 & 60.00 & 70.00 & 23.33 & 80.00 & 53.33 \\
        & & \texttt{place\_a2b\_right}           & 63.33 & 66.67 & 56.67 & 66.67 & 70.00 & 60.00 \\
        & & \texttt{place\_fan}                  & 66.67 & 33.33 & 40.00 & 10.00 & 30.00 & 40.00 \\
        & & \texttt{place\_mouse\_pad}           & 30.00 & 16.67 & 43.33 & 20.00 & 13.33 & 23.33 \\
        \cmidrule(lr){3-9}
        & & \textit{Avg. (short)} & 60.74 & 61.30 & 64.81 & 54.44 & 66.30 & 67.78 \\
        \midrule

        \multicolumn{9}{l}{\textit{\texttt{medium}} (21 tasks, 149--278 demo steps)} \\
        \multirow{22}{*}{$\pi_{0.5}$-LoRA}
        & \multirow{22}{*}{10}
        & \texttt{stamp\_seal}                 & 26.67 & 23.33 & 23.33 & 23.33 & 20.00 & 30.00 \\
        & & \texttt{place\_a2b\_left}            & 66.67 & 50.00 & 63.33 & 60.00 & 66.67 & 63.33 \\
        & & \texttt{move\_can\_pot}              & 93.33 & 63.33 & 70.00 & 13.33 & 70.00 & 70.00 \\
        & & \texttt{rotate\_qrcode}              & 76.67 & 63.33 & 66.67 & 43.33 & 73.33 & 70.00 \\
        & & \texttt{move\_stapler\_pad}          & 20.00 & 30.00 & 23.33 & 23.33 & 40.00 & 33.33 \\
        & & \texttt{place\_container\_plate}     & 96.67 & 86.67 & 93.33 & 80.00 & 96.67 & 90.00 \\
        & & \texttt{place\_bread\_skillet}       & 76.67 & 70.00 & 63.33 & 70.00 & 66.67 & 73.33 \\
        & & \texttt{scan\_object}                & 23.33 & 13.33 & 20.00 & 10.00 & 33.33 & 43.33 \\
        & & \texttt{place\_empty\_cup}           & 96.67 & 93.33 & 90.00 & 83.33 & 100.00 & 96.67 \\
        & & \texttt{place\_shoe}                 & 43.33 & 23.33 & 36.67 & 33.33 & 33.33 & 40.00 \\
        & & \texttt{open\_laptop}                & 93.33 & 90.00 & 90.00 & 86.67 & 93.33 & 83.33 \\
        & & \texttt{handover\_mic}               & 46.67 & 30.00 & 36.67 & 16.67 & 46.67 & 53.33 \\
        & & \texttt{place\_dual\_shoes}          & 36.67 & 26.67 & 13.33 & 10.00 & 3.33 & 13.33 \\
        & & \texttt{place\_bread\_basket}        & 63.33 & 56.67 & 80.00 & 56.67 & 66.67 & 70.00 \\
        & & \texttt{place\_burger\_fries}        & 90.00 & 90.00 & 93.33 & 90.00 & 93.33 & 93.33 \\
        & & \texttt{dump\_bin\_bigbin}           & 86.67 & 86.67 & 70.00 & 60.00 & 70.00 & 83.33 \\
        & & \texttt{place\_object\_basket}       & 83.33 & 50.00 & 73.33 & 40.00 & 63.33 & 66.67 \\
        & & \texttt{shake\_bottle}               & 100.00 & 100.00 & 100.00 & 86.67 & 100.00 & 96.67 \\
        & & \texttt{place\_can\_basket}          & 66.67 & 26.67 & 43.33 & 26.67 & 50.00 & 60.00 \\
        & & \texttt{put\_object\_cabinet}        & 26.67 & 10.00 & 36.67 & 6.67 & 13.33 & 16.67 \\
        & & \texttt{shake\_bottle\_horizontally} & 100.00 & 100.00 & 96.67 & 93.33 & 100.00 & 100.00 \\
        \cmidrule(lr){3-9}
        & & \textit{Avg. (medium)} & 67.30 & 56.35 & 61.11 & 48.25 & 61.90 & 64.13 \\
        \midrule

        \multicolumn{9}{l}{\textit{\texttt{long}} (11 tasks, 282--625 demo steps)} \\
        \multirow{12}{*}{$\pi_{0.5}$-LoRA}
        & \multirow{12}{*}{10}
        & \texttt{handover\_block}             & 30.00 & 0.00 & 6.67 & 0.00 & 3.33 & 6.67 \\
        & & \texttt{place\_cans\_plasticbox}     & 83.33 & 83.33 & 53.33 & 13.33 & 86.67 & 90.00 \\
        & & \texttt{stack\_bowls\_two}           & 86.67 & 76.67 & 73.33 & 40.00 & 73.33 & 86.67 \\
        & & \texttt{stack\_blocks\_two}          & 93.33 & 53.33 & 76.67 & 10.00 & 76.67 & 63.33 \\
        & & \texttt{hanging\_mug}                & 10.00 & 3.33 & 3.33 & 0.00 & 6.67 & 6.67 \\
        & & \texttt{blocks\_ranking\_rgb}        & 90.00 & 70.00 & 66.67 & 30.00 & 90.00 & 73.33 \\
        & & \texttt{blocks\_ranking\_size}       & 43.33 & 33.33 & 40.00 & 23.33 & 43.33 & 43.33 \\
        & & \texttt{stack\_bowls\_three}         & 56.67 & 33.33 & 26.67 & 3.33 & 53.33 & 46.67 \\
        & & \texttt{stack\_blocks\_three}        & 20.00 & 10.00 & 3.33 & 0.00 & 6.67 & 16.67 \\
        & & \texttt{open\_microwave}             & 83.33 & 76.67 & 56.67 & 56.67 & 76.67 & 73.33 \\
        & & \texttt{put\_bottles\_dustbin}       & 10.00 & 6.67 & 0.00 & 3.33 & 16.67 & 6.67 \\
        \cmidrule(lr){3-9}
        & & \textit{Avg. (long)} & 55.15 & 40.61 & 36.97 & 16.36 & 48.48 & 46.67 \\
        \midrule
        \multicolumn{3}{c}{\textbf{Avg. (50 tasks)}} &
        \textbf{62.27} &
        \textbf{54.67} &
        \textbf{57.13} &
        \textbf{43.47} &
        \textbf{60.53} &
        \textbf{61.60} \\
        \bottomrule
    \end{tabular}%
    }
\label{tab:robotwin2_pi05_full}
\end{table}

\clearpage
\begin{table}[t]
    \centering
    \caption{Per-task success rates of ABot-M0 and its student policies on RoboTwin 2.0 with the ALOHA 14-DoF embodiment. Tasks are ordered by average demonstration trajectory length and grouped into \texttt{short}, \texttt{medium} and \texttt{long} horizons. Each task is evaluated over 30 episodes from a shared accepted-seed manifest.}
    \setlength{\tabcolsep}{6pt}
    \renewcommand{\arraystretch}{1.1}
    \resizebox{\textwidth}{!}{%
    \begin{tabular}{c c l c c c c c c}
        \toprule
        \multirow{2}{*}{\textbf{Model}} &
        \multirow{2}{*}{\textbf{Teacher steps}} &
        \multirow{2}{*}{\textbf{Task}} &
        \multirow{2}{*}{\textbf{Teacher}} &
        \multicolumn{5}{c}{\textbf{1-step}} \\
        \cmidrule(lr){5-9}
        & & & &
        \textbf{Teacher} &
        \textbf{sCD} &
        \textbf{MFD} &
        \textbf{Drift} &
        \textbf{DriftOPD} \\
        \midrule

        \multicolumn{9}{l}{\textit{\texttt{short}} (18 tasks, 77--147 demo steps)} \\
        \multirow{19}{*}{ABot-M0}
        & \multirow{19}{*}{10}
        & \texttt{click\_bell}                 & 96.67 & 100.00 & 90.00 & 0.00 & 80.00 & 93.33 \\
        & & \texttt{click\_alarmclock}           & 100.00 & 100.00 & 96.67 & 3.33 & 93.33 & 90.00 \\
        & & \texttt{grab\_roller}                & 100.00 & 93.33 & 83.33 & 0.00 & 93.33 & 90.00 \\
        & & \texttt{turn\_switch}                & 80.00 & 63.33 & 70.00 & 23.33 & 73.33 & 53.33 \\
        & & \texttt{lift\_pot}                   & 100.00 & 93.33 & 93.33 & 0.00 & 90.00 & 6.67 \\
        & & \texttt{beat\_block\_hammer}         & 86.67 & 83.33 & 80.00 & 0.00 & 83.33 & 83.33 \\
        & & \texttt{move\_playingcard\_away}     & 40.00 & 53.33 & 50.00 & 0.00 & 46.67 & 50.00 \\
        & & \texttt{press\_stapler}              & 10.00 & 16.67 & 13.33 & 16.67 & 36.67 & 30.00 \\
        & & \texttt{pick\_diverse\_bottles}      & 33.33 & 53.33 & 66.67 & 0.00 & 43.33 & 53.33 \\
        & & \texttt{pick\_dual\_bottles}         & 46.67 & 63.33 & 73.33 & 0.00 & 56.67 & 56.67 \\
        & & \texttt{place\_phone\_stand}         & 43.33 & 40.00 & 56.67 & 0.00 & 56.67 & 36.67 \\
        & & \texttt{place\_object\_stand}        & 76.67 & 76.67 & 60.00 & 0.00 & 63.33 & 73.33 \\
        & & \texttt{adjust\_bottle}              & 83.33 & 76.67 & 73.33 & 0.00 & 76.67 & 80.00 \\
        & & \texttt{place\_object\_scale}        & 86.67 & 86.67 & 86.67 & 0.00 & 86.67 & 86.67 \\
        & & \texttt{move\_pillbottle\_pad}       & 90.00 & 70.00 & 70.00 & 0.00 & 86.67 & 90.00 \\
        & & \texttt{place\_a2b\_right}           & 76.67 & 83.33 & 80.00 & 0.00 & 73.33 & 80.00 \\
        & & \texttt{place\_fan}                  & 90.00 & 86.67 & 60.00 & 0.00 & 70.00 & 70.00 \\
        & & \texttt{place\_mouse\_pad}           & 23.33 & 26.67 & 10.00 & 0.00 & 33.33 & 40.00 \\
        \cmidrule(lr){3-9}
        & & \textit{Avg. (short)} & 70.19 & 70.37 & 67.41 & 2.41 & 69.07 & 64.63 \\
        \midrule

        \multicolumn{9}{l}{\textit{\texttt{medium}} (21 tasks, 149--278 demo steps)} \\
        \multirow{22}{*}{ABot-M0}
        & \multirow{22}{*}{10}
        & \texttt{stamp\_seal}                 & 50.00 & 50.00 & 40.00 & 0.00 & 50.00 & 60.00 \\
        & & \texttt{place\_a2b\_left}            & 90.00 & 90.00 & 83.33 & 0.00 & 83.33 & 90.00 \\
        & & \texttt{move\_can\_pot}              & 73.33 & 60.00 & 50.00 & 0.00 & 73.33 & 76.67 \\
        & & \texttt{rotate\_qrcode}              & 90.00 & 73.33 & 70.00 & 0.00 & 80.00 & 70.00 \\
        & & \texttt{move\_stapler\_pad}          & 26.67 & 26.67 & 20.00 & 0.00 & 23.33 & 33.33 \\
        & & \texttt{place\_container\_plate}     & 93.33 & 96.67 & 90.00 & 0.00 & 80.00 & 80.00 \\
        & & \texttt{place\_bread\_skillet}       & 80.00 & 60.00 & 63.33 & 0.00 & 53.33 & 63.33 \\
        & & \texttt{scan\_object}                & 50.00 & 36.67 & 50.00 & 0.00 & 46.67 & 53.33 \\
        & & \texttt{place\_empty\_cup}           & 96.67 & 96.67 & 96.67 & 0.00 & 80.00 & 90.00 \\
        & & \texttt{place\_shoe}                 & 46.67 & 36.67 & 50.00 & 0.00 & 40.00 & 43.33 \\
        & & \texttt{open\_laptop}                & 93.33 & 83.33 & 90.00 & 6.67 & 80.00 & 90.00 \\
        & & \texttt{handover\_mic}               & 66.67 & 53.33 & 40.00 & 0.00 & 60.00 & 60.00 \\
        & & \texttt{place\_dual\_shoes}          & 33.33 & 30.00 & 26.67 & 0.00 & 23.33 & 23.33 \\
        & & \texttt{place\_bread\_basket}        & 90.00 & 76.67 & 86.67 & 0.00 & 76.67 & 76.67 \\
        & & \texttt{place\_burger\_fries}        & 90.00 & 86.67 & 76.67 & 0.00 & 76.67 & 76.67 \\
        & & \texttt{dump\_bin\_bigbin}           & 93.33 & 80.00 & 90.00 & 0.00 & 83.33 & 93.33 \\
        & & \texttt{place\_object\_basket}       & 56.67 & 66.67 & 83.33 & 0.00 & 76.67 & 70.00 \\
        & & \texttt{shake\_bottle}               & 100.00 & 100.00 & 100.00 & 46.67 & 100.00 & 100.00 \\
        & & \texttt{place\_can\_basket}          & 73.33 & 70.00 & 66.67 & 0.00 & 63.33 & 53.33 \\
        & & \texttt{put\_object\_cabinet}        & 63.33 & 50.00 & 30.00 & 0.00 & 30.00 & 26.67 \\
        & & \texttt{shake\_bottle\_horizontally} & 96.67 & 100.00 & 100.00 & 33.33 & 100.00 & 100.00 \\
        \cmidrule(lr){3-9}
        & & \textit{Avg. (medium)} & 73.97 & 67.78 & 66.83 & 4.13 & 65.71 & 68.10 \\
        \midrule

        \multicolumn{9}{l}{\textit{\texttt{long}} (11 tasks, 282--625 demo steps)} \\
        \multirow{12}{*}{ABot-M0}
        & \multirow{12}{*}{10}
        & \texttt{handover\_block}             & 0.00 & 3.33 & 0.00 & 0.00 & 6.67 & 3.33 \\
        & & \texttt{place\_cans\_plasticbox}     & 96.67 & 70.00 & 66.67 & 0.00 & 93.33 & 83.33 \\
        & & \texttt{stack\_bowls\_two}           & 100.00 & 96.67 & 86.67 & 0.00 & 83.33 & 86.67 \\
        & & \texttt{stack\_blocks\_two}          & 56.67 & 66.67 & 56.67 & 0.00 & 56.67 & 66.67 \\
        & & \texttt{hanging\_mug}                & 10.00 & 0.00 & 0.00 & 0.00 & 20.00 & 16.67 \\
        & & \texttt{blocks\_ranking\_rgb}        & 26.67 & 46.67 & 33.33 & 0.00 & 30.00 & 33.33 \\
        & & \texttt{blocks\_ranking\_size}       & 10.00 & 6.67 & 6.67 & 0.00 & 0.00 & 3.33 \\
        & & \texttt{stack\_bowls\_three}         & 83.33 & 76.67 & 66.67 & 0.00 & 60.00 & 66.67 \\
        & & \texttt{stack\_blocks\_three}        & 6.67 & 0.00 & 3.33 & 0.00 & 3.33 & 16.67 \\
        & & \texttt{open\_microwave}             & 10.00 & 33.33 & 16.67 & 10.00 & 36.67 & 46.67 \\
        & & \texttt{put\_bottles\_dustbin}       & 6.67 & 6.67 & 0.00 & 0.00 & 3.33 & 3.33 \\
        \cmidrule(lr){3-9}
        & & \textit{Avg. (long)} & 36.97 & 36.97 & 30.61 & 0.91 & 35.76 & 38.79 \\
        \midrule
        \multicolumn{3}{c}{\textbf{Avg. (50 tasks)}} &
        \textbf{64.47} &
        \textbf{61.93} &
        \textbf{59.07} &
        \textbf{2.80} &
        \textbf{60.33} &
        \textbf{60.40} \\
        \bottomrule
    \end{tabular}%
    }
\label{tab:robotwin2_abot_full}
\end{table}

\clearpage
\begin{table}[t]
\centering
\caption{Average task success rates on LIBERO, reported separately for each evaluation suite. Each task is evaluated over 50 episodes under paired environment settings.}
\scriptsize
\setlength{\tabcolsep}{4pt}
\renewcommand{\arraystretch}{1.2}
\resizebox{\textwidth}{!}{%
    \begin{tabular}{c c l c c c c c c}
        \toprule
        \multirow{2}{*}{\textbf{Model}} &
        \multirow{2}{*}{\textbf{Teacher steps}} &
        \multirow{2}{*}{\textbf{Suite}} &
        \multirow{2}{*}{\textbf{Teacher}} &
        \multicolumn{5}{c}{\textbf{1-step}} \\
        \cmidrule(lr){5-9}
        & & & &
        \textbf{Teacher} &
        \textbf{sCD} &
        \textbf{MFD} &
        \textbf{Drift} &
        \textbf{DriftOPD} \\
        \midrule
        \multirow{5}{*}{$\pi_{0.5}$-LoRA} & \multirow{5}{*}{10} & \texttt{libero\_spatial} (10) & 98.8\% & 99.2\% & 98.6\% & 99.2\% & 98.2\% & 98.4\% \\
        & & \texttt{libero\_object} (10) & 98.2\% & 98.4\% & 97.2\% & 97.2\% & 98.4\% & 97.4\% \\
        & & \texttt{libero\_goal} (10) & 98.0\% & 98.2\% & 97.4\% & 98.2\% & 98.0\% & 97.0\% \\
        & & \texttt{libero\_10} (10) & 92.4\% & 93.0\% & 89.2\% & 87.8\% & 91.8\% & 91.8\% \\
        & & \textbf{Average (40)} & \textbf{96.85\%} & \textbf{97.20\%} & \textbf{95.60\%} & \textbf{95.60\%} & \textbf{96.60\%} & \textbf{96.15\%} \\
        \midrule
        \multirow{5}{*}{SmolVLA} & \multirow{5}{*}{10} & \texttt{libero\_spatial} (10) & 66.6\% & 76.6\% & 69.6\% & 36.8\% & 64.0\% & 68.0\% \\
        & & \texttt{libero\_object} (10) & 92.0\% & 95.4\% & 89.4\% & 78.8\% & 90.2\% & 89.6\% \\
        & & \texttt{libero\_goal} (10) & 78.0\% & 78.8\% & 73.6\% & 76.6\% & 76.6\% & 76.6\% \\
        & & \texttt{libero\_10} (10) & 43.0\% & 58.0\% & 43.4\% & 22.0\% & 44.6\% & 41.8\% \\
        & & \textbf{Average (40)} & \textbf{69.90\%} & \textbf{77.20\%} & \textbf{69.00\%} & \textbf{53.55\%} & \textbf{68.85\%} & \textbf{69.00\%} \\
        \midrule
        \multirow{5}{*}{ABot-M0} & \multirow{5}{*}{10} & \texttt{libero\_spatial} (10) & 98.8\% & 98.8\% & 99.2\% & 0.0\% & 99.2\% & 99.6\% \\
        & & \texttt{libero\_object} (10) & 99.8\% & 99.6\% & 99.8\% & 0.0\% & 99.4\% & 99.8\% \\
        & & \texttt{libero\_goal} (10) & 99.0\% & 98.4\% & 98.6\% & 0.0\% & 98.4\% & 98.4\% \\
        & & \texttt{libero\_10} (10) & 96.6\% & 95.2\% & 94.6\% & 0.0\% & 94.0\% & 94.6\% \\
        & & \textbf{Average (40)} & \textbf{98.55\%} & \textbf{98.00\%} & \textbf{98.05\%} & \textbf{0.00\%} & \textbf{97.75\%} & \textbf{98.10\%} \\
        \bottomrule
    \end{tabular}%
}
\label{tab:libero}
\end{table}

\begin{table}[t]
\caption{Average task success rates on LIBERO-Plus, reported separately for each evaluation suite. Each of the 10,030 perturbed task instances is evaluated over a single episode under paired environment settings.}
\centering
\scriptsize
\setlength{\tabcolsep}{4pt}
\renewcommand{\arraystretch}{1.2}
\resizebox{\textwidth}{!}{%
    \begin{tabular}{c c l c c c c c c}
        \toprule
        \multirow{2}{*}{\textbf{Model}} &
        \multirow{2}{*}{\textbf{Teacher steps}} &
        \multirow{2}{*}{\textbf{Suite}} &
        \multirow{2}{*}{\textbf{Teacher}} &
        \multicolumn{5}{c}{\textbf{1-step}} \\
        \cmidrule(lr){5-9}
        & & & &
        \textbf{Teacher} &
        \textbf{sCD} &
        \textbf{MFD} &
        \textbf{Drift} &
        \textbf{DriftOPD} \\
        \midrule
        \multirow{8}{*}{$\pi_{0.5}$-LoRA} & \multirow{8}{*}{10} & \texttt{Sensor Noise} (1,601) & 86.70\% & 84.51\% & 85.38\% & 86.20\% & 85.51\% & 86.38\% \\
        & & \texttt{Camera Viewpoints} (1,599) & 71.36\% & 71.79\% & 70.61\% & 72.61\% & 70.86\% & 71.23\% \\
        & & \texttt{Robot Initial States} (1,550) & 75.16\% & 71.74\% & 73.16\% & 72.97\% & 74.39\% & 74.58\% \\
        & & \texttt{Language Instructions} (1,537) & 84.65\% & 86.21\% & 84.65\% & 85.75\% & 85.88\% & 85.49\% \\
        & & \texttt{Objects Layout} (1,525) & 85.90\% & 85.77\% & 86.69\% & 86.23\% & 86.75\% & 86.56\% \\
        & & \texttt{Light Conditions} (1,142) & 97.81\% & 97.11\% & 96.50\% & 96.58\% & 97.81\% & 96.76\% \\
        & & \texttt{Background Textures} (1,076) & 95.45\% & 96.56\% & 95.26\% & 94.89\% & 96.47\% & 96.93\% \\
        & & \textbf{Average (10,030)} & \textbf{84.24\%} & \textbf{83.69\%} & \textbf{83.55\%} & \textbf{84.04\%} & \textbf{84.28\%} & \textbf{84.35\%} \\
        \midrule
        \multirow{8}{*}{SmolVLA} & \multirow{8}{*}{10} & \texttt{Sensor Noise} (1,601) & 0.00\% & 0.06\% & 0.06\% & 0.12\% & 0.06\% & 0.00\% \\
        & & \texttt{Camera Viewpoints} (1,599) & 0.00\% & 0.00\% & 0.06\% & 0.00\% & 0.00\% & 0.13\% \\
        & & \texttt{Robot Initial States} (1,550) & 0.32\% & 0.26\% & 0.52\% & 0.19\% & 0.52\% & 0.52\% \\
        & & \texttt{Language Instructions} (1,537) & 5.01\% & 6.18\% & 5.99\% & 3.84\% & 4.88\% & 6.57\% \\
        & & \texttt{Objects Layout} (1,525) & 10.62\% & 10.69\% & 9.05\% & 13.51\% & 11.67\% & 11.28\% \\
        & & \texttt{Light Conditions} (1,142) & 14.62\% & 21.28\% & 17.86\% & 12.00\% & 17.51\% & 19.18\% \\
        & & \texttt{Background Textures} (1,076) & 9.29\% & 9.20\% & 9.01\% & 8.27\% & 10.13\% & 10.41\% \\
        & & \textbf{Average (10,030)} & \textbf{5.09\%} & \textbf{6.03\%} & \textbf{5.39\%} & \textbf{4.95\%} & \textbf{5.69\%} & \textbf{6.12\%} \\
        \midrule
        \multirow{8}{*}{ABot-M0} & \multirow{8}{*}{10} & \texttt{Sensor Noise} (1,601) & 86.51\% & 87.26\% & 87.26\% & 0.00\% & 87.07\% & 87.57\% \\
        & & \texttt{Camera Viewpoints} (1,599) & 59.35\% & 58.47\% & 57.41\% & 0.00\% & 57.54\% & 58.29\% \\
        & & \texttt{Robot Initial States} (1,550) & 68.65\% & 67.48\% & 67.29\% & 0.00\% & 67.29\% & 67.74\% \\
        & & \texttt{Language Instructions} (1,537) & 87.18\% & 86.73\% & 86.21\% & 0.00\% & 88.03\% & 87.77\% \\
        & & \texttt{Objects Layout} (1,525) & 83.02\% & 83.21\% & 82.95\% & 0.00\% & 82.95\% & 82.69\% \\
        & & \texttt{Light Conditions} (1,142) & 96.85\% & 95.53\% & 96.32\% & 0.00\% & 95.36\% & 95.18\% \\
        & & \texttt{Background Textures} (1,076) & 93.22\% & 92.01\% & 93.31\% & 0.00\% & 92.75\% & 92.10\% \\
        & & \textbf{Average (10,030)} & \textbf{80.89\%} & \textbf{80.37\%} & \textbf{80.28\%} & \textbf{0.00\%} & \textbf{80.38\%} & \textbf{80.48\%} \\
        \bottomrule
    \end{tabular}%
}
\label{tab:libero-plus}
\end{table}

\clearpage
\begin{figure}[t]
  \centering
   \includegraphics[width=\linewidth]{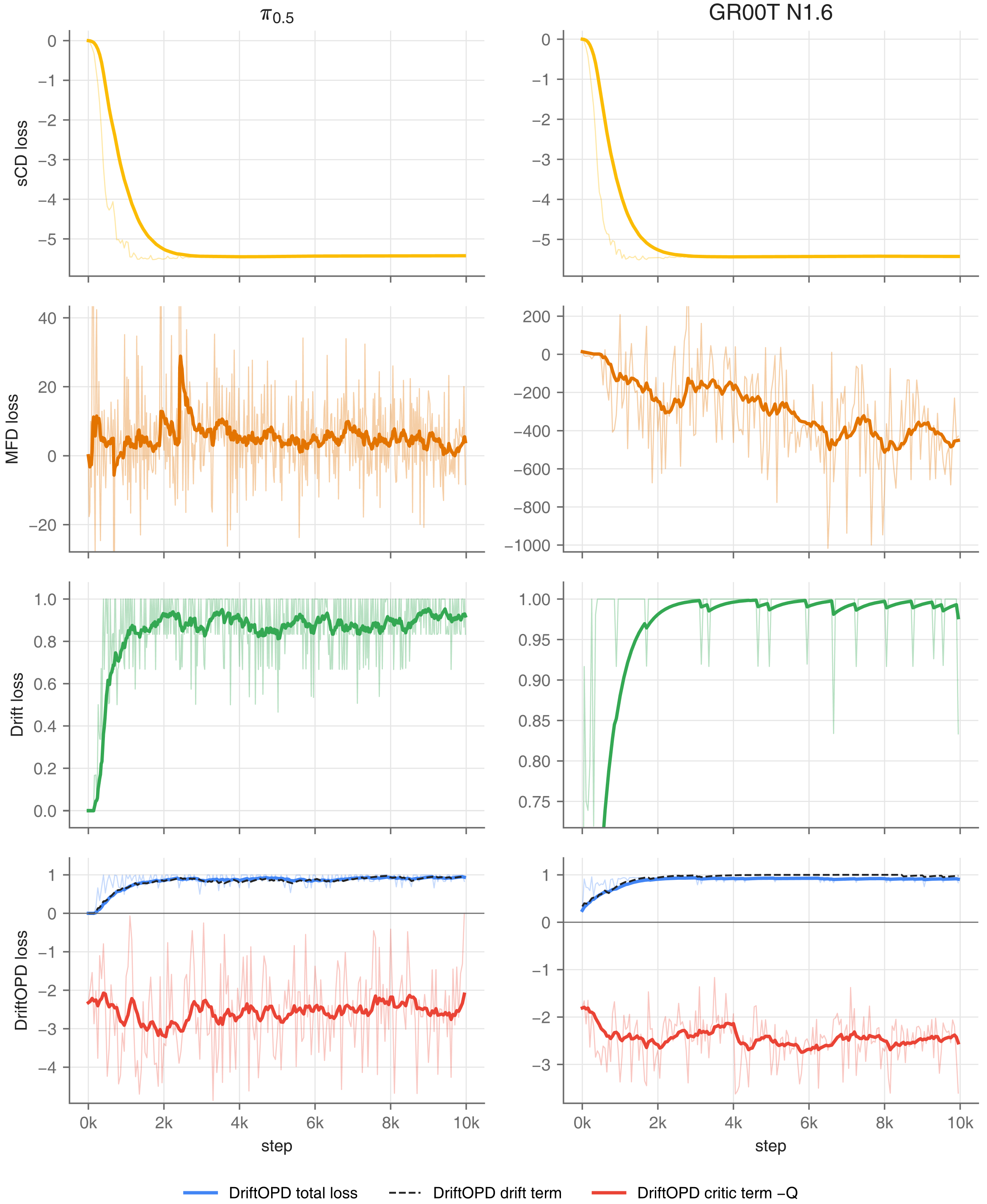}
   \caption{Training curves of one-step distillation on RoboCasa365. Faint lines are logged values and bold lines are EMAs. For DriftOPD, the total loss is shown with its drift term (dashed) and critic term $-Q$ (red).}
   \label{fig:train_loss}
\end{figure}

\clearpage
\begin{figure}[t]
  \centering
   \includegraphics[width=\linewidth]{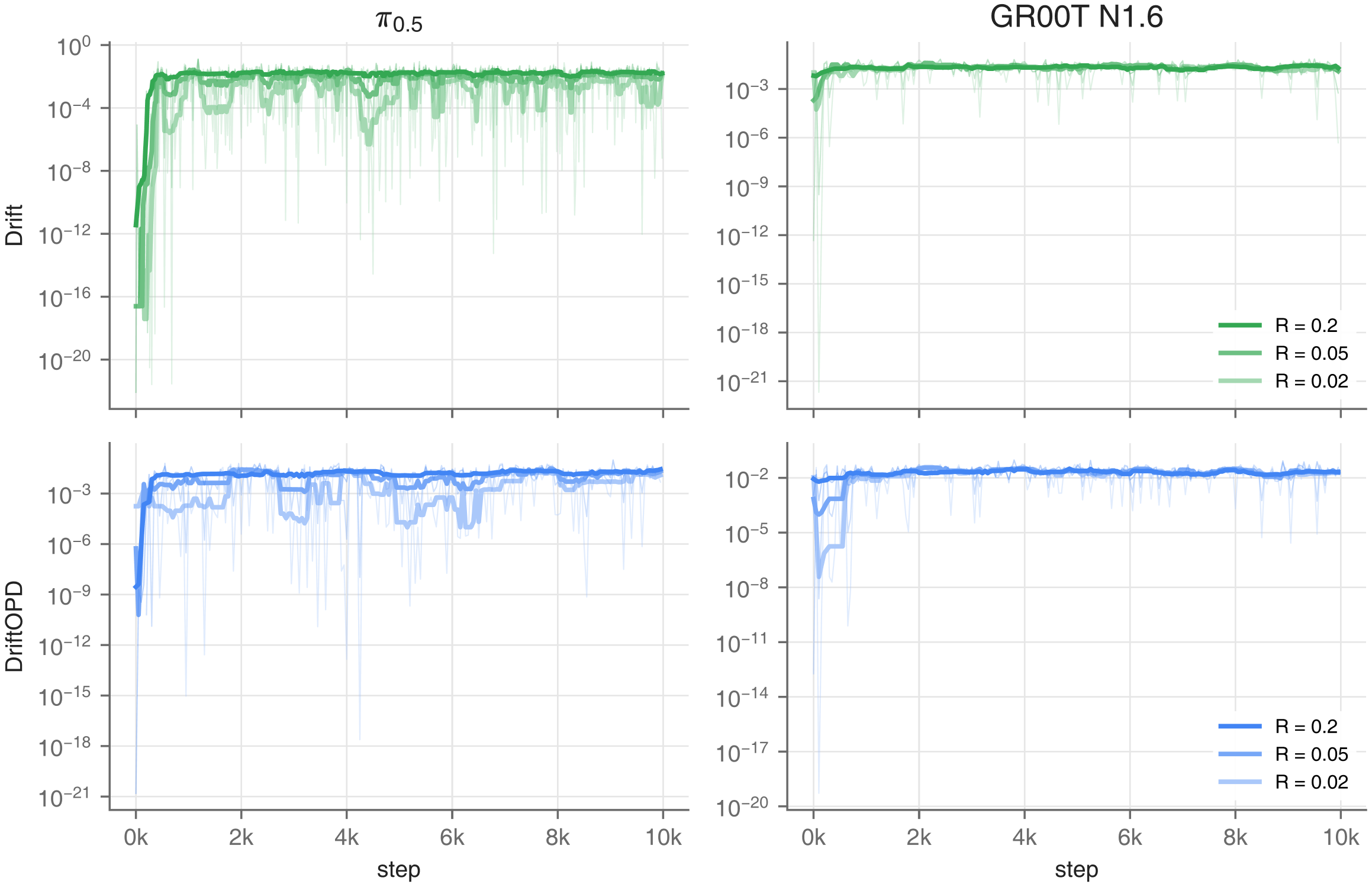}
   \caption{Bandwidth-wise RMS of the drifting loss term before normalization, for the drifting distillation (Drift) and DriftOPD on RoboCasa365.}
   \label{fig:train_drift_loss}
\end{figure}

\section{Additional Experiments}
\label{appendix:analysis}

\subsection{Oscillation Analysis for Real-World Robot}
\label{appendix:oscilation}

\begin{table}[H]
\centering
\caption{Real-world comparison of task success rate and action oscillation across the multi-step teacher and one-step policies. Values in parentheses indicate partial success rates (e.g., picking the object but failing to place or handover it), and oscillation is computed according to \Eqref{eq:oscillation}.}
\resizebox{\textwidth}{!}{%
  \begin{tabular}{l cccccc cccccc}
      \toprule
      \multirow{3}{*}{\textbf{Task}} & \multicolumn{6}{c}{\textbf{Success Rate (\%)}} & \multicolumn{6}{c}{\textbf{Oscillation ($^\circ$)}} \\
      \cmidrule(lr){2-7} \cmidrule(lr){8-13}
      & \multicolumn{1}{c}{\textbf{4-step}} & \multicolumn{5}{c}{\textbf{1-step}} & \multicolumn{1}{c}{\textbf{4-step}} & \multicolumn{5}{c}{\textbf{1-step}} \\
      \cmidrule(lr){3-7} \cmidrule(lr){9-13}
      & \textbf{Teacher} & \textbf{Teacher} & \textbf{sCD} & \textbf{MFD} & \textbf{Drift} & \textbf{DriftOPD} & \textbf{Teacher} & \textbf{Teacher} & \textbf{sCD} & \textbf{MFD} & \textbf{Drift} & \textbf{DriftOPD} \\
      \midrule
      \multicolumn{13}{l}{\textit{unimanual (right arm)}} \\
      Easy        & 93.3 & 40.0 & 0.0 (60.0) & 0.0 (0.0) & 86.7 & \textbf{93.3} & 1.70 & 1.77 & 4.64 & 7.55 & \textbf{1.97} & 2.01 \\
      Moderate    & 73.3 & 46.7 & 6.7 (20.0) & 0.0 (0.0) & \textbf{60.0} & \textbf{60.0} & 1.61 & 1.90 & 4.81 & 7.77 & \textbf{1.83} & 1.89 \\
      Hard        & 66.7 & 13.3 & 0.0 (0.0)& 0.0 (0.0) & 60.0 & \textbf{66.7} & 1.65 & 1.79 & 4.87 & 7.69 & \textbf{1.70} & 1.89 \\
      Challenging & 46.7 & 0.0  & 0.0 (0.0)& 0.0 (0.0) & \textbf{13.3} & \textbf{13.3} & 1.64 & 2.03 & 5.01 & 7.66 & 1.71 & \textbf{1.68} \\
      \cmidrule(lr){1-13}
      \textit{Avg} & 70.0 & 25.0 & 1.7 & 0.0 & 55.0 & \textbf{58.3} & 1.65 & 1.87 & 4.83 & 7.67 & \textbf{1.80} & 1.87 \\
      \midrule
      \multicolumn{13}{l}{\textit{bimanual}} \\
      Handover (LR) & 73.3 & 60.0 & 0.0 (0.0) & 0.0 (0.0) & 53.3 & \textbf{80.0} & 1.69 & 2.08 & 3.89 & 5.44 & \textbf{2.05} & 2.19 \\
      Handover (RL) & 86.7 & 53.3 & 0.0 (26.7) & 0.0 (20.0) & 40.0 & \textbf{73.3} & 1.69 & 1.93 & 3.75 & 5.16 & 2.11 & \textbf{2.02} \\
      \cmidrule(lr){1-13}
      \textit{Avg} & 80.0 & 56.7 & 0.0 & 0.0 & 46.7 & \textbf{76.7} & 1.69 & 2.00 & 3.82 & 5.30 & \textbf{2.08} & 2.10 \\
      \bottomrule
  \end{tabular}%
}
\label{tab:oscilation-real}
\end{table}

\paragraph{Oscillations for real world experiments.}
Table~\ref{tab:oscilation-real} is the full version of Table~\ref{tab:realworld-gr00t}, covering every task and every method we ran on the real robot with GR00T N1.6 as the teacher, in a right-arm unimanual variant (7-D actions) and a bimanual variant (14-D actions). Oscillation is reported in degrees, measured on the joint-position commands the policy streams to the follower arms at 30\,Hz and averaged over the arm joints of each embodiment. sCD and MFD oscillate severely, by a wide margin over the 4-step teacher, and the same pattern holds on every task and on both embodiments. Drift and DriftOPD instead stay at the level of the teacher, so collapsing the sampler to a single step costs them no smoothness while sCD and MFD pay for it with visible jitter.

\clearpage
\subsection{Stability Across LoRA Ranks}
\label{appendix:lora-rank}

\begin{figure}[H]
  \centering
   \includegraphics[width=0.6\linewidth]{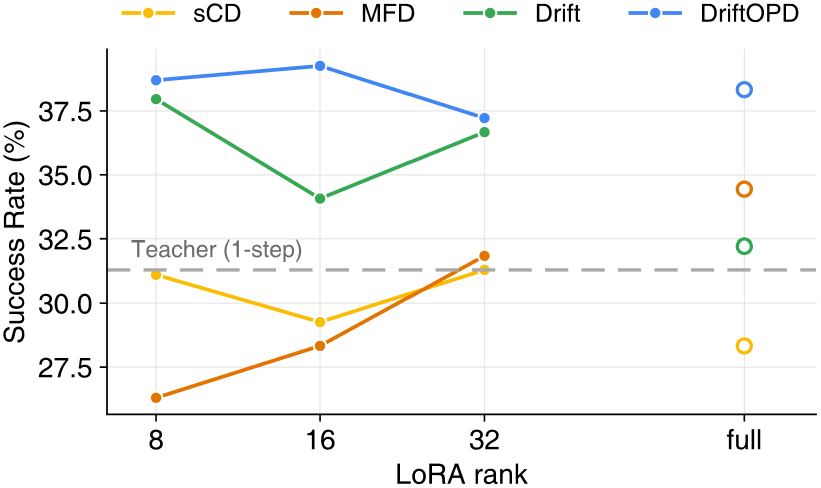}
   \caption{LoRA rank sweep for $\pi_{0.5}$ on RoboCasa365 (\texttt{atomic\_seen}). Hollow markers: full fine-tuning; dashed line: 1-step teacher (31.30\%).}
   \label{fig:pi05-lora-r}
\end{figure}

\textbf{LoRA ranks.}
The LoRA rank sets how much capacity the student has to move away from the teacher it is initialized from, and one-step distillation is exactly the regime where that budget matters: too little capacity and no objective can close the multi-step gap, too much and the student is free to drift away from the teacher in ways the objective must control. An objective that only wins at one rank would therefore say little. We sweep the rank for $\pi_{0.5}$ on \texttt{atomic\_seen} and compare all objectives at each setting (\Figref{fig:pi05-lora-r}). DriftOPD outperforms the other objectives at every rank and at full fine-tuning.

\subsection{Analysis for Single-Positive Drifting Field}
\label{appendix:mse}

\begin{figure}[H]
  \centering
   \includegraphics[width=0.6\linewidth]{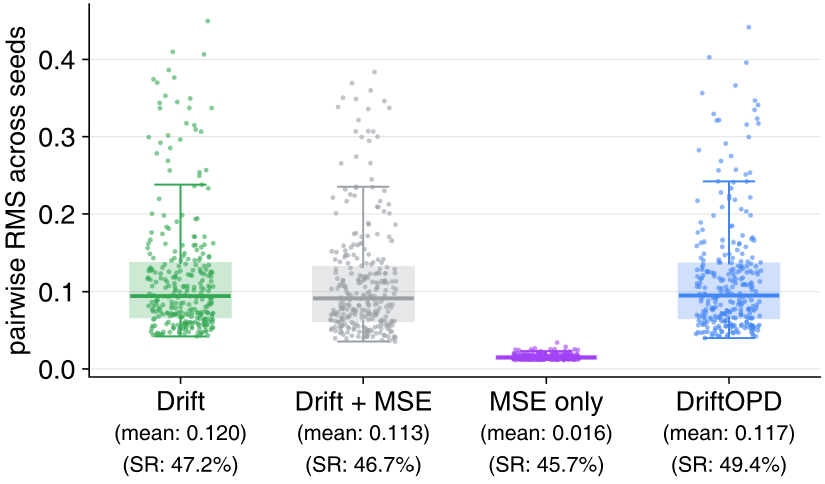}
   \caption{
   Seed-to-seed action diversity at a fixed observation (GR00T N1.6, RoboCasa365, 10k-step students). For 288 training observations (18 tasks $\times$ 16), each student samples $K = 32$ one-step action chunks that differ only in input noise. The y-axis is the pairwise RMS among the 32 chunks (12 action dims $\times$ 16 steps, normalized space); each dot is one observation, boxes show the median and IQR, and the mean is given under each label. MSE only collapses to a near-deterministic output (0.016), while Drift (0.120) and DriftOPD (0.117) keep about 7$\times$ more spread.
   }
   \label{fig:mse}
\end{figure}

\textbf{Drifting loss vs. MSE loss.}
\citet{ImplicitDriftingPolicy} argue that a drifting field built from a single positive degenerates into a mean squared error (MSE) objective unless additional practical strategies are used. Since our formulation uses one positive per context, we test directly whether it suffers from that degeneracy.

\Figref{fig:mse} measures seed-to-seed diversity rather than success, because the failure mode in question is a policy that ignores its input noise. We draw 18 tasks at random and 16 observations from each, and at every one of those 288 observations each student generates 32 action chunks that differ only in the sampled noise $z\sim\mathcal{N}(0,I)$. \textit{Drift + MSE} and \textit{MSE only} denote a plain sum of the drifting objective and an MSE term, and distillation with the MSE term alone; both follow the training protocol of Table~\ref{tab:student-optimization}. The MSE-only student collapses to an almost deterministic map: its spread is $0.016$, an order of magnitude below Drift ($0.120$) and DriftOPD ($0.117$), so it discards the initial noise entirely. Success rate alone would not have exposed this, as it degrades far less than the diversity does.

Drawing multiple independent noise samples per observation during training did not change the measured diversity, so the collapse is not simply a matter of how many samples each context sees. We instead attribute the gap to three ingredients that our field inherits from \citet{Drifting} and from the distillation setting itself: the self-masking heuristic, which excludes each generated sample from its own kernel sum; initialization from pretrained weights, which places the student in a region where the teacher's noise dependence already exists; and the multi-bandwidth kernel, which keeps a repulsive component alive at scales where a single bandwidth would have saturated. Drifting distillation applied to VLAs is therefore free of the collapse that the single-positive analysis predicts.

\subsection{Success Field Drifting}
\label{appendix:success-field}

\begin{figure}[H]
  \centering
   \includegraphics[width=0.6\linewidth]{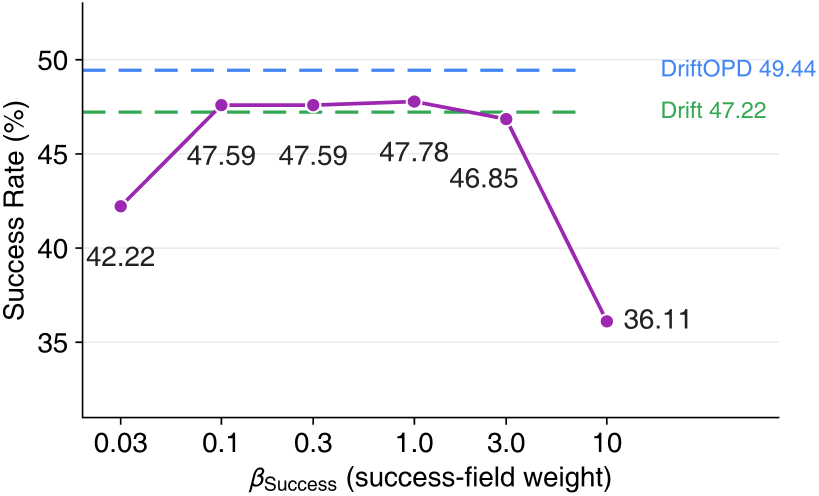}
   \caption{Effect of the success-field weight $\beta_{\mathrm{Success}}$ on GR00T N1.6 (RoboCasa365 \texttt{atomic\_seen}, 18 tasks $\times$ 30 episodes). Dashed lines indicate Drift (47.22) and DriftOPD (49.44). Success rate remains flat for $\beta_{\mathrm{Success}} \in [0.1, 3.0]$, drops at $\beta_{\mathrm{Success}} = 0.03$, and collapses at $\beta_{\mathrm{Success}} = 10$.}
   \label{fig:success-field}
\end{figure}

\textbf{Chunk-level success signal: additional success drifting field.}
DriftOPD supplies its success signal per context, through a critic conditioned on the observation the student is distilling against. We ask whether that conditioning is necessary, or whether one-step drifting distillation can also be aided by a success signal defined at the level of the action chunk alone. Suppose successful actions form a manifold in chunk space, so that whether a chunk succeeds is a property of the chunk itself rather than of the observation it was produced for. Success could then be supervised globally, leaving the ground-truth target untouched and adding a separate matching term against the distribution of successful chunks. Writing the marginal distribution of the student over observations as
\begin{equation}
    \bar{q}_\theta(a)=\mathbb{E}_o\!\left[q_\theta(a\mid o)\right],
    \label{eq:success-marginal}
\end{equation}
the local field is the usual per-context drifting field and the success field is
its global counterpart,
\begin{equation}
\begin{aligned}
    V_{\mathrm{GT}}(a,o)&=\nabla_a\log p_{\mathrm{GT}}(a\mid o)-\nabla_a\log q_\theta(a\mid o), \\
    V_{\mathrm{succ}}(a)&=\nabla_a\log p_{\mathrm{succ}}(a)-\nabla_a\log\bar{q}_\theta(a).
\end{aligned}
\label{eq:success-fields}
\end{equation}
The two are combined with a single coefficient,
\begin{equation}
    \mathcal{L}=\mathcal{L}_{\mathrm{GT}}
    + \beta_\mathrm{success}\,\mathcal{L}_{\mathrm{succ}},
    \qquad
    \mathcal{L}_{\mathrm{succ}}
    =
    \mathbb{E}_{a\sim\bar{q}_\theta}
    \left[
        \left\lVert a-\operatorname{sg}\!\left(a+\eta V_{\mathrm{succ}}\right)\right\rVert^2
    \right].
    \label{eq:success-loss}
\end{equation}
The sweep in \Figref{fig:success-field} suggests the chunk-level signal is a weak substitute for the conditioned one. Over a broad range of $\beta_{\mathrm{success}}$ the success rate stays close to the Drift baseline. 
A global success field can therefore be added without breaking training, but it brings no substantial gain to one-step drifting distillation.
Whether a chunk succeeds appears to remain a function of the observation it was generated for. Marginalizing the success signal over observations discards much of what makes it useful, and the per-context (i.e., episode) form DriftOPD uses is what carries it.

\section{Limitations and Future Work}
\label{appendix:limitation}

Our current instantiation uses a critic trained only from demonstrations, which allows DriftOPD to remain fully rollout-free. While an ideal critic can exactly provide the future-gradient contribution of the sequence-level reverse KLD objective, the practical critic provides an offline approximation whose action-gradient quality is limited by demonstration coverage. This limitation is particularly evident when the one-step teacher already matches its multi-step counterpart, as in several LIBERO and LIBERO-Plus settings (Tables~\ref{tab:libero} and~\ref{tab:libero-plus}), where the current critic provides little additional improvement beyond the teacher/student baseline. Importantly, this behavior reflects the limitation of the learned critic rather than an intrinsic ceiling of the sequence-level OPD formulation. A stronger critic with broader coverage could provide more informative guidance toward action sequences that improve beyond the multi-step policy, motivating future work on Q-functions that more accurately capture how the current action chunk influences subsequent task progress.

Finally, although we study continuous action-chunk VLAs, the proposed sequence-level OPD formulation is not restricted to this setting. For autoregressive VLAs that predict discrete actions or action tokens, step-wise conditional probabilities are directly available, making the local reverse-KL term and its future contribution potentially more direct to instantiate than for continuous flow-based action experts. Extending sequence-level OPD to such next-action prediction models is therefore a natural direction for future work.

\end{document}